\documentclass[accepted]{uai2026} 
                        
\usepackage[american]{babel}
\usepackage{multirow} 
\usepackage{natbib} 
\usepackage{mathtools} 
\usepackage{booktabs} 
\usepackage{tikz} 
\usepackage{amsmath}

\usepackage{amssymb}  
\usepackage{mathtools}
\usepackage{pifont}
\newcommand{\cmark}{\ding{51}}
\newcommand{\xmark}{\ding{55}}
\newcommand{\Ck}{C_k}
\newcommand{\vC}{\boldsymbol{C}}
\newcommand{\vp}{\boldsymbol{p}}
\newcommand{\vmu}{\boldsymbol{\mu}}
\newcommand{\Var}{\mathrm{Var}}
\newcommand{\Covmat}{\mathrm{Cov}}
\newcommand{\MI}{\mathcal{I}}
\newcommand{\Ent}{H}
\newcommand{\E}{\mathbb{E}}

\usepackage{amssymb}

\usepackage{amsthm}

\theoremstyle{plain}
\newtheorem{theorem}{Theorem}[section]

\newtheorem{lemma}[theorem]{Lemma}
\newtheorem{corollary}[theorem]{Corollary}
\theoremstyle{definition}
\newtheorem{definition}[theorem]{Definition}

\theoremstyle{remark}
\newtheorem{remark}[theorem]{Remark}

\title{Not Just How Much, But Where: Decomposing Epistemic Uncertainty into Per-Class Contributions}

\author[1]{\href{mailto:mame.toure@mail.mcgill.ca?Subject=Your UAI 2026 paper}{Mame Diarra Toure}{}}
\author[1]{David A. Stephens}

\affil[1]{Department of Mathematics and Statistics \\ McGill University}

\begin{document}
\maketitle
\begin{abstract}
In safety-critical classification, the cost of failure is often asymmetric. Yet Bayesian deep learning summarises epistemic uncertainty with a single scalar, mutual information (MI), which cannot distinguish whether a model's ignorance involves a benign or safety-critical class.
We decompose MI into a per-class vector $\Ck(x)=\sigma_k^{2}/(2\mu_k)$, with $\mu_k{=}\E[p_k]$ and $\sigma_k^2{=}\Var[p_k]$ across posterior samples. 
The decomposition follows from a second-order Taylor expansion of the entropy; the $1/\mu_k$ weighting corrects boundary suppression and makes $\Ck$ comparable across rare and common classes.
By construction $\sum_k \Ck \approx \mathrm{MI}$, and a companion skewness diagnostic flags inputs where the approximation degrades.
After characterising the axiomatic properties of $\Ck$, we validate it on three tasks:
(i)~selective prediction for diabetic retinopathy, where critical-class $\Ck$ reduces selective risk by 34.7\% over MI and 56.2\% over variance baselines;
(ii)~out-of-distribution detection on clinical and image benchmarks, where $\sum_k \Ck$ achieves the highest AUROC 
and the per-class view exposes asymmetric
shifts invisible to MI;
and (iii)~a controlled label-noise study in which $\sum_k \Ck$ shows less sensitivity to injected aleatoric noise than MI under end-to-end Bayesian training, while both metrics degrade under transfer learning.
Across all tasks, the quality of the posterior approximation shapes uncertainty at least as strongly as the choice of metric, suggesting that how uncertainty is propagated through the network matters as much as how it is measured. 
\end{abstract}
\section{Introduction}\label{sec:intro}

Deep learning classifiers deployed in safety-critical domains operate in
environments where the cost of failure is \emph{asymmetric}: missing a
sight-threatening retinal condition is categorically different from a false
positive; auto-allowing hate speech carries risks that auto-blocking benign
content does not.
Bayesian deep learning addresses this by maintaining a posterior over model
parameters, enabling principled decomposition of predictive uncertainty into
\emph{aleatoric} (irreducible data noise) and \emph{epistemic} (model
ignorance, reducible with data) components.
The foundational decomposition was formalised by~\citet{kendall2017uncertainties}
for vision tasks and by~\citet{depeweg2018decomposition} for latent-variable
models via the law of total variance; \citet{lakshminarayanan2017simple}
established deep ensembles as a scalable alternative to variational inference,
with~\citet{gustafsson2020evaluating} demonstrating empirically that ensembles
provide more reliable, better-calibrated uncertainty estimates than MC~dropout
across synthetic-to-real transfer benchmarks.

All of these methods, however, summarise epistemic uncertainty with a \emph{single scalar
per input}: the mutual information~$\MI(y;\boldsymbol{\omega}\mid x) =
\Ent(\vmu) - \E[\Ent(\vp)]$~\citep{houlsby2011bayesian,gal2016dropout}, where $\boldsymbol{\omega}$ denotes the
random model parameters under the posterior predictive distribution.
This scalar reveals \emph{how uncertain} the model is, but not \emph{which
classes} are driving that uncertainty, a distinction that matters because model
ignorance is rarely distributed uniformly across the label space.
A scalar MI of $0.3$~nats carries very different implications depending on
whether the confusion involves two benign classes or a benign and a
safety-critical one.

Recent work has approached the per-class direction from several angles, each
with a specific limitation (detailed comparison in
Appendix~\ref{app:related}).
\citet{sale2024labelwise} introduced a label-wise framework reducing $K$-class
problems to $K$ binary sub-problems, with per-class epistemic variance
$\text{EU}_k = \Var(\Theta_k)$ satisfying strong axiomatic
properties~\citep{wimmer2023quantifying}; however, raw variance suffers from
\emph{boundary suppression}, the constraint $\Var[p_k]\leq\mu_k(1-\mu_k)$
forces it to vanish for rare classes regardless of actual posterior
disagreement.
Dirichlet-based approaches~\citep{sensoy2018evidential,malinin_predictive_2018,
duan2024evidential} produce per-class epistemic covariances in a single
forward pass, but impose a strong distributional assumption on the probability
simplex.
Dataset-level methods~\citep{baltaci2023class,khan2019striking} address class
imbalance and difficulty but conflate aleatoric and epistemic components and
cannot provide input-level attribution.
Scalar epistemic metrics, including BALD~\citep{houlsby2011bayesian} for active
learning and the variance-MI connection of~\citet{smith2018understanding},
remain agnostic to class identity.
\citet{mucsanyi2024benchmarking} observed that aleatoric and epistemic estimates are often rank-correlated; \citet{dejong2026} prove this is a necessary consequence of estimator validity, shifting the defining criterion from decorrelation to \emph{orthogonality}, a property that no current method fully achieves.

In summary, no existing method provides an input-specific, normalised per-class
epistemic vector with a direct additive connection to MI. The Taylor expansion is therefore not the contribution in isolation; it is the
tool that exposes an additive, class-resolved attribution of the standard MI
quantity. The resulting construction differs from raw per-class variance by its
entropy-curvature weighting, from one-vs-all decompositions by preserving the
multiclass MI connection, and from Dirichlet/evidential approaches by making no
parametric assumption on the posterior predictive distribution.

We introduce $\vC(x) = [C_1(x),\dots,C_K(x)]^\top$, where
  $C_k(x) = \tfrac{1}{2}\,\Var[p_k](x)\,/\,\mu_k(x)$ 
is derived from a second-order Taylor approximation of MI
(Section~\ref{sec:method}).
The $1/\mu_k$ normalization arises naturally from the entropy Hessian and
counteracts boundary suppression; by construction,
$\sum_k C_k(x)\approx\MI(y;\boldsymbol{\omega}\mid x)$,
so each $C_k$ attributes a well-defined share of total epistemic uncertainty
to class~$k$.
\paragraph{Contributions.}
(1)~We derive $\vC(x)$ from the Taylor expansion of MI
(Section~\ref{sec:derivation}), characterise its axiomatic
properties relative to~\citet{wimmer2023quantifying}
(Section~\ref{sec:axioms}), and introduce a skewness diagnostic
that flags when the approximation degrades
(Section~\ref{sec:skewness}).
(2)~We validate $\vC(x)$ across three tasks: selective prediction
with class-specific deferral for diabetic retinopathy
(Section~\ref{sec:exp_dr}), OoD detection where per-class
contributions reveal asymmetric distributional shift
(Section~\ref{sec:ood}), and a controlled disentanglement study
showing that $\sum_k C_k$ is less sensitive to injected label noise
than MI under end-to-end Bayesian training, while both metrics
degrade substantially under transfer learning
(Section~\ref{sec:exp_disentangle})\footnote{Concurrently, \citet{dejong2026} show that pretrained deep ensembles exhibit worse AU/EU disentanglement than models trained from scratch, a conclusion we reach independently in Section~\ref{sec:exp_disentangle} via both the degradation of the $C_k$ approximation and increased AU/EU entanglement under transfer learning.}.%
\section{Per-Class Epistemic Uncertainty}\label{sec:method}
We decompose MI into a $K$-dimensional vector $\vC(x)$ via second-order Taylor expansion of the entropy. Each component $C_k$ quantifies class $k$'s epistemic contribution, and the sum approximates MI.
\paragraph{Setup and Notation}

Consider a $K$-class classifier with input $x$.
We perform $S$ stochastic forward passes (e.g., MC~dropout, posterior sampling, or ensemble members), each producing a probability vector $\vp^{(s)}(x) = \mathrm{softmax}(\boldsymbol{z}^{(s)}(x)) \in \Delta^{K-1}$, where $\boldsymbol{z}^{(s)}$ denotes the logits from pass~$s$.
The Monte Carlo estimates of the mean and covariance are:
{\small
\begin{align}
  \vmu(x) &= \tfrac{1}{S}\textstyle\sum_{s=1}^{S}\vp^{(s)}(x), \label{eq:mu}\\
  \Covmat[\vp] &= \tfrac{1}{S-1}\textstyle\sum_{s=1}^{S}
    (\vp^{(s)}-\vmu)(\vp^{(s)}-\vmu)^\top, \label{eq:cov}
\end{align}
}
with diagonal elements $\Var[p_k] = \frac{1}{S-1}\sum_s(p_k^{(s)}-\mu_k)^2$.
Throughout, we rely on the standard decomposition
{\small
\begin{equation}\label{eq:decomp}
  \underbrace{\Ent(\vmu)}_{\text{Total}}
  = \underbrace{\E[\Ent(\vp)]}_{\text{Aleatoric}}
  + \underbrace{\Ent(\vmu)-\E[\Ent(\vp)]}_{\text{Epistemic }(\MI)},
\end{equation}
}
where $\Ent(\vp)=-\sum_k p_k\log p_k$ is the Shannon entropy and the
epistemic term equals the mutual information $\MI(y;\boldsymbol{\omega}\mid x)$.

\subsection{From Scalar MI to a Per-Class Decomposition}\label{sec:derivation}

Our strategy is to approximate $\E[\Ent(\vp)]$ via a Taylor expansion around $\vmu$, which yields an expression for MI that decomposes additively over classes.

\begin{lemma}[Entropy Derivatives]\label{lem:entropy_derivatives}
The Shannon entropy $\Ent(\vp)=-\sum_k p_k\log p_k$ satisfies
\small{\begin{equation}\label{eq:hessian}
  \frac{\partial\Ent}{\partial p_k} = -1-\log p_k,\quad
  \frac{\partial^2\Ent}{\partial p_k\partial p_j} = -\frac{\delta_{kj}}{p_k},\quad
  \frac{\partial^3\Ent}{\partial p_k^3} = \frac{1}{p_k^2}.
\end{equation}}\noindent where $\delta_{kj}$ is the Kronecker delta ($\delta_{kj}=1$ if $k=j$, $0$ otherwise).

\end{lemma}

The Hessian is diagonal and negative semidefinite on the interior of the simplex. The diagonal Hessian is what makes a per-class decomposition possible: there are
no cross-class terms at second order, so the quadratic approximation splits
cleanly over classes.
\begin{theorem}[MI Approximation \footnote{\citet{smith2018understanding} also exploit a Taylor expansion to
relate softmax variance to MI, but expand the \emph{logarithm}
$\log p_k$ rather than the entropy function $\Ent(\vp)$ itself (their
Section~3.2, Eq.~10).
This yields a class-\emph{averaged} variance
$\frac{1}{K}\sum_k\Var[p_k]$ with equal weights per class, whereas
expanding the entropy via its Hessian yields the
curvature-\emph{weighted} sum $\frac{1}{2}\sum_k\Var[p_k]/\mu_k$,
in which the $1/\mu_k$ factor enables per-class attribution.}]\label{thm:taylor}
Let $\vp^{(s)}\sim Q$ with mean $\vmu=\E_Q[\vp]$ and per-class variance
$\Var[p_k]$.
Then
\begin{equation}\label{eq:EH-approx}
  \E_Q[\Ent(\vp)] \;\approx\; \Ent(\vmu) - \frac{1}{2}\sum_{k=1}^{K}
    \frac{\Var[p_k]}{\mu_k},
\end{equation}
with remainder $O(\E[\|\vp-\vmu\|^3])$, and consequently
\begin{equation}\label{eq:MI-approx}
  \MI(y;\boldsymbol{\omega}\mid x) \;\approx\;
  \frac{1}{2}\sum_{k=1}^{K}\frac{\Var[p_k](x)}{\mu_k(x)}.
\end{equation}
\end{theorem}
The simplex constraint forces $\Covmat[\vp]$ to have rank at most $K{-}1$, coupling the $K$ components; this does not invalidate the trace formula. See Appendix~\ref{app:proof:taylor} for proof. 
\begin{remark}[Exact classwise MI identity]
Because Shannon entropy is separable across classes, MI also admits the exact
additive identity
{\small \begin{equation}
\label{eq:exact-classwise-mi}
\MI(y;\boldsymbol{\omega}\mid x)
=
\sum_{k=1}^{K} M_k(x),
\end{equation}
where
\begin{equation} 
\label{eq:exact-classwise-mi-term}
M_k(x)
=
\E\bigl[p_k\log p_k\bigr]
-
\mu_k\log\mu_k
=
\E\left[
p_k\log\left(\frac{p_k}{\mu_k}\right)
\right],
\end{equation}
}
which tends to zero as $\mu_k\to 0^+$. Therefore, although $M_k$ is exact, it can still be attenuated for rare classes.
Our goal is different: we seek a variance-based, boundary-aware classwise
score whose dependence on posterior disagreement is explicit. The quantity
$C_k$ introduced below is the leading second-order approximation to this exact
term. Its $1/\mu_k$ weighting exposes the entropy-curvature correction that
prevents raw variance from being suppressed near low-probability classes, while
the skewness diagnostic in Section~\ref{sec:skewness} indicates when
higher-order terms are non-negligible.
\end{remark}

\subsection{Per-Class Epistemic Uncertainty Vector}\label{sec:definition}

The additive structure of~\eqref{eq:MI-approx} invites a natural definition.

\begin{definition}[Per-Class Epistemic Uncertainty]\label{def:Ck}
For a $K$-class classifier with mean $\vmu(x)$ and per-class variance
$\Var[p_k](x)$ from $S$ stochastic forward passes, define the
\textbf{per-class epistemic uncertainty vector}
$\vC(x)=[C_1(x),\dots,C_K(x)]^\top$ with
\begin{equation}\label{eq:Ck}
  \boxed{C_k(x) \;:=\; \frac{1}{2}\,\frac{\Var[p_k](x)}{\mu_k(x)},}
\end{equation}
In practice, we add $\varepsilon=10^{-10}$ to the denominator for numerical stability.%
\footnote{The ratio $\Var[p_k]/\mu_k$ coincides with the
\emph{index of dispersion}~\citep{CoxLewis1966} (Fano factor~\citep{Fano1947}).
Here the $\tfrac{1}{2}$ prefactor and the $1/\mu_k$ weighting both emerge from
the entropy Taylor expansion, so the components sum to a well-defined
information-theoretic quantity.}
When $\mu_k=0$ across all passes, $\Var[p_k]=0$ and $C_k=0$ trivially.
\end{definition}

By construction, the vector satisfies
\begin{equation}\label{eq:additive}
  \sum_{k=1}^{K} C_k(x) \;\approx\; \MI(y;\boldsymbol{\omega} \mid x).
\end{equation}
so each $C_k$ attributes a well-defined share of total epistemic uncertainty to
class~$k$, providing the localisation that scalar MI lacks.

\subsection{Why Variance Alone Fails: Boundary Suppression}\label{sec:motivation}

A natural alternative to $C_k$ would be to use the raw variance $\Var[p_k]$
as a per-class epistemic measure, as proposed by \citet{sale2024labelwise}.
However, variance on a bounded interval exhibits a structural limitation near
the simplex boundary that the $1/\mu_k$ normalisation in~\eqref{eq:Ck} is
designed to address.

\begin{lemma}[Variance Bound on the Simplex]\label{lem:boundary}
For any class $k$ with mean prediction $\mu_k \in [0,1]$:
\begin{equation}\label{eq:var-bound}
  \Var[p_k] \;\leq\; \mu_k(1 - \mu_k).
\end{equation}
Consequently, $\Var[p_k] \to 0$ as $\mu_k \to 0$ or $\mu_k \to 1$,
regardless of the degree of model disagreement in logit space.
See Appendix~\ref{app:proof:boundary}.
\end{lemma}

This bound reveals a concrete problem: for a class with $\mu_k \approx 0.01$,
variance is capped at roughly $0.01$ even if the model's stochastic passes
exhibit maximal disagreement about that class.
The following lemma shows that $C_k$ corrects exactly this pathology.

\begin{lemma}[Boundary Behaviour of $C_k$]\label{lem:ck-boundary}
Under the same conditions:
\begin{equation}
  C_k = \frac{1}{2}\,\frac{\Var[p_k]}{\mu_k} \;\leq\; \frac{1}{2}(1-\mu_k),
\end{equation}
which approaches $\frac{1}{2}$ (not zero) as $\mu_k \to 0$.
While $\Var[p_k]$ is crushed to zero near the simplex
boundary, $C_k$ retains a non-vanishing upper bound.
See Appendix~\ref{app:proof:boundary}.

\end{lemma}
The mechanism is not ad hoc: the Hessian entry $-1/\mu_k$ grows large when $\mu_k$ is small, so a given amount of probability variance carries more information-theoretic weight for low-probability classes. Conversely, when $\mu_k \to 1$ the curvature vanishes and $C_k \to 0$, the correct limit. This renders $C_k$ comparable across classes with very different base rates, a property $\Var[p_k]$ lacks by construction.

\subsection{Axiomatic Analysis}\label{sec:axioms}
Having described both the diagonal vector $\vC(x)$ and the off-diagonal covariance structure, we now examine whether $\vC(x)$ satisfies standard properties of epistemic uncertainty measures. \citet{wimmer2023quantifying} propose a set of axioms formalising desirable behaviour for epistemic uncertainty measures. We assess the aggregate $\mathrm{EU}_{\mathrm{approx}}(Q) = \sum_k \Ck$ against five of these axioms, where $Q$ denotes the second-order distribution over probability vectors $\vp$ with mean $\vmu = \E_Q[\vp]$.

\begin{theorem}[Axiomatic Profile]\label{thm:axioms}
The approximate epistemic uncertainty $\mathrm{EU}_{\mathrm{approx}}(Q) = \sum_{k=1}^K C_k$ satisfies axioms A0, A1, and A3, and violates A2 and A5 of \citet{wimmer2023quantifying}. See Appendix \ref{app:proof:axiom} for details and proof.
\end{theorem}

The violations of A2 and A5 are not artefacts of the approximation; they are inherited from MI itself~\citep{wimmer2023quantifying}.
More revealing is the relationship between A5 and boundary suppression:

\begin{corollary}[A5 Violation as Boundary Correction]\label{cor:a5-tradeoff}
The violation of A5 is the precise mechanism that counteracts boundary suppression.
A5 requires EU to be insensitive to $\mu_k$; but by Lemma~\ref{lem:ck-boundary}, sensitivity to $\mu_k$ is what prevents $C_k$ from vanishing as $\mu_k \to 0$.
\end{corollary}

This trade-off clarifies the design space: the label-wise variance $\mathrm{EU}_k = \Var(\Theta_k)$ of \citet{sale2024labelwise} satisfies A5 but pays the cost of boundary suppression.
Our $C_k$ sacrifices A5 to preserve comparability across classes with different base rates, a trade-off we argue is favourable in safety-critical settings
where rare classes matter most.
\begin{remark}[A3 vs.\ Exact MI]\label{rem:a3}
Unlike exact MI, which \citet{wimmer2023quantifying} show violates A3 in
general (their Proposition~4), $\sum_k C_k$ satisfies it strictly.
Since the sum is linear in each $\Var[p_k]$ with fixed positive
coefficients $1/(2\mu_k)$, any mean-preserving spread necessarily
increases it, so greater model disagreement always produces greater
epistemic uncertainty.
\end{remark}

\subsection{Reliability Diagnostic via Skewness}\label{sec:skewness}
The same $1/\mu_k$ factor that corrects boundary suppression also controls where the Taylor expansion breaks down.
As $\mu_k \to 0$, the curvature $|\partial^2\Ent/\partial p_k^2| = 1/\mu_k$ diverges, and the quadratic approximation becomes increasingly sensitive to higher-order terms.
A natural question is: \emph{when can we trust $C_k$?}
We answer this by examining the next term in the expansion.

\begin{lemma}[Third-Order Correction]\label{lem:third-order}
Including the third-order term from Lemma~\ref{lem:entropy_derivatives}, the expected entropy satisfies
\begin{equation}\label{eq:EH-third}
  \E[\Ent(\vp)]
    \;\approx\;
    \Ent(\vmu)
    - \frac{1}{2}\sum_k \frac{\Var[p_k]}{\mu_k}
    + \frac{1}{6}\sum_k \frac{m_{3,k}}{\mu_k^2}\,,
\end{equation}
where $m_{3,k} = \E[(p_k - \mu_k)^3]$ is the third central moment, estimated as
{\small $\hat{m}_{3,k} = \frac{1}{S}\sum_{s=1}^{S}(p_k^{(s)} - \mu_k)^3$}.
\end{lemma}
We use the third-order term as a \emph{diagnostic} rather than a correction.
Including it yields
$C_k^{(3)} = \tfrac{1}{2}\Var[p_k]/\mu_k - \tfrac{1}{6}\,m_{3,k}/\mu_k^2$,
but the $1/\mu_k^2$ singularity can drive this quantity negative for
right-skewed distributions near the simplex boundary, violating
non-negativity~(A0).
More fundamentally, adding Taylor terms improves accuracy only when
posterior samples are tightly concentrated around the mean, which is
precisely not the high-uncertainty regime where the metric matters most.
The skewness ratio $\rho_k$ (Section~\ref{sec:skewness}) instead flags
when the second-order approximation is unreliable, without compromising
the guarantees of $C_k$.

\begin{definition}[Skewness Diagnostic]\label{def:rho}
For each class $k$ with $\Var[p_k] > 0$, define the reliability indicator
\begin{equation}\label{eq:rho}
  \boxed{\rho_k(x)
    \;=\;
    \frac{|m_{3,k}|}{3\,\mu_k \cdot \Var[p_k]}\,.}
\end{equation}
\end{definition}

To see why this is the right ratio, observe that the second-order contribution to MI from class $k$ is $\frac{1}{2}\Var[p_k]/\mu_k$, while the third-order correction is $\frac{1}{6}|m_{3,k}|/\mu_k^2$.
Their ratio is
{\small
\begin{equation}
  \frac{\tfrac{1}{6}|m_{3,k}|/\mu_k^2}{\tfrac{1}{2}\Var[p_k]/\mu_k}
    \;=\;
    \frac{|m_{3,k}|}{3\,\mu_k\cdot\Var[p_k]}
    \;=\; \rho_k\,.
\end{equation}
}%
When $\rho_k \ll 1$, the third-order correction is negligible relative to $C_k$, and the approximation is reliable for class $k$.
When $\rho_k$ is appreciable, the second-order approximation is degrading and
$C_k$ should be interpreted with caution. In our experiments we use
$\rho_k>0.3$ as a practical reliability threshold, but this value should be read
as a task-dependent tolerance for approximation error rather than a universal
constant.
The diagnostic is computable from the same $S$ stochastic passes at negligible additional cost and can be reported alongside $\vC(x)$.

\paragraph{Off-diagonal structure and the CBEC metric.}
While $\vC(x)$ captures diagonal variance per class, the empirical covariance
$\Covmat[\vp]$ encodes complementary information about \emph{pairwise confusion}.
A strongly negative entry $\Covmat[p_i, p_j] \ll 0$ means that probability mass
flows systematically between classes~$i$ and~$j$ across stochastic passes: the
model is actively trading one for the other rather than being diffusely uncertain.
In safety-critical settings this directional signal matters: a model torn between
a safe and a critical class is qualitatively different from one torn between two
safe classes, yet both can produce identical $\sum_k \Ck$.
When the skewness diagnostic (Section~\ref{sec:skewness}) flags unreliable $C_k$
for rare critical classes ($\mu_k \approx 0$), we gate their product by the
empirical negative correlation between safe and critical classes:
\begin{equation}\label{eq:cbec}
  \text{CBEC}(x)
    = \sum_{i \in \mathcal{S}}\sum_{j \in \mathcal{C}}
      \sqrt{C_i \cdot C_j}\;\cdot\;\max(0,\,-\rho_{ij}),
\end{equation}
where $\mathcal{S}$ and $\mathcal{C}$ are the safe and critical class sets and
$\rho_{ij}$ is the empirical Pearson correlation of $p_i$ and $p_j$ across
the $S$ forward passes.
Three design choices interact.
The \emph{geometric mean} $\sqrt{C_i \cdot C_j}$ requires \emph{both} classes to
carry elevated epistemic uncertainty: single-class inflation is dampened by the
square root, and if either class is certain the product vanishes.
The \emph{correlation gate} $\max(0,-\rho_{ij})$ filters coincidental co-elevation:
the score is nonzero only when MC~draws show the model actively exchanging
probability between the two classes, a Taylor-free condition that remains
reliable even when $C_k$ degrades.
The \emph{restricted domain} $\mathcal{S}\times\mathcal{C}$ focuses the sum
exclusively on cross-boundary confusion, excluding within-safe and within-critical
co-elevation that carry no deferral signal.
Together, these properties make CBEC the preferred fallback when $\rho_k > 0.3$
for critical classes; $C_{\text{crit\_max}}$ remains the primary metric when the
Taylor approximation is reliable.

\section{Selective Prediction for Diabetic Retinopathy}\label{sec:exp_dr}
We evaluate $\vC(x)$ on selective prediction~\citep{geifman2017selective} for diabetic retinopathy~(DR) grading, a task where the cost of failure is acutely asymmetric: missing a sight-threatening retinopathy carries irreversible consequences, while a false positive triggers only a follow-up examination.

\subsection{Experimental Setup}\label{sec:dr-setup}

\paragraph{Data.}
We pool three public DR grading benchmarks: EyePACS~\citep{cuadros2009eyepacs},
APTOS~2019~\citep{aptos2019}, and
Messidor-2~\citep{decenciere2014feedback}, totalling $39{,}970$ colour fundus
photographs.
The International Clinical DR Severity Scale~\citep{wilkinson2003proposed} defines
five grades; we consolidate the two most severe into a single grade because both
require urgent referral.
The resulting four-class distribution places Grades~2--3 as \emph{critical}
($\mathcal{C} = \{2,3\}$, requiring treatment) and Grades~0--1 as \emph{safe}
($\mathcal{S} = \{0,1\}$); data are split patient-stratified into
train/val/test with no patient leakage (full details in Appendix~\ref{app:dataset}).

\paragraph{Model.}
We construct a fully Bayesian EfficientNet-B4~\citep{tan2019efficientnet} using
the low-rank variational inference framework of \citet{toure2026singular}, with
a scale-mixture Gaussian prior ($\pi = 0.5$) and KL~annealing.
Full architectural details are in Appendix~\ref{app:model}.

\paragraph{Inference and evaluation.}
Uncertainty is estimated from $S = 30$ stochastic forward passes through the posterior.
Our primary safety metric is the \textbf{critical false negative rate}~(Critical~FNR, Appendix \ref{app:eval-metrics}): among critical-class samples retained at a given coverage level, the fraction misclassified as safe.
We summarise performance across coverage levels via the area under the selective risk curve~(AUSC)~\citep{elyaniw2010}, with 95\% CIs from 200 bootstrap resamples.
On the test set, the Bayesian model achieves accuracy $0.8$ and quadratic weighted kappa $0.65$; per-class performance in Appendix~\ref{app:classification}.

\subsection{Deferral Policies}\label{sec:deferral-policies}

We compare 10 deferral policies across four families~(Table~\ref{tab:policies}; formal definitions in Appendix~\ref{app:metrics}).

\begin{table}[!ht]
\centering
\caption{Deferral policies evaluated. Higher score $=$ defer.
$\mathcal{C}{=}\{2,3\}$, $\mathcal{S}{=}\{0,1\}$.}
\label{tab:policies}
\footnotesize
\setlength{\tabcolsep}{4pt}
\begin{tabular}{lll}
\toprule
\textbf{Family} & \textbf{Policy} & \textbf{Score} \\
\midrule
\multirow{3}{*}{Scalar}
  & Entropy   & $H[\bar{p}]$ \\
  & MI        & $H[\bar{p}] - \E[H[p^{(s)}]]$ \\
  & MaxProb   & $1 - \max_k \mu_k$ \\
  & Sale\_EU\_global & $\sum_k \sigma^2_k$ \\
\midrule
\multirow{3}{*}{\shortstack[l]{Per-class\\variance}}
  & Var\_crit        & $\max_{k \in \mathcal{C}} \sigma^2_k$ \\
  & Sale\_EU\_crit   & $\sum_{k \in \mathcal{C}} \sigma^2_k$ \\
\midrule
Per-class MI
  & OvA\_MI & $\sum_{k \in \mathcal{C}} \mathrm{MI}_k^{\text{bin}}$ \\
\midrule
\multirow{3}{*}{\shortstack[l]{Per-class $C_k$\\(proposed)}}
  & $C_{\text{crit\_sum}}$ & $\sum_{k \in \mathcal{C}} C_k$ \\
  & $C_{\text{crit\_max}}$ & $\max_{k \in \mathcal{C}} C_k$ \\
  & CBEC & Eq.~\ref{eq:cbec} \\
\bottomrule
\end{tabular}
\end{table}

Scalar baselines are class-agnostic; per-class variance targets critical classes
via raw $\sigma^2_k$ but is confounded by $\mu_k$
(Lemma~\ref{lem:boundary}); OvA\_MI computes binary MI per critical class.
Our three proposed metrics build on $\sum_k C_k \approx \MI$
(Section~\ref{sec:definition}): $C_{\text{crit\_max}}$ gives the sharpest
signal when $C_k$ estimates are reliable ($\rho_k < 0.3$); $C_{\text{crit\_sum}}$
is a smoother alternative; CBEC (Eq.~\ref{eq:cbec}) gates safe--critical pairs
by empirical negative correlation for robustness when skewness degrades $C_k$.

\subsection{Results}\label{sec:results-dr}
\paragraph{Theory validation.}
Before evaluating downstream performance, we verify the additive
decomposition~\eqref{eq:additive}: $\sum_k C_k$ and exact MI achieve Pearson
$r = 0.988$ and Spearman $r = 0.998$ across all $7{,}948$ test samples (Figure~\ref{fig:theory-validation}),
confirming the second-order approximation preserves epistemic uncertainty rankings.
The skewness diagnostic~(Section~\ref{sec:skewness}) shows $C_k$ is reliable
($\rho_k < 0.3$) for ${>}94\%$ of safe-class samples and $82\%$/$63\%$ of
Grade~2/3 respectively. Because the cutoff is a practical tolerance rather
than a universal constant, we also sweep thresholds
$\tau \in \{0.1,0.2,0.3,0.5\}$ in Appendix~\ref{app:skewness}. Tightening the
criterion to $\tau=0.1$ lowers the reliable fraction for critical grades to
roughly $68$--$75\%$, while relaxing it to $\tau=0.5$ raises it to
$93$--$97\%$. The qualitative pattern is stable across thresholds: safe
classes remain highly reliable, whereas rare critical classes are the most
sensitive to higher-order Taylor terms.

\paragraph{Main results.}
Table~\ref{tab:selective-bnn} reports AUSC and Critical~FNR at 80\% coverage with bootstrap CIs.
$C_{\text{crit\_max}}$ achieves the lowest AUSC ($0.285$), reducing the area under the selective risk curve by 34.7\% relative to MI ($0.436$, $p < 0.005$) and by 56.2\% relative to Sale\_EU\_crit ($0.650$, $p < 0.005$).
At the clinically relevant 80\% coverage operating point, $C_{\text{crit\_max}}$ achieves FNR $0.302$, compared to $0.339$ for MI and $0.409$ for Sale\_EU\_crit. Figure~\ref{fig:selective-main} visualises these results: $C_{\text{crit\_max}}$ dominates across the full coverage range (left), with non-overlapping bootstrap interquartile ranges against scalar baselines (right).
The improvement over raw variance is especially striking.
Sale\_EU\_crit and Var\_crit (AUSC $0.650$ and $0.606$) perform \emph{worse than entropy} ($0.604$), confirming that boundary suppression~(Lemma~\ref{lem:boundary}) severely confounds unnormalised variance for critical classes with $\mu_k \approx 0.06$.
The $1/\mu_k$ normalisation in $C_k$ corrects precisely this: Grade~3 has mean probability $\sim$0.06, giving entropy curvature $\sim$17, so even small variance there carries large information-theoretic weight. This normalisation benefit is robust to the choice of inference method:
a deep ensemble \citep{lakshminarayanan2017simple} of five members yields $C_{\text{crit\_max}}$ AUSC $0.390$
vs.\ $0.447$ for Sale\_EU\_crit ($12.9\%$ reduction; Appendix~\ref{app:deep-ensemble}).

\begin{table}[!ht]
\centering
\caption{Selective prediction for DR grading (Bayesian EfficientNet-B4, $S{=}30$).
AUSC integrates Critical FNR over all coverages (lower $=$ safer).
\textbf{Bold} = best; CIs: 200 bootstrap resamples.}
\label{tab:selective-bnn}
\resizebox{\columnwidth}{!}{
\begin{tabular}{llcc}
\toprule
Family & Policy & AUSC (Critical FNR)$\downarrow$ & Critical FNR @80\%$\downarrow$ \\
\midrule
\multirow{3}{*}{Scalar}
  & Entropy     & $0.604 \pm 0.022$ & $0.401 \pm 0.016$ \\
  & MI          & $0.436 \pm 0.019$ & $0.339 \pm 0.014$ \\
  & MaxProb     & $0.639 \pm 0.022$ & $0.439 \pm 0.017$ \\
    & Sale\_EU\_global & $0.457 \pm 0.018$ & $0.341 \pm 0.015$ \\
\midrule
\multirow{2}{*}{Per-class var.}
  & Var\_crit   & $0.606 \pm 0.014$ & $0.379 \pm 0.016$ \\
  & Sale\_EU\_crit & $0.650 \pm 0.013$ & $0.409 \pm 0.016$ \\
\midrule
Per-class MI
  & OvA\_MI     & $0.452 \pm 0.017$ & $0.367 \pm 0.015$ \\
\midrule
\multirow{3}{*}{\shortstack[l]{Per-class $C_k$\\(proposed)}}
  & $C_{\text{crit\_sum}}$ & $0.327 \pm 0.017$ & $0.321 \pm 0.014$ \\
  & $C_{\text{crit\_max}}$ & $\mathbf{0.285 \pm 0.016}$ & $\mathbf{0.302 \pm 0.013}$ \\
  & CBEC        & $0.416 \pm 0.020$ & $0.335 \pm 0.014$ \\
\bottomrule
\end{tabular}}
\end{table}

\begin{figure}[!ht]
  \centering
  \includegraphics[width=\columnwidth]{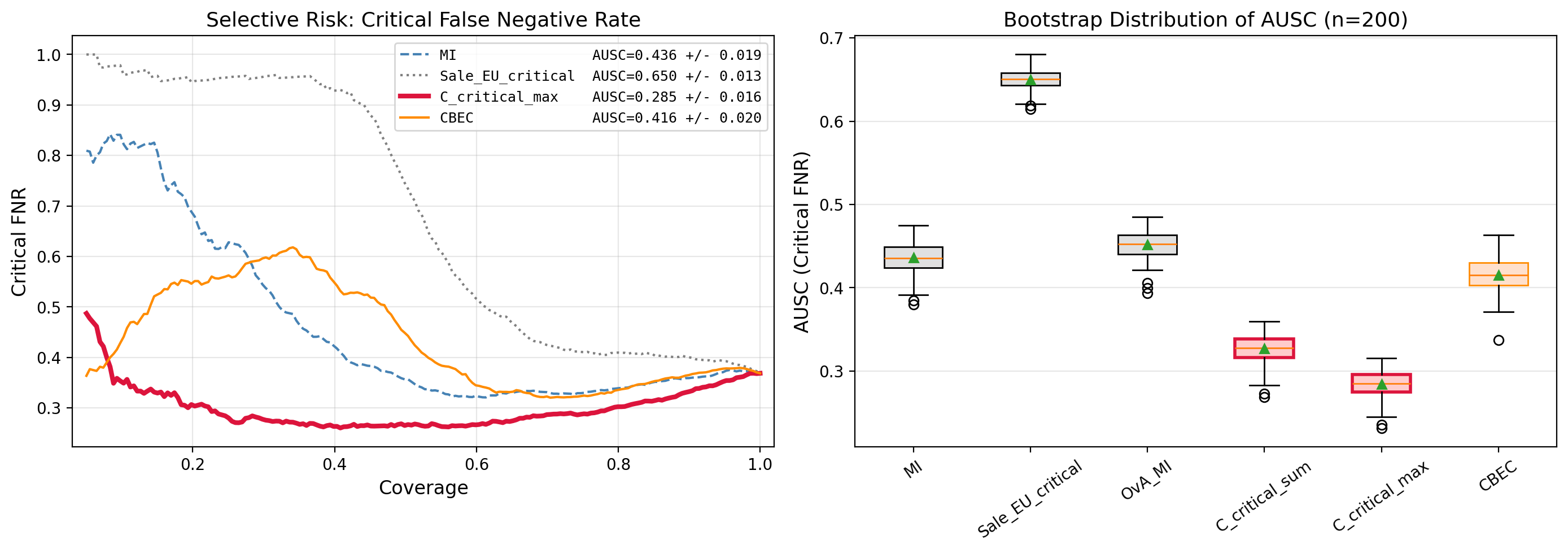}
  \caption{Selective prediction for DR.
  \textbf{Left:} Critical FNR vs.\ coverage;
  $C_{\text{crit\_max}}$ dominates all baselines across the full range.
  \textbf{Right:} Bootstrap AUSC distribution ($n{=}200$);
  $C_k$-based policies (red/orange) show non-overlapping interquartile ranges against scalar baselines (grey).}
  \label{fig:selective-main}
\end{figure}


\paragraph{Interpretability: Error Signatures.}
A distinctive advantage of $\vC(x)$ over scalar MI is that it reveals the \emph{structure} of model confusion.
Figure~\ref{fig:error-signature} shows that catastrophic misses (Grade~3 predicted as Grade~0) and severity underestimates (Grade~3 predicted as Grade~2) have nearly identical scalar MI ($0.024$ vs.\ $0.027$~nats) but very different $C_k$ signatures.
The catastrophic miss concentrates epistemic mass on $C_2$, identifying Grade~2 as the \emph{bottleneck} of confusion, while the severity underestimate elevates $C_0$, indicating doubt about healthy status.
These qualitatively different failure modes, invisible to scalar metrics, suggest
distinct remediation strategies; the epistemic confusion matrix
(Appendix~\ref{app:confusion}) further confirms that cross-boundary confusion
is $2.7\times$ stronger than within-group confusion, validating the clinical
partition (detailed analysis in Appendix~\ref{app:error-signatures}).

\paragraph{Robustness across inference methods.}
To test robustness to the posterior approximation, we repeat the evaluation using
standard MC~dropout~\citep{gal2016dropout} ($p{=}0.3$, $S{=}30$).
The ranking changes markedly: CBEC achieves the lowest AUSC ($0.197 \pm 0.012$,
winning 100\% of bootstrap samples), a 53.6\% reduction over MI ($0.425 \pm 0.024$),
while $C_{\text{crit\_max}}$ drops to $0.419$.
MC~dropout is known to produce regions of spurious confidence where the approximate
posterior underestimates uncertainty~\citep{smith2018understanding}. In our
framework, this inflates $\rho_k$ for critical classes, degrading the Taylor
approximation underlying $C_k$ while leaving CBEC's correlation gate intact.
A deep ensemble of five members confirms the pattern: CBEC again achieves
the lowest AUSC ($0.223$), even though $C_k$ reliability remains high
($\rho_k < 0.3$ for ${>}85\%$ of Grade~3 samples), indicating that the correlation gate
captures cross-boundary structure beyond what single-class targeting provides.
This confirms the intended complementarity: \emph{$C_{\text{crit\_max}}$ is
the primary metric under well-calibrated posteriors; CBEC provides a robust
alternative across inference regimes.} Full results in Appendix~\ref{app:mcdropout} and~\ref{app:deep-ensemble}.
\begin{figure}[!ht]
  \centering
  \includegraphics[width=\columnwidth]{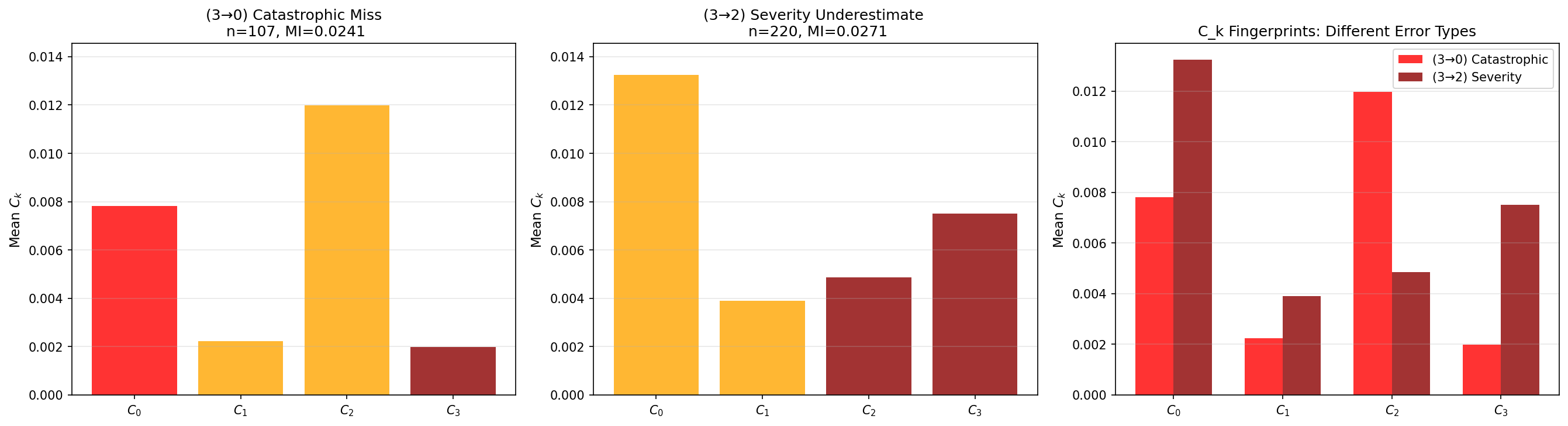}
  \caption{Epistemic signatures for Grade~3 errors with similar MI but distinct
$C_k$ patterns. \textbf{Left:} Catastrophic miss ($3 \to 0$, MI$=$0.024)
concentrates on $C_2$. \textbf{Centre:} Severity underestimate ($3 \to 2$,
MI$=$0.027) elevates $C_0$. \textbf{Right:} Grouped comparison showing that
the $C_k$ fingerprints differ despite similar MI; lighter red denotes
catastrophic safe predictions, darker red denotes within-critical severity
underestimates.}
  \label{fig:error-signature}
\end{figure}

\section{Out-of-Distribution Detection}\label{sec:ood}

We evaluate the per-class epistemic uncertainty metric on out-of-distribution
(OoD) detection, where elevated epistemic uncertainty is often used as a signal
of distributional shift. This experiment should be read as a comparative
within-posterior evaluation on the studied shifts, not as a claim that MI,
$\sum_k C_k$, $EU_{\text{var}}$, or related uncertainty scores are universally
reliable OoD detectors. In particular, uncertainty-based OoD scores can suffer
from overconfidence failure modes in ReLU classifiers
\citep{ulmer2021know}; our question is whether, given the same predictive
posterior, the class-resolved decomposition provides more useful separation and
diagnostics than scalar alternatives.

\subsection{Setup}
\paragraph{Datasets.}
We consider two OoD detection tasks: (1)~\emph{FashionMNIST}~\citep{xiao2017fashion}$\to$\emph{KMNIST}~\citep{clanuwat2018deep}, a 10-class image task (50k/10k train/test, 10k OoD); and (2)~\emph{MIMIC-III ICU}~\citep{johnson2016mimic}$\to$\emph{Newborn}, binary mortality prediction on 44 tabular features (40k/4.5k train/test, 5.4k OoD).

\paragraph{Model.}
Both tasks use low-rank Bayesian neural networks with Gaussian posteriors.
FashionMNIST: two hidden layers (1200 units, rank 25), $S{=}50$ MC samples.
MIMIC-III: two hidden layers (128 units, rank 15), $S{=}512$ MC samples, with class-weighted training to address mortality imbalance (8.6\% prevalence). See Appendix~\ref{app:ood-model} for full architectural and training details. 

\paragraph{Metrics.}
We compare four uncertainty scores: negative maximum softmax probability~\citep{hendrycks2017}, MI, $EU_{\text{var}} = \sum_k \sigma^2_k$, and our proposed $\sum_k C_k$.
Performance is measured by AUROC(ID samples labelled 0, OoD samples labelled 1); results are mean $\pm$ std over 5 seeds.

\subsection{Results}

Table~\ref{tab:ood-main} summarises OoD detection performance.
$\sum_k C_k$ achieves the highest AUROC on both datasets: $0.735$ on FashionMNIST versus $0.724$ for MI, and $0.815$ on MIMIC-III versus $0.802$ for MI.
The OoD-to-ID ratio column reveals the mechanism: on FashionMNIST, $\sum_k C_k$ amplifies the distributional shift signal to $6.43\times$ versus $5.92\times$ for MI.
On FashionMNIST, this effect is predicted by Lemma~\ref{lem:third-order}: the gap $\sum_k C_k - \text{MI} \approx \frac{1}{6}\sum_k m_{3,k}/\mu_k^2$ grows with posterior skewness, and OoD samples produce higher skewness diagnostic $\rho_k$ than ID samples across all ten classes (Appendix~\ref{app:skewness_ood}, Figure~\ref{fig:rho_fmnist}).
On MIMIC-III the picture is more nuanced: the survival class ($k{=}0$) follows the expected pattern ($\rho_0^{\text{OOD}} > \rho_0^{\text{ID}}$, median shift $+43\%$), but the mortality class ($k{=}1$) reverses because near-zero $\mu_1$ for newborns collapses $\Var[p_1]$ faster than $1/\mu_1$ can compensate, concentrating the third-order signal entirely in class~0 (Appendix~\ref{app:skewness_ood}).

The unnormalised baseline $EU_{\text{var}}$ underperforms MI on both datasets
despite matching its ratio on FashionMNIST ($5.92\times$): without
mean-normalisation, raw variance $\sigma^2_k$ is bounded above by
$\mu_k(1-\mu_k)$, compressing the dynamic range of uncertainty scores.
On MIMIC-III $EU_{\text{var}}$ achieves the \emph{highest} ratio
($1.71\times$) yet the \emph{lowest} AUROC among epistemic measures.
A high mean ratio does not imply good separability: AUROC depends on the
full distributional overlap, and the un-normalised variance sum concentrates
both ID and OoD scores near zero with a shared heavy tail, so the
numerically large mean shift is swamped by within-group spread
(Appendix~\ref{app:distributions}, Figure~\ref{figures:distributions_mimic}). This pattern is consistent across inference methods: a deep ensemble
replicates the same ranking ($\sum_k C_k{=}0.753 > \text{MI}{=}0.752
> EU_{\text{var}}{=}0.750$) on MIMIC-III, with $EU_{\text{var}}$
again achieving the highest ratio ($2.72\times$) and the lowest AUROC (Appendix~\ref{app:ood-ensemble}). 


\begin{table}[!ht]
\centering
\caption{OoD detection: AUROC and OoD/ID mean-uncertainty ratio. Best in \textbf{bold}.}
\label{tab:ood-main}
\resizebox{\columnwidth}{!}{
\begin{tabular}{llcccc}
\toprule
 & & \multicolumn{2}{c}{FashionMNIST $\to$ KMNIST} & \multicolumn{2}{c}{MIMIC-III ICU $\to$ Newborn} \\
\cmidrule(lr){3-4} \cmidrule(lr){5-6}
Family & Method & AUROC$\uparrow$ & Ratio$\uparrow$ & AUROC$\uparrow$ & Ratio$\uparrow$ \\
\midrule
\multirow{3}{*}{Baselines}
  & Neg.\ MSP & $0.665 \pm 0.013$ & $2.07$ & $0.688 \pm 0.030$ & $1.24$ \\
  & MI & $0.724 \pm 0.009$ & $5.92$ & $0.802 \pm 0.004$ & $1.61$ \\
  & $EU_{\text{var}}$ & $0.710 \pm 0.010$ & $5.92$ & $0.778 \pm 0.015$ & $\mathbf{1.71}$ \\
\midrule
Proposed
  & $\sum_k C_k$ & $\mathbf{0.735 \pm 0.009}$ & $\mathbf{6.43}$ & $\mathbf{0.815 \pm 0.017}$ & $1.62$ \\
\bottomrule
\end{tabular}}
\end{table}
\paragraph{Per-class decomposition.}
On the binary MIMIC-III task, the decomposition $\sum_k C_k = C_0 + C_1$
reveals that distributional shift affects the two classes asymmetrically
($2.15\times$ for survival vs.\ $1.30\times$ for mortality): the larger OoD
signal resides in the \emph{non-critical} class, so $C_1$-only yields
AUROC $0.740$, well below $\sum_k C_k$ at $0.815$.
On FashionMNIST the pattern is broader but still uneven: all ten per-class
$C_k$ values increase from ID to OoD, but with substantially different
magnitudes across classes (Appendix~\ref{app:perclass_fmnist},
Figure~\ref{fig:ck_fmnist}). Individual $C_k$ scores yield non-trivial
per-class OoD separation, but remain below the aggregate $\sum_k C_k$,
indicating that KMNIST induces a broad structural shift whose strength is
distributed unevenly across the FashionMNIST label space.
Together, these two cases illustrate a task-dependent trade-off:
critical-class targeting improves selective prediction
(Section~\ref{sec:exp_dr}) but all-class aggregation is necessary for OoD
detection, where the locus of distributional shift is not known a priori and
can be either asymmetric (MIMIC-III) or uniform (FashionMNIST).

Table~\ref{tab:fmnist-perclass-auroc} reports the corresponding per-class
OoD scores for FashionMNIST$\to$KMNIST. Each individual $C_k$ gives
non-trivial ID/OoD separation, with AUROC ranging from $0.537$ for Ankle
boot to $0.683$ for Bag. The ranking broadly tracks the OoD/ID amplification:
classes with the largest increases, such as Bag, Trouser, and Sandal, also
yield among the strongest per-class detectors. At the same time, all
individual $C_k$ scores remain below the aggregate $\sum_k C_k$ AUROC of
$0.735 \pm 0.009$, confirming that this benchmark induces a broad shift for
which aggregation is optimal, while the per-class decomposition reveals where
the shift is strongest.
\begin{table}[!ht]
\centering
\caption{Per-class OoD detection on FashionMNIST$\to$KMNIST using each
$C_k(x)$ as an individual OoD score.}
\label{tab:fmnist-perclass-auroc}
\footnotesize
\resizebox{\columnwidth}{!}{
\begin{tabular}{lcc}
\toprule
Class & OoD/ID ratio & AUROC using $C_k$ \\
\midrule
T-shirt/top & $6.971 \pm 0.788$ & $0.670 \pm 0.027$ \\
Trouser     & $16.046 \pm 2.370$ & $0.656 \pm 0.024$ \\
Pullover    & $3.477 \pm 0.549$ & $0.578 \pm 0.020$ \\
Dress       & $9.428 \pm 1.005$ & $0.662 \pm 0.020$ \\
Coat        & $3.481 \pm 0.602$ & $0.577 \pm 0.026$ \\
Sandal      & $14.729 \pm 3.755$ & $0.650 \pm 0.019$ \\
Shirt       & $4.467 \pm 0.337$ & $0.615 \pm 0.025$ \\
Sneaker     & $3.494 \pm 0.906$ & $0.549 \pm 0.054$ \\
Bag         & $18.854 \pm 2.173$ & $0.683 \pm 0.030$ \\
Ankle boot  & $2.921 \pm 1.439$ & $0.537 \pm 0.050$ \\
\bottomrule
\end{tabular}}
\end{table}

\section{Epistemic Sensitivity to Data Quality}\label{sec:exp_disentangle}
We assess whether $\sum_k \Ck$ remains insensitive to irreducible aleatoric
noise under controlled label-noise injection.

\subsection{Setup}

\paragraph{Data.}
We use Fashion-MNIST ($K{=}10$, 60{,}000 images) and CIFAR-10 ($K{=}10$, 50{,}000 images) \citep{krizhevsky2009learning}.
For each dataset we inject symmetric label noise at rates $\alpha \in \{0.1, 0.2, 0.3, 0.4, 0.5\}$: a fraction $\alpha$ of training labels is replaced uniformly at random, increasing aleatoric uncertainty without altering the input distribution.

\paragraph{Model.}
We compare end-to-end Bayesian training with transfer learning using low-rank Gaussian posteriors ($S{=}50$ MC~samples).
Fashion-MNIST: fully Bayesian; two hidden layers ($400$ units, rank~$15$).
CIFAR-10: (i) end-to-end Conv~$32{\to}64{\to}128$, rank~$32$; (ii) frozen ResNet-18 with Bayesian head.
Full details, alternative posteriors, and CIFAR-100 results in Appendix~\ref{app:disentangle-model}.

\paragraph{Protocol.}
For each noise level $\alpha$ we compute the mean aleatoric uncertainty $\bar{U}_a(\alpha) = \frac{1}{N}\sum_i \E[\Ent(\vp_i)]$ and the mean epistemic uncertainty $\bar{U}_e(\alpha)$ using either MI or $\sum_k \Ck$ over the clean test set.
The relative disentanglement ratio
\begin{equation}\label{eq:disentangle-ratio}
  R_{\mathrm{rel}}(\alpha)
  \;=\;
  \frac{(\bar{U}_e(\alpha) - \bar{U}_e(0))\,/\,\bar{U}_e(0)}
       {(\bar{U}_a(\alpha) - \bar{U}_a(0))\,/\,\bar{U}_a(0)}
\end{equation}
measures the percentage increase in epistemic uncertainty per percentage increase in aleatoric uncertainty.
$R_{\mathrm{rel}} = 0$ indicates perfect disentanglement; $R_{\mathrm{rel}} = 1$ indicates full entanglement; values above $0.3$ signal substantial contamination (see Appendix~\ref{app:disentangle-transfer} for the motivation behind the relative formulation over the absolute ratio).
\subsection{Results}
Table~\ref{tab:disentangle-main} reports $|R_{\mathrm{rel}}|$ for low-rank Bayesian models under three configurations.
Under end-to-end training, both metrics are near-perfectly disentangled ($|R_{\mathrm{rel}}| \ll 0.3$ everywhere), confirming that end-to-end Bayesian training with low-rank posteriors is the regime that produces the least entangled uncertainty estimates.
It is also the regime under which the second-order Taylor expansion most accurately approximates exact MI: at $\alpha{=}0$, $\sum_k \Ck\,/\,\text{MI}$ is $1.00\times$ on Fashion-MNIST and $1.02\times$ on CIFAR-10 (Table~\ref{tab:baseline_comparison}).
$\sum_k \Ck$ achieves lower $|R_{\mathrm{rel}}|$ in 9 of the 10 end-to-end conditions, the exception being CIFAR-10 at $\alpha{=}0.1$ where both values remain below $0.05$.
The rightmost block contrasts the \emph{same model type and dataset} (CIFAR-10, low-rank) under transfer learning: ratios jump to $0.74$--$1.97$, an order-of-magnitude degradation that isolates the training regime as the causal factor.
\begin{table}[!ht]
\caption{$|R_{\mathrm{rel}}(\alpha)|$ for low-rank Bayesian models.
Lower is better; $\dagger$\,=\,$\sum_k \Ck$ less entangled.
The first two blocks use end-to-end training (e2e); the third uses a frozen
ResNet-50 backbone with a low-rank Bayesian head (TL).}
\label{tab:disentangle-main}
\footnotesize
\centering
\begin{tabular*}{\columnwidth}{@{\extracolsep{\fill}}lcccccc@{}}
\toprule
& \multicolumn{2}{c}{FMNIST (e2e)} & \multicolumn{2}{c}{CIFAR-10 (e2e)} & \multicolumn{2}{c}{CIFAR-10 (TL)} \\
\cmidrule(lr){2-3} \cmidrule(lr){4-5} \cmidrule(lr){6-7}
$\alpha$ & MI & $\sum_k \Ck$ & MI & $\sum_k \Ck$ & MI & $\sum_k \Ck$ \\
\midrule
0.1 & 0.023 & $\mathbf{0.005}^{\dagger}$ & $\mathbf{0.040}$ & 0.047 & $\mathbf{1.830}$ & 1.970 \\
0.2 & 0.010 & $\mathbf{0.003}^{\dagger}$ & 0.012 & $\mathbf{0.007}^{\dagger}$ & $\mathbf{1.344}$ & 1.435 \\
0.3 & 0.008 & $\mathbf{0.002}^{\dagger}$ & 0.009 & $\mathbf{0.004}^{\dagger}$ & $\mathbf{0.991}$ & 1.039 \\
0.4 & 0.005 & $\mathbf{0.003}^{\dagger}$ & 0.020 & $\mathbf{0.016}^{\dagger}$ & $\mathbf{0.832}$ & 0.868 \\
0.5 & 0.009 & $\mathbf{0.002}^{\dagger}$ & 0.020 & $\mathbf{0.017}^{\dagger}$ & $\mathbf{0.737}$ & 0.760 \\
\bottomrule
\end{tabular*}
\end{table}
\paragraph{The posterior approximation matters as much as the metric.}
Table~\ref{tab:disentangle-main} identifies training regime as a key
factor.  Under end-to-end training, posterior expressiveness tracks
entanglement: on Fashion-MNIST the low-rank model achieves a mean
$|R_{\mathrm{rel}}|$ below $0.01$ for both metrics, while the
full-rank model reaches $0.11$\,(MI)\,/\,$0.10$\,($\sum_k \Ck$) and
MC~dropout $0.10$\,/\,$0.06$ respectively
(Appendix~\ref{app:disentangle-transfer}).
Transfer learning disrupts this hierarchy: even the low-rank model
produces $|R_{\mathrm{rel}}|$ exceeding $0.73$ on CIFAR-10, an order
of magnitude worse than any end-to-end configuration, because the
frozen backbone supplies features never optimised for calibrated
predictive variance.
This suggests that the high entanglement reported by
\citet{mucsanyi2024benchmarking} reflects the training regime as much
as the decomposition formula.  When uncertainty is not propagated
through the full network, none of the three posterior families we
evaluate can compensate, motivating the skewness diagnostic of
Section~\ref{sec:skewness} to flag such configurations.
Across end-to-end settings (Figure~\ref{fig:disentangle-ratios-e2e}),
$\sum_k \Ck$ achieves lower $|R_{\mathrm{rel}}|$ than MI in 19 of 20
conditions. 
Under transfer learning both metrics degrade and MI edges ahead,
coinciding with baseline ratios $\sum_k \Ck\,/\,\text{MI}$ above
roughly $1.6\times$ (Table~\ref{tab:baseline_comparison}), i.e.\
where the Taylor residual is largest.
Full transfer-learning results and an analysis of the rank
constraint's effect on uncertainty scale appear in
Appendix~\ref{app:disentangle-transfer}.
\paragraph{Sensitivity to class cardinality.}
On CIFAR-100 \citep{krizhevsky2009learning}, the $1/\mu_k$ normalisation amplifies contributions from
low-probability classes, inflating $\sum_k \Ck$ relative to MI by
$1.17\times$ for the low-rank model and $1.89\times$ for MC~dropout.
The $\mathcal{O}(K^2)$ scaling analysis and mitigation strategies are
detailed in Appendix~\ref{app:disentangle-scaling}.
\section{Discussion and Conclusion}\label{sec:discussion}

We presented a per-class epistemic uncertainty decomposition $\vC(x)$ derived from a second-order Taylor expansion of MI, where the $1/\mu_k$ normalisation arises naturally from the entropy Hessian and counteracts boundary suppression.
The decomposition reduces to a simple ratio of variance to mean per class, adds negligible cost to any posterior sampling pipeline, and closely tracks scalar MI by summation.

The three experiments exposed a consistent theme: the value of $\vC(x)$ lies not in being a better scalar summary but in enabling class-specific queries that scalar metrics cannot express.
For selective prediction, targeting critical-class $C_k$ outperformed all scalar and variance-based baselines; for OoD detection, the per-class view revealed that distributional shift can concentrate asymmetrically across classes, a structure invisible to any aggregate score; for disentanglement, the decomposition provided a lens to diagnose when and why epistemic and aleatoric signals become entangled.

A finding we did not anticipate is that the posterior approximation shaped metric behaviour at least as strongly as the metric itself.
Switching from variational inference to MC~dropout reversed the ranking of our deferral policies, because dropout inflated skewness precisely for the rare classes where the Taylor expansion is most sensitive.
The disentanglement experiments sharpened this further: freezing a pretrained backbone and attaching a Bayesian head degraded AU/EU separation by over an order of magnitude compared to end-to-end training, even when the Bayesian component was identical.
This raises a question for the growing literature on post-hoc Bayesian methods: if features are learned without any posterior objective, the resulting variance structure may not support meaningful epistemic attribution regardless of the last-layer treatment.
We do not suggest post-hoc methods lack practical value, but the quality of their uncertainty outputs deserves scrutiny commensurate with the attention given to their scalability.

We note several limitations.
The additive link $\sum_k \Ck \approx \MI$ loosens under high skewness over low-probability classes, since the third-order remainder scales as $1/\mu_k^2$; $\rho_k$ diagnoses but does not correct this, and the CBEC fallback requires domain knowledge to define class partitions.
The $1/\mu_k$ normalisation introduces $\mathcal{O}(K^2)$ scaling of the
aggregate, necessitating truncation or reweighting in high-cardinality
settings, though the per-class vector itself remains fully interpretable.
Our CIFAR-100 experiments probe this cardinality effect in the disentanglement
setting, but we do not evaluate high-cardinality OoD detection; ImageNet- or
iNaturalist-scale class-conditional shifts remain an important stress test for
future work.

All models use MFVI, low-rank Gaussian VI, MC~dropout, or deep ensembles. 
Characterising how $C_k$ behaves under other approximate inference
schemes such as Laplace approximations~\citep{daxberger2021laplace}
or stochastic weight averaging~\citep{maddox2019simple} is a natural
next step.

The decomposition is not tied to any specific posterior family or architecture.
Future work includes extending $\vC(x)$ to structured prediction tasks where class identity is less cleanly defined, integrating it into active learning acquisition functions that target specific classes, and combining low-rank ensembles with per-class attribution for richer uncertainty profiles.
More fundamentally, making end-to-end Bayesian training more practical and scalable strikes us as a more promising path than refining post-hoc corrections, and the per-class decomposition is designed to complement that direction.\footnote{Code and trained models are available \href{https://github.com/arradiat/Decomposing-epistemic-uncertainty-UAI-2026}{here} 
.}

\clearpage

\bibliography{biblio}

@inproceedings{gal2016dropout,
  title     = {Dropout as a {B}ayesian Approximation: Representing Model Uncertainty in Deep Learning},
  author    = {Gal, Yarin and Ghahramani, Zoubin},
  booktitle = {Proceedings of the 33rd International Conference on Machine Learning (ICML)},
  pages     = {1050--1059},
  year      = {2016},
  publisher = {PMLR}
}

@inproceedings{kendall2017uncertainties,
  title     = {What Uncertainties Do We Need in {B}ayesian Deep Learning for Computer Vision?},
  author    = {Kendall, Alex and Gal, Yarin},
  booktitle = {Advances in Neural Information Processing Systems (NeurIPS)},
  volume    = {30},
  pages     = {5574--5584},
  year      = {2017}
}

@inproceedings{lakshminarayanan2017simple,
  title     = {Simple and Scalable Predictive Uncertainty Estimation Using Deep Ensembles},
  author    = {Lakshminarayanan, Balaji and Pritzel, Alexander and Blundell, Charles},
  booktitle = {Advances in Neural Information Processing Systems (NeurIPS)},
  volume    = {30},
  year      = {2017}
}

@inproceedings{gustafsson2020evaluating,
  title     = {Evaluating Scalable {B}ayesian Deep Learning Methods for Robust Computer Vision},
  author    = {Gustafsson, Fredrik K. and Danelljan, Martin and Sch{\"o}n, Thomas B.},
  booktitle = {Proceedings of the IEEE/CVF Conference on Computer Vision and Pattern Recognition Workshops (CVPRW)},
  pages     = {318--319},
  year      = {2020}
}

@inproceedings{sale2024labelwise,
  title     = {Label-Wise Aleatoric and Epistemic Uncertainty Quantification},
  author    = {Sale, Yusuf and Hofman, Paul and L{\"o}hr, Timo and Wimmer, Lisa and Nagler, Thomas and H{\"u}llermeier, Eyke},
  booktitle = {Proceedings of the 40th Conference on Uncertainty in Artificial Intelligence (UAI)},
  pages     = {3159--3179},
  volume    = {244},
  series    = {Proceedings of Machine Learning Research},
  publisher = {PMLR},
  year      = {2024}
}

@inproceedings{duan2024evidential,
  title     = {Evidential Uncertainty Quantification: A Variance-Based Perspective},
  author    = {Duan, Ruxiao and Caffo, Brian and Bai, Harrison X. and Sair, Haris I. and Jones, Craig},
  booktitle = {Proceedings of the IEEE/CVF Winter Conference on Applications of Computer Vision (WACV)},
  pages     = {2132--2141},
  year      = {2024}
}

@inproceedings{wimmer2023quantifying,
  title     = {Quantifying Aleatoric and Epistemic Uncertainty in Machine Learning: Are Conditional Entropy and Mutual Information Appropriate Measures?},
  author    = {Wimmer, Lisa and Sale, Yusuf and Hofman, Paul and Bischl, Bernd and H{\"u}llermeier, Eyke},
  booktitle = {Proceedings of the 39th Conference on Uncertainty in Artificial Intelligence (UAI)},
  pages     = {2282--2292},
  volume    = {216},
  series    = {Proceedings of Machine Learning Research},
  publisher = {PMLR},
  year      = {2023}
}

@inproceedings{khan2019striking,
  title={Striking the right balance with uncertainty},
  author={Khan, Salman and Hayat, Munawar and Zamir, Syed Waqas and Shen, Jianbing and Shao, Ling},
  booktitle={Proceedings of the IEEE/CVF Conference on Computer Vision and Pattern Recognition},
  pages={103--112},
  year={2019}
}

@article{baltaci2023class,
  title   = {Class Uncertainty: A Measure to Mitigate Class Imbalance},
  author  = {Baltaci, Zeynep Sonat and Oksuz, Kemal and Kuzucu, Selim and Tezoren, Kivanc and Konar, Berkin Kerim and Ozkan, Alpay and Akbas, Emre and Kalkan, Sinan},
  journal = {arXiv preprint arXiv:2311.14090},
  year    = {2023}
}

@inproceedings{sensoy2018evidential,
 author = {Sensoy, Murat and Kaplan, Lance and Kandemir, Melih},
 booktitle = {Advances in Neural Information Processing Systems},
 editor = {S. Bengio and H. Wallach and H. Larochelle and K. Grauman and N. Cesa-Bianchi and R. Garnett},
 pages = {},
 publisher = {Curran Associates, Inc.},
 title = {Evidential Deep Learning to Quantify Classification Uncertainty},
 url = {https://proceedings.neurips.cc/paper_files/paper/2018/file/a981f2b708044d6fb4a71a1463242520-Paper.pdf},
 volume = {31},
 year = {2018}
}

@inproceedings{malinin_predictive_2018,
	title = {Predictive Uncertainty Estimation via Prior Networks},
	volume = {31},
	url = {https://proceedings.neurips.cc/paper_files/paper/2018/file/3ea2db50e62ceefceaf70a9d9a56a6f4-Paper.pdf},
	booktitle = {Advances in Neural Information Processing Systems},
	publisher = {Curran Associates, Inc.},
	author = {Malinin, Andrey and Gales, Mark},
	editor = {Bengio, S. and Wallach, H. and Larochelle, H. and Grauman, K. and Cesa-Bianchi, N. and Garnett, R.},
	year = {2018},
}

@article{johnson2016mimic,
  title   = {{MIMIC-III}, a freely accessible critical care database},
  author  = {Johnson, Alistair E. W. and Pollard, Tom J. and Shen, Lu and Lehman, Li-wei H. and Feng, Mengling and Ghassemi, Mohammad and Moody, Benjamin and Szolovits, Peter and Celi, Leo Anthony and Mark, Roger G.},
  journal = {Scientific Data},
  volume  = {3},
  number  = {1},
  pages   = {160035},
  year    = {2016}
}

@article{clanuwat2018deep,
  title={Deep learning for classical japanese literature},
  author={Clanuwat, Tarin and Bober-Irizar, Mikel and Kitamoto, Asanobu and Lamb, Alex and Yamamoto, Kazuaki and Ha, David},
  journal={arXiv preprint arXiv:1812.01718},
  year={2018}
}

@article{xiao2017fashion,
  title={Fashion-mnist: a novel image dataset for benchmarking machine learning algorithms},
  author={Xiao, Han and Rasul, Kashif and Vollgraf, Roland},
  journal={arXiv preprint arXiv:1708.07747},
  year={2017}
}

@inproceedings{depeweg2018decomposition,
  title={Decomposition of uncertainty in Bayesian deep learning for efficient and risk-sensitive learning},
  author={Depeweg, Stefan and Hernandez-Lobato, Jose-Miguel and Doshi-Velez, Finale and Udluft, Steffen},
  booktitle={International conference on machine learning},
  pages={1184--1193},
  year={2018},
  organization={PMLR}
}

@inproceedings{gal2017deep,
  title={{D}eep {B}ayesian Active Learning with Image Data},
  author={Gal, Yarin and Islam, Riashat and Ghahramani, Zoubin},
  booktitle={International conference on machine learning},
  pages={1183--1192},
  year={2017},
  organization={PMLR}
}

@misc{dejong2026,
      title={Measuring Orthogonality as the Blind-Spot of Uncertainty Disentanglement}, 
      author={Ivo Pascal de Jong and Andreea Ioana Sburlea and Matthia Sabatelli and Matias Valdenegro-Toro},
      year={2026},
      eprint={2408.12175},
      archivePrefix={arXiv},
      primaryClass={cs.LG},
      url={https://arxiv.org/abs/2408.12175}, 
}

@inproceedings{mucsanyi2024benchmarking,
  title={Benchmarking uncertainty disentanglement: specialized uncertainties for specialized tasks},
  author={Mucs{\'a}nyi, B{\'a}lint and Kirchhof, Michael and Oh, Seong Joon},
  booktitle={Proceedings of the 38th International Conference on Neural Information Processing Systems},
  pages={50972--51038},
  year={2024}
}

@inproceedings{geifman2017selective,
  title={Selective Classification for Deep Neural Networks},
  author={Geifman, Yonatan and El-Yaniv, Ran},
  booktitle={Advances in Neural Information Processing Systems (NeurIPS)},
  volume={30},
  year={2017}
}

@inproceedings{hendrycks2017,
  title     = {A Baseline for Detecting Misclassified and Out-of-Distribution Examples in Neural Networks},
  author    = {Hendrycks, Dan and Gimpel, Kevin},
  booktitle = {International Conference on Learning Representations (ICLR)},
  year      = {2017}
}

@article{camarasa2021quantitative,
   title={A Quantitative Comparison of Epistemic Uncertainty Maps Applied to Multi-Class Segmentation},
   volume={1},
   ISSN={2766-905X},
   url={http://dx.doi.org/10.59275/j.melba.2021-ec49},
   DOI={10.59275/j.melba.2021-ec49},
   number={UNSURE2020},
   journal={Machine Learning for Biomedical Imaging},
   publisher={Machine Learning for Biomedical Imaging},
   author={Camarasa, Robin and Bos, Daniel and Hendrikse, Jeroen and Nederkoorn, Paul and Kooi, M. Eline and van der Lugt, Aad and de Bruijne, Marleen},
   year={2021},
   month=sep, pages={1–39} }

@inproceedings{daxberger2021laplace,
  title={Laplace redux--effortless Bayesian deep learning},
  author={Daxberger, Erik and Kristiadi, Agustinus and Immer, Alexander and Eschenhagen, Runa and Bauer, Matthias and Hennig, Philipp},
  booktitle={Proceedings of the 35th International Conference on Neural Information Processing Systems},
  pages={20089--20103},
  year={2021}
}

@inproceedings{maddox2019simple,
  title={A simple baseline for Bayesian uncertainty in deep learning},
  author={Maddox, Wesley J and Garipov, Timur and Izmailov, Pavel and Vetrov, Dmitry and Wilson, Andrew Gordon},
  booktitle={Proceedings of the 33rd International Conference on Neural Information Processing Systems},
  pages={13153--13164},
  year={2019}
}

@article{blei2017variational,
  title={Variational inference: A review for statisticians},
  author={Blei, David M and Kucukelbir, Alp and McAuliffe, Jon D},
  journal={Journal of the American statistical Association},
  volume={112},
  number={518},
  pages={859--877},
  year={2017},
  publisher={Taylor \& Francis}
}

@inbook{Turner_Sahani_2011, place={Cambridge}, title={Two problems with variational expectation maximisation for time series models}, booktitle={Bayesian Time Series Models}, publisher={Cambridge University Press}, author={Turner, Richard Eric and Sahani, Maneesh}, year={2011}, pages={104–124}}

@inproceedings{smith2018understanding,
  title     = {Understanding Measures of Uncertainty for Adversarial Example Detection},
  author    = {Smith, Lewis and Gal, Yarin},
  booktitle = {Proceedings of the 34th Conference on Uncertainty in Artificial Intelligence (UAI)},
  pages     = {560--569},
  year      = {2018}
}

@article{wilkinson2003proposed,
  title={Proposed international clinical diabetic retinopathy and diabetic macular edema disease severity scales},
  author={Wilkinson, Charles P and Ferris III, Frederick L and Klein, Ronald E and Lee, Paul P and Agardh, Carl David and Davis, Matthew and Dills, Diana and Kampik, Anselm and Pararajasegaram, Rangasamy and Verdaguer, Juan T and others},
  journal={Ophthalmology},
  volume={110},
  number={9},
  pages={1677--1682},
  year={2003},
  publisher={Elsevier}
}

@article{decenciere2014feedback,
  title={Feedback on a publicly distributed image database: the Messidor database},
  author={Decenci{\`e}re, Etienne and Zhang, Xiwei and Cazuguel, Guy and Lay, Bruno and Cochener, B{\'e}atrice and Trone, Caroline and Gain, Philippe and Ord{\'o}{\~n}ez-Varela, John-Richard and Massin, Pascale and Erginay, Ali and others},
  journal={Image Analysis \& Stereology},
  pages={231--234},
  year={2014}
}

@article{cuadros2009eyepacs,
  title={EyePACS: an adaptable telemedicine system for diabetic retinopathy screening},
  author={Cuadros, Jorge and Bresnick, George},
  journal={Journal of diabetes science and technology},
  volume={3},
  number={3},
  pages={509--516},
  year={2009},
  publisher={SAGE Publications}
}

@inproceedings{toure2026singular,
  title     = {Singular Bayesian Neural Networks},
  author    = {Toure, Mame Diarra and Stephens, David A.},
  booktitle = {International Conference on Machine Learning},
  year      = {2026},
  url={https://openreview.net/forum?id=rKD8HtgY8G}
}

@inproceedings{ulmer2021know,
  title={Know your limits: Uncertainty estimation with relu classifiers fails at reliable ood detection},
  author={Ulmer, Dennis and Cin{\`a}, Giovanni},
  booktitle={Uncertainty in artificial intelligence},
  pages={1766--1776},
  year={2021},
  organization={PMLR}
}

@article{krizhevsky2009learning,
  title={Learning multiple layers of features from tiny images},
  author={Krizhevsky, Alex and Hinton, Geoffrey and others},
  year={2009},
  publisher={Toronto, ON, Canada}
}

@article{elyaniw2010,
  title={On the Foundations of Noise-Free Selective Classification},
  author={El-Yaniv, Ran and Wiener, Yair},
  journal={JMLR},
  volume={11},
  pages={1605--1641},
  year={2010}
}

@article{houlsby2011bayesian,
  title={Bayesian active learning for classification and preference learning},
  author={Houlsby, Neil and Husz{\'a}r, Ferenc and Ghahramani, Zoubin and Lengyel, M{\'a}t{\'e}},
  journal={arXiv preprint arXiv:1112.5745},
  year={2011}
}

@misc{aptos2019,
  title        = {{APTOS} 2019 Blindness Detection},
  author       = {{Asia Pacific Tele-Ophthalmology Society}},
  year         = {2019},
  howpublished = {Kaggle Competition},
  url          = {https://www.kaggle.com/c/aptos2019-blindness-detection}
}

@inproceedings{tan2019efficientnet,
  title     = {{EfficientNet}: Rethinking Model Scaling for Convolutional Neural Networks},
  author    = {Tan, Mingxing and Le, Quoc V.},
  booktitle = {Proceedings of the 36th International Conference on Machine Learning (ICML)},
  pages     = {6105--6114},
  year      = {2019},
  publisher = {PMLR}
}

@inproceedings{blundell2015weight,
  title     = {Weight Uncertainty in Neural Networks},
  author    = {Blundell, Charles and Cornebise, Julien and Kavukcuoglu, Koray and Wierstra, Daan},
  booktitle = {Proceedings of the 32nd International Conference on Machine Learning (ICML)},
  pages     = {1613--1622},
  year      = {2015},
  publisher = {PMLR}
}

@article{Fano1947,
  author  = {Ugo Fano},
  title   = {Ionization Yield of Radiations. {II}. {T}he Fluctuations of the Number of Ions},
  journal = {Physical Review},
  volume  = {72},
  number  = {1},
  pages   = {26--29},
  year    = {1947},
}

@book{CoxLewis1966,
  author    = {David R. Cox and Peter A. W. Lewis},
  title     = {The Statistical Analysis of Series of Events},
  publisher = {Methuen},
  address   = {London},
  year      = {1966},
}

\newpage

\onecolumn

\title{Not Just How Much, But Where: Decomposing Epistemic Uncertainty into Per-Class Contributions\\(Supplementary Material)}
\maketitle

\appendix

%
\section{Background and Related Work}
\label{app:related}

This appendix provides the detailed literature context that supports
the compressed survey in the main text.
We organise the discussion into four parts: (i)~the axiomatic framework
and the label-wise variance approach; (ii)~evidential and Dirichlet-based
per-class uncertainty; (iii)~class-imbalance and multi-class segmentation
perspectives; and (iv)~AU/EU disentanglement, scalar epistemic baselines,
and the orthogonality criterion.

\subsection{Axiomatic Framework and the Label-Wise Approach}
\label{app:related:labelwise}

\paragraph{Wimmer et al.\ (2023) -- Axioms for epistemic uncertainty.}
\citet{wimmer2023quantifying} formalise a set of axioms A0--A5 for
second-order uncertainty measures in $K$-class classification.
An uncertainty measure is defined over a \emph{second-order distribution}
$Q \in \Delta^{(2)}_K$, a distribution over probability simplices
$\vp \in \Delta^K$, with mean $\vmu = \E_Q[\vp]$.
The five axioms relevant to our analysis are:

\begin{itemize}
  \item \textbf{A0} (Non-negativity): $\mathrm{EU}(Q) \geq 0$.
  \item \textbf{A1} (Vanishing at certainty): $\mathrm{EU}(Q) = 0$
    if and only if $Q = \delta_{\boldsymbol{\theta}}$ for some
    $\boldsymbol{\theta} \in \Delta^K$.
  \item \textbf{A2} (Maximum at uniform): $\mathrm{EU}$ is maximised when
    $Q$ is the uniform distribution on $\Delta^K$.
  \item \textbf{A3} (Monotone under mean-preserving spread): if $Q'$ is
    a mean-preserving spread of $Q$ (i.e.\ $Q'$ has the same mean but
    weakly larger variance in every direction), then
    $\mathrm{EU}(Q') \geq \mathrm{EU}(Q)$.
  \item \textbf{A5} (Location-shift invariance): a spread-preserving
    shift of $Q$'s mean along the simplex does not change
    $\mathrm{EU}(Q)$.
\end{itemize}

Wimmer et al.\ show that mutual information satisfies A0 and A1, but
violates A2 (their Proposition~2), A3 (Proposition~4; a counterexample
exists already for $K{=}2$), and A5 (Proposition~3).
No single measure simultaneously satisfies all five axioms.
Our aggregate $\sum_k C_k$ satisfies A0, A1, and A3 but violates A2 and A5;
as Corollary~\ref{cor:a5-tradeoff} establishes, the A5 violation is the
\emph{precise mechanism} that corrects boundary suppression.

\paragraph{Sale et al.\ (2024) -- Label-wise decomposition.}
\citet{sale2024labelwise} introduce a label-wise framework that reduces a
$K$-class problem to $K$ binary sub-problems via one-vs-rest binarisation.
For each class $k$, the true class probability $\Theta_k$ is treated as
a scalar random variable following a second-order distribution $Q_k$
(capturing model uncertainty over possible probability values for class~$k$).
Per-class total, aleatoric, and epistemic uncertainty are defined via
proper scoring rules; under squared loss the epistemic measure is
\begin{equation}
  \mathrm{EU}_k \;=\; \Var(\Theta_k),
\end{equation}
and the global quantities are sums: $\mathrm{EU} = \sum_k \Var(\Theta_k)$.
Sale et al.\ prove (their Theorem~3.2) that these variance-based measures
satisfy axioms A0, A1, A3 (strict version), A4 (strict version), A5, A6,
and A7, a stronger axiomatic profile than entropy-based measures, which
fail A5~\citep{wimmer2023quantifying}.

Two limitations motivate our normalised alternative.
First, \emph{boundary suppression}: because $p_k \in [0,1]$, the
constraint $\Var[p_k] \leq \mu_k(1-\mu_k)$ (Lemma~\ref{lem:boundary})
forces $\mathrm{EU}_k \to 0$ as $\mu_k \to 0$, regardless of how
strongly the stochastic forward passes disagree.
This ``edge-squeeze'' masks epistemic uncertainty for precisely the rare,
low-base-rate classes where it matters most, rendering $\mathrm{EU}_k$
incomparable across classes with different mean probabilities.
Second, the framework operates on marginals $Q_k$ in isolation;
consequently, the global sum $\sum_k H(\bar{\theta}_k)$ does not recover
the standard multiclass mutual information decomposition, preventing
direct additive attribution of MI to individual classes.
Our $C_k$ addresses both limitations: the $1/\mu_k$ normalisation
removes the base-rate confound, and by construction
$\sum_k C_k \approx \MI$ (Section~\ref{sec:definition}).

\subsection{Evidential and Dirichlet-Based Per-Class Uncertainty}
\label{app:related:evidential}

\paragraph{Sensoy et al.\ (2018) and Malinin \& Gales (2018).}
\citet{sensoy2018evidential} propose Evidential Deep Learning (EDL),
training a neural network to output Dirichlet concentration parameters
$\boldsymbol{\alpha} = (\alpha_1,\dots,\alpha_K)$ in a single forward
pass, thereby placing a distribution directly over the simplex.
\citet{malinin_predictive_2018} independently introduced Prior Networks
with the same Dirichlet parameterisation but a different training objective
(reverse KL from a target Dirichlet), and formalised the decomposition of
predictive uncertainty into distributional uncertainty (epistemic) and
data uncertainty (aleatoric) within this framework.
Both approaches provide per-class quantities derived from the Dirichlet
parameters without requiring multiple stochastic forward passes.

\paragraph{Duan et al.\ (2024).}
\citet{duan2024evidential} extend the evidential paradigm by applying the
law of total covariance to the Dirichlet-Categorical hierarchy.
For a Dirichlet with concentration $\boldsymbol{\alpha}$ and strength
$\alpha_0 = \sum_k \alpha_k$, they derive the full per-class epistemic
covariance matrix (their Equation~21)
\begin{equation}
  \mathrm{Cov}^{\mathrm{epis}}_{c,c'}
  \;=\;
  \frac{1}{\alpha_0+1}
  \bigl(\bar{\mu}_c \,\mathbf{1}_{c=c'} - \bar{\mu}_c\bar{\mu}_{c'}\bigr),
\end{equation}
whose diagonal gives the per-class epistemic variance (Equation~24)
$U_c^{\mathrm{epis}} = \frac{1}{\alpha_0+1}\bar{\mu}_c(1-\bar{\mu}_c)$,
and whose off-diagonal entries reveal pairwise class confusion.
The covariance structure therefore carries both marginal and joint
epistemic information in closed form.

However, the entire covariance matrix is determined solely by
$\bar{\boldsymbol{\mu}} = \boldsymbol{\alpha}/\alpha_0$ and the scalar
strength $\alpha_0$.
This means that two inputs with the same Dirichlet mean but very
different patterns of stochastic-forward-pass disagreement will receive
identical epistemic covariance matrices, a rigidity absent from
empirical MC-based estimates.
More fundamentally, placing a Dirichlet prior on the simplex is a strong
distributional assumption: the Dirichlet family is conjugate and
analytically convenient, but real posterior predictive distributions
under deep networks are rarely Dirichlet-shaped.
Our approach makes no distributional assumption and recovers
class-confusion information from the empirical covariance
$\widehat{\mathrm{Cov}}[\vp]$ computed directly from MC samples.

\subsection{Class Imbalance and Multi-Class Segmentation}
\label{app:related:imbalance}

\paragraph{Khan et al.\ (2019).}
\citet{khan2019striking} demonstrate that MC-dropout predictive
uncertainty (computed as the second moment of the softmax output across
dropout masks, following~\citealt{gal2016dropout}) correlates with both
the rarity of classes and the difficulty of individual samples.
They exploit this correlation to reshape classification margins: rare,
high-uncertainty classes receive larger margin penalties, improving
generalisation for under-represented classes across face verification,
attribute prediction, and classification tasks.
Importantly, the uncertainty used is the total predictive uncertainty
(the combined effect of aleatoric and epistemic sources via the
MC-dropout second moment); no explicit AU/EU decomposition is performed,
and the uncertainty is used as a \emph{class-level signal} derived from
the training set rather than an input-level epistemic vector.

\paragraph{Baltaci et al.\ (2023).}
\citet{baltaci2023class} propose Class Uncertainty (CU), defined as the
average predictive entropy across training examples within each class:
\begin{equation}
  \tilde{\mu}^U_c
  \;=\;
  \frac{1}{N_c}\sum_{i=1}^{N_c} u^{(i)},
  \qquad
  u^{(i)} = -\sum_{c' \in \mathcal{C}} \bar{p}_{c'}^{(i)} \log \bar{p}_{c'}^{(i)},
\end{equation}
where $\bar{p}_{c'}^{(i)}$ is the ensemble-averaged class probability for
sample~$i$ (their Equations~8--9); a normalised version
$\mu^U_c = \tilde{\mu}^U_c / \sum_{c'} \tilde{\mu}^U_{c'}$ is then used
as the class-imbalance weight (Equation~10).
CU captures \emph{semantic imbalance}, difficulty differences among
classes that go beyond mere class cardinality, and is shown to
correlate more strongly with class-wise test error than cardinality alone
(Spearman $\rho = 0.92$ vs.\ $0.82$).

CU has two limitations relative to our goal.
First, it is a \emph{dataset-level aggregate}: it aggregates over training
examples and is not defined for a single test input, so it cannot provide
input-level epistemic attribution.
Second, it uses predictive entropy, which conflates aleatoric and
epistemic components; no decomposition into EU and AU is performed.

\paragraph{Camarasa et al.\ (2021).}
\citet{camarasa2021quantitative} provide a systematic quantitative
comparison of epistemic uncertainty metrics in multi-class carotid artery
MR segmentation, evaluating one-vs-all entropy, mutual information, and
class-wise variance maps across both combined and class-specific settings.
Their key empirical finding is that raw class-wise variance underperforms
one-vs-all entropy in the class-specific setting (where it decorrelates
from misclassification rate), while entropy-based metrics maintain
stronger correlation with segmentation errors across all classes.
Lemma~\ref{lem:boundary} establishes this as a mathematical inevitability
rather than a dataset artefact: variance is structurally suppressed near
the simplex boundary regardless of actual posterior disagreement, making
it an unreliable epistemic signal precisely where the information would
be most valuable.

\subsection{AU/EU Disentanglement, Scalar Baselines, and
            the Orthogonality Criterion}
\label{app:related:disentangle}

\paragraph{Smith \& Gal (2018) -- Taylor connection and MC-dropout
failure modes.}
\citet{smith2018understanding} analyse the relationship between softmax
variance and mutual information through a Taylor expansion of the
logarithm.
They show (their Equation~10, Section~3.2) that the leading term in
the MI series is exactly proportional to the mean variance across
classes:
\begin{equation}
  \hat{I}
  \;=\;
  \sum_j \!\left(\tfrac{1}{T}\textstyle\sum_i p_{ij}^2\right) - \hat{p}_j^2
  + \mathcal{O}(\text{higher order}),
\end{equation}
which is, up to a multiplicative constant, the mean variance score
$\hat{\sigma}^2 = \frac{1}{C}\sum_j \Var[p_j]$.
This provides a theoretical explanation for why variance-based
uncertainty estimates empirically approximate MI and underscores the
Taylor-expansion approach underlying our $C_k$ decomposition; the key
difference is that our expansion is applied \emph{per class} rather than
averaged, yielding class-level attribution.

Smith \& Gal also document failure modes of MC-dropout: the approximate
posterior tends to underestimate uncertainty and can produce
spuriously confident regions in latent space where the model
assigns high confidence to inputs far from any training data.
We observe this behaviour in the DR experiment
(Section~\ref{sec:exp_dr}): under MC-dropout, the skewness diagnostic
$\rho_k$ is elevated for critical (rare) classes, degrading the
$C_k$ Taylor approximation and causing $C_{\mathrm{crit\_max}}$ to
underperform its BNN counterpart.

\paragraph{Houlsby et al.\ (2011) -- BALD and scalar epistemic uncertainty.}
\citet{houlsby2011bayesian} introduce Bayesian Active Learning by
Disagreement (BALD), originally for Gaussian processes and binary
classifiers, selecting data points that maximise the mutual
information $\MI(y;\boldsymbol{\omega}\mid x)$ between labels and model
parameters; this objective was subsequently adopted as the canonical
scalar epistemic measure in Bayesian deep learning~\citep{gal2017deep}.
BALD demonstrated its power for reducing posterior uncertainty in
active learning.
However, MI is class-agnostic: a high MI value signals that \emph{some}
class confusion is driving uncertainty, but does not identify
\emph{which} classes are the locus of the model's ignorance.
Our $C_k$ vector decomposes MI into per-class contributions, preserving
the additive structure of the BALD objective while exposing the
class-level anatomy of epistemic uncertainty.

\paragraph{Hendrycks \& Gimpel (2017) -- Maximum softmax probability.}
\citet{hendrycks2017} demonstrate that the maximum softmax
probability (MSP), $\max_k \mu_k$, provides a surprisingly effective
baseline for detecting both misclassified and out-of-distribution
examples, despite the fact that softmax probabilities are known to be
poorly calibrated as direct confidence estimates in deep networks.
We include the complementary uncertainty score
$1 - \mathrm{MSP} = 1 - \max_k \mu_k$ as a scalar baseline in our
deferral policy comparison (Table~\ref{tab:policies}).
This score is agnostic to the source of uncertainty (aleatoric vs.\
epistemic) and to class identity, making it the most basic comparator
against which our class-structured epistemic metrics are evaluated.

\paragraph{Mucs\'{a}nyi et al.\ (2024) -- Benchmarking disentanglement.}
\citet{mucsanyi2024benchmarking} conduct a large-scale benchmark of
uncertainty quantification methods and find that the aleatoric and
epistemic components produced by standard decomposition formulas
are highly rank-correlated across distributional methods:
$\rho(u_a, u_e) \geq 0.88$ on CIFAR-10 and
$\rho(u_a, u_e) \geq 0.78$ on ImageNet-1k.
They interpret this as evidence that current methods fail to disentangle
AU from EU, a conclusion that~\citet{dejong2026} subsequently show
requires qualification.

\paragraph{de Jong et al.\ (2026) -- Orthogonality as the correct
criterion.}
\citet{dejong2026} formalise what it means for AU and EU estimates to be
\emph{disentangled}.
They define two requirements, \emph{consistency} (each estimate
correlates with its corresponding ground-truth uncertainty) and
\emph{orthogonality} (each estimate is insensitive to changes in the
\emph{other} ground-truth uncertainty), and prove that both are
necessary and sufficient for disentanglement (their Theorems~3.1--3.3).

A central theoretical result (Theorem~3.2) is that on complex datasets
where aleatoric noise and epistemic difficulty co-occur, a high
correlation between the \emph{estimators} $u_a$ and $u_e$ is a
\emph{necessary consequence of estimator validity}, not a sign of
failure; \citet{mucsanyi2024benchmarking}'s high-$\rho$ finding is
therefore consistent with well-calibrated models on correlated datasets.
The correct diagnostic is orthogonality, formalised as their
Uncertainty Disentanglement Error (UDE).
Evaluating a range of methods across multiple datasets, de Jong et al.\
find that Deep Ensembles with Information-Theoretic (IT) disentanglement
achieve the best UDE, but that no current method fully satisfies both
consistency and orthogonality.

Critically for our work, \citet{dejong2026} show that disentanglement
behaviour is strongly modulated by whether the model is trained
from scratch or fine-tuned from a pretrained backbone (their Section~5,
Figure~7).
For a Deep Ensemble ResNet-18 pretrained on ImageNet-1k and fine-tuned
to CIFAR-10, the UDE is substantially higher than for the same
architecture trained from scratch ($0.545$ vs.\ $0.332$), primarily
because $|{\rho(u_e, U_a)}| = 0.985$ in the pretrained case, aleatoric
and epistemic uncertainty are almost fully conflated, compared to
$|{\rho(u_e, U_a)}| = 0.312$ for the from-scratch model.
This finding is directly relevant to our Section~\ref{sec:exp_disentangle},
where we observe the same qualitative pattern independently: the $C_k$
Taylor approximation degrades considerably when a frozen pretrained
backbone replaces end-to-end Bayesian training, and the AU/EU
entanglement of MI and $\sum_k C_k$ increases markedly in the transfer-learning
regime.

\paragraph{Depeweg et al.\ (2018) -- AU/EU decomposition via total variance.}
\citet{depeweg2018decomposition} formalise the decomposition of
predictive uncertainty into aleatoric and epistemic components for
latent-variable models.
Applying the law of total variance to the predictive distribution
$p(y \mid x)$ marginalised over the posterior $p(\omega \mid \mathcal{D})$
yields:
\begin{equation}
  \underbrace{\Var[y \mid x]}_{\text{total}}
  \;=\;
  \underbrace{\E_\omega[\Var[y \mid x, \omega]]}_{\text{aleatoric}}
  +
  \underbrace{\Var_\omega[\E[y \mid x, \omega]]}_{\text{epistemic}}.
\end{equation}
This identity is the regression analogue of the information-theoretic
decomposition MI $= H(\vmu) - \E[H(\vp)]$ used in classification;
it also underlies the label-wise variance measures of
\citet{sale2024labelwise}, which interpret $\Var(\Theta_k)$ as the
epistemic contribution from class $k$ via the same total-variance
decomposition applied marginally.
\section{Proofs}
\label{app:proofs}

\subsection{Proof of Theorem~\ref{thm:taylor} (MI Approximation)}
\label{app:proof:taylor}

\begin{proof}
Expand $\Ent(\vp^{(s)})$ to second order around $\vmu$:
\begin{equation}
  \Ent(\vp^{(s)})
  \;\approx\; \Ent(\vmu)
  + \nabla\Ent(\vmu)^\top(\vp^{(s)}-\vmu)
  + \tfrac{1}{2}(\vp^{(s)}-\vmu)^\top
    \nabla^2\Ent(\vmu)\,(\vp^{(s)}-\vmu).
\end{equation}
Take expectations over the $S$ passes.
The linear term vanishes by definition of $\vmu = \E[\vp^{(s)}]$.
By Lemma~\ref{lem:entropy_derivatives} the Hessian is
$\nabla^2\Ent(\vmu)=\mathrm{diag}(-1/\mu_1,\dots,-1/\mu_K)$,
so the quadratic term becomes
\begin{equation}
  \tfrac{1}{2}\,\E\!\left[
    (\vp-\vmu)^\top\nabla^2\Ent(\vmu)\,(\vp-\vmu)
  \right]
  = \tfrac{1}{2}\,\mathrm{tr}\!\left[
    \nabla^2\Ent(\vmu)\,\Covmat[\vp]
  \right]
  = -\frac{1}{2}\sum_{k=1}^{K}\frac{\Var[p_k]}{\mu_k}.
\end{equation}
The second equality uses the fact that the Hessian is diagonal, so the
trace selects only the diagonal elements of $\Covmat[\vp]$:
$\mathrm{tr}[\mathrm{diag}(-1/\mu_k)\,\Covmat[\vp]]= -\sum_k\Var[p_k]/\mu_k$.
Substituting into~\eqref{eq:decomp} gives
\begin{equation}
  \MI(y;\boldsymbol{\omega}\mid x)
  = \Ent(\vmu)-\E[\Ent(\vp)]
  \;\approx\; \frac{1}{2}\sum_{k=1}^{K}\frac{\Var[p_k](x)}{\mu_k(x)},
\end{equation}
with remainder $O(\E[\|\vp-\vmu\|^3])$ from the discarded cubic terms.
\end{proof}

\paragraph{Remark on the simplex constraint.}
The constraint $\sum_k p_k=1$ forces $\sum_k(p_k-\mu_k)=0$, so
$\Covmat[\vp]$ has zero row- and column-sums and rank at most $K-1$.
This does not invalidate the trace formula: the approximation operates
on the full $K$-dimensional covariance, and the simplex constraint is
already absorbed into the empirical estimates~\eqref{eq:mu}--\eqref{eq:cov}.
Summing~\eqref{eq:MI-approx} over all $K$ classes therefore remains
valid.
\subsection{Proofs of Lemmas~\ref{lem:boundary} and~\ref{lem:ck-boundary}}
\label{app:proof:boundary}

\begin{proof}[Proof of Lemma~\ref{lem:boundary} (Variance Bound on the Simplex)]
Since $p_k^{(s)}\in[0,1]$ implies $(p_k^{(s)})^2\leq p_k^{(s)}$, we have
\begin{equation}
  \Var[p_k]
  = \E[p_k^2]-\mu_k^2
  \leq \E[p_k]-\mu_k^2
  = \mu_k - \mu_k^2
  = \mu_k(1-\mu_k).
\end{equation}
Since $\mu_k(1-\mu_k)\to 0$ as $\mu_k\to 0$ or $\mu_k\to 1$, we have
$\Var[p_k]\to 0$ at the simplex boundary regardless of the degree of
disagreement among the $S$ forward passes.
\end{proof}

\begin{proof}[Proof of Lemma~\ref{lem:ck-boundary} (Boundary Behaviour of $C_k$)]
Dividing both sides of~\eqref{eq:var-bound} by $2\mu_k>0$:
\begin{equation}
  C_k = \frac{1}{2}\,\frac{\Var[p_k]}{\mu_k}
  \;\leq\; \frac{1}{2}\,\frac{\mu_k(1-\mu_k)}{\mu_k}
  = \frac{1}{2}(1-\mu_k).
\end{equation}
As $\mu_k\to 0$, the upper bound $\tfrac{1}{2}(1-\mu_k)\to\tfrac{1}{2}$,
so $C_k$ is not forced to zero near the simplex boundary.
\end{proof}

\paragraph{Remark on attainability.}
The variance bound is tight: it is attained when $p_k^{(s)}\in\{0,1\}$
for all $s$, i.e., each forward pass is maximally confident about class $k$ but
the passes disagree on the direction.
In this case $\Var[p_k]=\mu_k(1-\mu_k)$ and
$C_k=\tfrac{1}{2}(1-\mu_k)\approx\tfrac{1}{2}$ for $\mu_k\ll 1$: exactly
the non-vanishing signal that motivates the normalisation.
\subsection{Proof of Theorem~\ref{thm:axioms}}
\label{app:proof:axiom}

We follow the axiomatic framework of \citet{wimmer2023quantifying},
as extended by \citet{sale2024labelwise}.
Throughout, $Q \in \Delta^{(2)}_K$ is a second-order distribution over
probability vectors $\vp \in \Delta^{K-1}$, with mean $\vmu = \E_Q[\vp]$.
We write $\mathrm{EU}(Q) = \sum_{k=1}^K C_k =
\tfrac{1}{2}\sum_{k=1}^K \Var[p_k]/\mu_k$.

\paragraph{Axioms evaluated.}
\begin{itemize}
  \item \textbf{A0} (Non-negativity): $\mathrm{EU}(Q) \geq 0$.
  \item \textbf{A1} (Vanishing at certainty): $\mathrm{EU}(Q) = 0$ iff
    $Q = \delta_{\boldsymbol{\theta}}$ for some
    $\boldsymbol{\theta} \in \Delta^{K-1}$.
  \item \textbf{A2} (Maximality at uniform): $\mathrm{EU}$ is maximised
    when $Q = Q_{\mathrm{unif}}$, the uniform distribution on
    $\Delta^{K-1}$.
  \item \textbf{A3} (Monotone under MPS, strict version): if $Q'$ is a
    mean-preserving spread of $Q$ then $\mathrm{EU}(Q') > \mathrm{EU}(Q)$.
  \item \textbf{A5} (Location-shift invariance): if $Q'$ is a
    spread-preserving location shift of $Q$ then
    $\mathrm{EU}(Q') = \mathrm{EU}(Q)$.
\end{itemize}

\paragraph{Definitions.}
Following \citet{sale2024labelwise} (Definition~2.1), $Q'$ is a
\emph{mean-preserving spread} (MPS) of $Q$ if
$\vp' \stackrel{d}{=} \vp + \boldsymbol{Z}$ where $\boldsymbol{Z}$
satisfies $\E[\boldsymbol{Z} \mid \sigma(\vp)] = \boldsymbol{0}$
a.s.\ and $\max_k \Var[Z_k] > 0$.
$Q'$ is a \emph{spread-preserving location shift} of $Q$ if
$\vp' \stackrel{d}{=} \vp + \boldsymbol{z}$ for some constant
$\boldsymbol{z} \neq \boldsymbol{0}$.

\begin{proof}
\noindent\textbf{A0.}
$\Var[p_k] = \E[p_k^2] - \mu_k^2 \geq 0$ by Jensen's inequality.
When $\mu_k > 0$, $C_k = \tfrac{1}{2}\Var[p_k]/\mu_k \geq 0$.
When $\mu_k = 0$: $p_k \geq 0$ and $\E[p_k] = 0$ imply $p_k = 0$
$Q$-a.s., so $\Var[p_k] = 0$ and $C_k := 0$ by convention.
Hence $\mathrm{EU}(Q) = \sum_k C_k \geq 0$.

\medskip\noindent\textbf{A1.}
($\Rightarrow$) If $Q = \delta_{\boldsymbol{\theta}}$, then
$\Var[p_k] = 0$ for every $k$, so $\mathrm{EU}(Q) = 0$.
($\Leftarrow$) Suppose $\mathrm{EU}(Q) = 0$.
Since each $C_k \geq 0$ and $\sum_k C_k = 0$, every $C_k = 0$.
For $\mu_k > 0$ this forces $\Var[p_k] = 0$, hence $p_k = \mu_k$
$Q$-a.s.
For $\mu_k = 0$, $p_k = 0 = \mu_k$ $Q$-a.s.\ by non-negativity of $p_k$.
Hence $\vp = \vmu$ $Q$-a.s., i.e.\ $Q = \delta_{\vmu}$.

\medskip\noindent\textbf{A3 (strict).}
Let $\vp' \stackrel{d}{=} \vp + \boldsymbol{Z}$ with
$\E[\boldsymbol{Z} \mid \sigma(\vp)] = \boldsymbol{0}$ a.s.\ and
$\max_k \Var[Z_k] > 0$.

\emph{Mean preservation.}
$\mu_k' = \E[p_k + Z_k] = \mu_k + \E\!\bigl[\E[Z_k \mid \sigma(\vp)]\bigr] = \mu_k$.

\emph{Variance expansion.}
\begin{align}
  \Var[p_k'] = \Var[p_k + Z_k]
  = \Var[p_k] + \Var[Z_k] + 2\,\Covmat[p_k, Z_k].
\end{align}
The cross-term vanishes by the tower property and $\sigma(\vp)$-measurability
of $p_k$:
\begin{equation}
  \Covmat[p_k, Z_k]
  = \E[p_k Z_k] - \E[p_k]\E[Z_k]
  = \E\!\Bigl[\E[p_k Z_k \mid \sigma(\vp)]\Bigr]
    - \E[p_k]\cdot\E\!\Bigl[\E[Z_k \mid \sigma(\vp)]\Bigr]
  = \E\!\Bigl[p_k\underbrace{\E[Z_k \mid \sigma(\vp)]}_{=\;0}\Bigr]
    - \E[p_k]\cdot 0
  = 0.
\end{equation}
Hence $\Var[p_k'] = \Var[p_k] + \Var[Z_k]$, and since $\mu_k' = \mu_k$,
we have $C_k(Q') \geq C_k(Q)$ for all $k$.

\emph{Strict increase.}
By the MPS condition, $\Var[Z_j] > 0$ for some $j$.
We show $\mu_j > 0$.
If $\mu_j = 0$ then $p_j = 0$ $Q$-a.s., so $p_j' = Z_j$.
Since $\vp' \in \Delta^{K-1}$, we have $p_j' \geq 0$, hence $Z_j \geq 0$
$Q$-a.s.
Combined with $\E[Z_j] = \E\!\bigl[\E[Z_j \mid \sigma(\vp)]\bigr] = 0$,
this forces $Z_j = 0$ $Q$-a.s., contradicting $\Var[Z_j] > 0$.
Hence $\mu_j > 0$, giving $C_j(Q') = \tfrac{1}{2}(\Var[p_j]+\Var[Z_j])/\mu_j > C_j(Q)$,
and therefore $\mathrm{EU}(Q') > \mathrm{EU}(Q)$.

\medskip\noindent\textbf{A2 violated.}
Lemma~\ref{lem:boundary} gives $\Var[p_k] \leq \mu_k(1-\mu_k)$, so
\begin{equation}
  \mathrm{EU}(Q)
  = \frac{1}{2}\sum_{k=1}^K\frac{\Var[p_k]}{\mu_k}
  \leq \frac{1}{2}\sum_{k=1}^K(1-\mu_k)
  = \frac{K-1}{2},
\end{equation}
using $\sum_k \mu_k = 1$.
The bound is attained by the uniform vertex mixture
$Q_{\mathrm{vert}} = \frac{1}{K}\sum_{y=1}^K \delta_{\mathbf{e}_y}$,
for which $\mu_k = 1/K$,
$\Var[p_k] = \E[p_k^2] - \mu_k^2 = 1/K - 1/K^2 = (K-1)/K^2$, and
$\mathrm{EU}(Q_{\mathrm{vert}}) = (K-1)/2$.
By contrast, the uniform distribution on $\Delta^{K-1}$ equals
$\mathrm{Dirichlet}(\mathbf{1}_K)$, whose marginals satisfy
$\Var[p_k] = (K-1)/[K^2(K+1)]$ (standard Dirichlet second moment),
giving
\begin{equation}
  \mathrm{EU}(Q_{\mathrm{unif}})
  = K \cdot \frac{1}{2} \cdot \frac{(K-1)/[K^2(K+1)]}{1/K}
  = \frac{K-1}{2(K+1)}
  < \frac{K-1}{2}
  = \mathrm{EU}(Q_{\mathrm{vert}})
  \quad \forall\, K \geq 2.
\end{equation}
Hence $Q_{\mathrm{unif}}$ is not the maximiser and A2 is violated.

\medskip\noindent\textbf{A5 violated.}
Under $\vp' = \vp + \boldsymbol{z}$ (constant $\boldsymbol{z}$):
$\Var[p_k'] = \Var[p_k]$ by translation-invariance of variance, but
$\mu_k' = \mu_k + z_k$, so
$C_k(Q') = \tfrac{1}{2}\Var[p_k]/(\mu_k + z_k) \neq C_k(Q)$
whenever $z_k \neq 0$ and $\Var[p_k] > 0$.
Concretely, let $K = 2$ and
$Q = \tfrac{1}{2}\delta_{(0.2,\,0.8)} + \tfrac{1}{2}\delta_{(0.4,\,0.6)}$,
so $\vmu = (0.3,\,0.7)$, $\Var[p_k] = 0.01$ for both $k$, and
$\mathrm{EU}(Q) \approx 0.0238$.
Under $\boldsymbol{z} = (0.15,\,-0.15)$: $\vmu' = (0.45,\,0.55)$,
variances unchanged, giving
$\mathrm{EU}(Q') \approx 0.0202 \neq \mathrm{EU}(Q)$.
\end{proof}

\begin{remark}[\citet{wimmer2023quantifying} Proposition~4]
Exact MI violates the strict version of A3 already for $K = 2$.
The improvement of $\sum_k C_k$ is a direct consequence of its
linearity in $\{\Var[p_k]\}$: any non-trivial spread in any class
propagates immediately into a strict increase of the aggregate.
\end{remark}
\begin{table}[!ht]
\centering
\caption{Axiomatic profile of epistemic uncertainty measures. \cmark\ = satisfied, \xmark\ = violated. Exact MI results from \citet{wimmer2023quantifying}; Sale et al.\ results from \citet{sale2024labelwise}.}
\label{tab:axioms}
\resizebox{0.5\columnwidth}{!}{
\begin{tabular}{lccc}
\toprule
Axiom & $\sum_k \Ck$ (ours) & Exact MI & $\Var(\Theta_k)$ \\
\midrule
A0: Non-negativity          & \cmark & \cmark & \cmark \\
A1: Vanishing EU            & \cmark & \cmark & \cmark \\
A2: Maximality at uniform   & \xmark & \xmark & \xmark \\
A3: Monotonicity (MPS)      & \cmark & \xmark & \cmark \\
A5: Location-shift invariance & \xmark & \xmark & \cmark \\
\bottomrule
\end{tabular}
}
\end{table}

\section{Complete Metric Reference}
\label{app:metrics}

This appendix collects, in self-contained form, the definitions and
interpretations of every uncertainty metric evaluated in the selective
prediction experiments.
Throughout, we adopt the notation of Section~\ref{sec:method}:
$S$ stochastic forward passes produce probability vectors
$\vp^{(s)}(x)\in\Delta^{K-1}$, $s=1,\dots,S$, from which all
subsequent quantities are estimated.

\subsection{Monte Carlo Prediction Statistics}
\label{app:mc-stats}

The following sample statistics, computed from the $S$ stochastic
forward passes, serve as building blocks for all metrics defined
below.

\paragraph{Mean prediction.}
The predictive mean
\begin{equation}\label{eq:app-mu}
  \mu_k \;=\; \frac{1}{S}\sum_{s=1}^{S} p_k^{(s)},
  \qquad k = 0,\dots,K{-}1,
\end{equation}
represents the model's consensus probability for class~$k$.
The predicted label is $\hat{y}=\arg\max_k \mu_k$.

\paragraph{Per-class variance and covariance.}
The sample variance and covariance, with Bessel correction, are
\begin{align}
  \sigma_k^2
    &\;=\; \frac{1}{S-1}\sum_{s=1}^{S}\bigl(p_k^{(s)}-\mu_k\bigr)^2,
    \label{eq:app-var}\\[4pt]
  \Sigma_{ij}
    &\;=\; \frac{1}{S-1}\sum_{s=1}^{S}
      \bigl(p_i^{(s)}-\mu_i\bigr)\bigl(p_j^{(s)}-\mu_j\bigr),
    \label{eq:app-cov}
\end{align}
with Pearson correlation
$\rho_{ij}=\Sigma_{ij}/(\sigma_i\,\sigma_j)$.
A strongly negative $\rho_{ij}$ indicates that probability mass
flows systematically from class~$j$ to class~$i$ across forward
passes, i.e.\ the model actively confuses the two classes.

\paragraph{Third central moment.}
\begin{equation}\label{eq:app-m3}
  m_{3,k} \;=\; \frac{1}{S}\sum_{s=1}^{S}
    \bigl(p_k^{(s)}-\mu_k\bigr)^3.
\end{equation}
This quantity enters only through the skewness diagnostic~$\rho_k$
defined in Section~\ref{app:diagnostics}.

\subsection{Scalar Uncertainty Baselines}
\label{app:scalar}

The following standard metrics assign a single uncertainty score to
each input and do not distinguish between classes.

\paragraph{Predictive entropy.}
The Shannon entropy of the mean prediction,
\begin{equation}\label{eq:app-ent}
  H[\bar{p}] \;=\; -\sum_{k=0}^{K-1}\mu_k\log\mu_k,
\end{equation}
lies in $[0,\log K]$ and captures \emph{total} predictive
uncertainty, both epistemic and aleatoric.
It cannot distinguish a model that lacks training data from one
facing a genuinely ambiguous input; nor does it differentiate
clinically benign confusions (e.g.\ Grade~0 vs.\ Grade~1) from
dangerous ones (Grade~0 vs.\ Grade~3).

\paragraph{Mutual information.}
\begin{equation}\label{eq:app-mi}
  \mathrm{MI}
    \;=\; H[\bar{p}]
    \;-\; \frac{1}{S}\sum_{s=1}^{S}H\bigl[\vp^{(s)}\bigr],
\end{equation}
where $H[\vp^{(s)}]=-\sum_k p_k^{(s)}\log p_k^{(s)}$.
MI isolates epistemic uncertainty by subtracting the average
per-pass entropy (the aleatoric component): if all $S$ passes
produce identical softmax vectors, MI vanishes regardless of how
uncertain those vectors are.
While MI is the standard scalar epistemic measure, it remains a
single number that cannot reveal \emph{which} classes drive the
model's disagreement.
Two inputs with identical MI may involve entirely different
confusion patterns, as demonstrated by the error signatures in
Section~\ref{app:error-signatures}.

\paragraph{Maximum softmax uncertainty.}
\begin{equation}\label{eq:app-maxprob}
  \mathrm{MaxProb}
    \;=\; 1 - \max_k\;\mu_k\,.
\end{equation}
This is the simplest deferral baseline: it uses only the mean
prediction and does not exploit MC disagreement.
MaxProb is insensitive to the \emph{distribution} of probability
across non-top classes; two inputs
$\mu=(0.6,\,0.4,\,0,\,0)$ and $\mu=(0.6,\,0,\,0.4,\,0)$ receive
identical scores despite the latter involving a dangerous
safe--critical confusion.

\subsection{Per-Class Variance Baselines}
\label{app:per-class-var}

These metrics use the raw per-class variance~$\sigma_k^2$ as a
proxy for epistemic disagreement, without the $1/\mu_k$
normalisation that defines~$C_k$.

\paragraph{Critical variance.}
\begin{equation}\label{eq:app-varcrit}
  \mathrm{Var\_crit}
    \;=\; \max_{k\in\mathcal{C}}\;\sigma_k^2\,,
\end{equation}
where $\mathcal{C}=\{2,3\}$ denotes the set of critical
(treatment-requiring) classes.
This provides a targeted measure of MC disagreement restricted to
critical classes, but inherits the base-rate confound identified in
Lemma~\ref{lem:boundary}: a class with $\mu_k=0.5$ naturally
exhibits higher variance than one with $\mu_k=0.01$, even at equal
levels of epistemic disagreement in logit space.
\paragraph{Global and critical variance baselines.}
\citet{sale2024labelwise} explicitly define the global variance-based
epistemic uncertainty measure (their equation~22):
\begin{equation}
  \mathrm{Sale\_EU\_global}
    \;=\; \sum_{k=0}^{K-1}\sigma_k^2\,,
    \label{eq:app-sale-global}
\end{equation}
which aggregates label-wise variance over all classes and constitutes
a scalar epistemic measure directly comparable to MI and $\sum_k C_k$.
We additionally construct a critical-class restriction not defined by
\citet{sale2024labelwise}:
\begin{equation}
  \mathrm{Sale\_EU\_crit}
    \;=\; \sum_{k\in\mathcal{C}}\sigma_k^2\,,
    \label{eq:app-sale-crit}
\end{equation}
which provides a targeted analogue of $C_{\mathrm{crit\_sum}}$ without
curvature normalisation.
Both variants suffer from the same boundary suppression as
$\mathrm{Var\_crit}$: for critical classes with mean prediction
$\mu_k\approx 0.06$ in our data, variance is bounded by
$\mu_k(1-\mu_k)\approx 0.056$ regardless of the degree of model
disagreement (cf.\ Section~\ref{sec:motivation}).

\subsection{One-vs-All MI Baseline}
\label{app:per-class-mi}

An alternative per-class decomposition computes binary mutual
information for each class independently.
For class~$k$, define the binary variable
$q_k^{(s)}=(p_k^{(s)},\;1-p_k^{(s)})$ and the corresponding
binary MI
\begin{equation}\label{eq:app-ova}
  \mathrm{MI}_k^{\mathrm{bin}}
    \;=\; H[\bar{q}_k]
    \;-\; \frac{1}{S}\sum_{s=1}^{S}H\bigl[q_k^{(s)}\bigr].
\end{equation}
The deferral score sums over critical classes:
$\mathrm{OvA\_MI}=\sum_{k\in\mathcal{C}}\mathrm{MI}_k^{\mathrm{bin}}$.
This measure is exact (not Taylor-approximated) but discards the
multi-class structure: it does not reveal which non-$k$ classes
compete with class~$k$, and the per-class binary MI terms are not
additive components of the full multiclass MI.

\subsection{Proposed Metrics: Per-Class Epistemic Decomposition}
\label{app:proposed}

The metrics in this subsection are derived from the second-order
Taylor expansion of mutual information developed in
Section~\ref{sec:method}.

\paragraph{Per-class epistemic contribution.}
The central quantity is the per-class component
\begin{equation}\label{eq:app-Ck}
  C_k \;=\; \frac{1}{2}\,\frac{\sigma_k^2}{\mu_k}
\end{equation}
which satisfies $\sum_{k=0}^{K-1}C_k\approx\mathrm{MI}$ by
Theorem~\ref{thm:taylor}.
The $1/\mu_k$ weighting originates from the curvature of the
entropy surface ($\partial^2 H/\partial p_k^2=-1/\mu_k$):
a class with $\mu_k=0.01$ has curvature~$100$, so even modest
probability variance there carries substantial
information-theoretic weight, whereas a class with $\mu_k=0.5$
requires proportionally larger variance to indicate genuine
epistemic uncertainty.
The ratio $\sigma_k^2/\mu_k$ coincides with the classical
index of dispersion~\citep{CoxLewis1966,Fano1947}; the
$\tfrac{1}{2}$ prefactor arises from the Hessian structure,
ensuring that the components sum to MI rather than twice MI.

The approximation is reliable when the MC distribution for each
class is approximately symmetric.
Reliability is diagnosed by the skewness indicator~$\rho_k$
(Section~\ref{app:diagnostics}): when $\rho_k>0.3$ for a critical
class, $C_k$ should be interpreted with caution and CBEC
(Section~\ref{app:boundary}) is preferred.

\paragraph{Critical-class aggregations.}
Two natural aggregations restrict attention to the critical class
set $\mathcal{C}$:
\begin{align}
  C_{\mathrm{crit\_max}}
    &\;=\; \max_{k\in\mathcal{C}}\;C_k\,,
    \label{eq:app-cmax}\\[4pt]
  C_{\mathrm{crit\_sum}}
    &\;=\; \sum_{k\in\mathcal{C}}C_k\,.
    \label{eq:app-csum}
\end{align}
$C_{\mathrm{crit\_max}}$ is sensitive to single-class spikes: it
fires when \emph{any} critical class carries elevated epistemic
uncertainty, and is the recommended aggregation when $\rho_k<0.3$
for all critical classes.
$C_{\mathrm{crit\_sum}}$ is a smoother alternative that captures
cumulative uncertainty when confusion is distributed across
multiple critical classes.

\subsection{Boundary-Aware Metrics}
\label{app:boundary}

The metrics above score each class (or class subset) in isolation.
In applications with a clinically meaningful partition into safe
classes $\mathcal{S}=\{0,1\}$ and critical classes
$\mathcal{C}=\{2,3\}$, the most dangerous errors are those that
cross this boundary.
The following metric targets such cross-boundary confusion directly.

\paragraph{Cross-boundary epistemic confusion (CBEC).}
\begin{equation}\label{eq:app-cbec}
  \mathrm{CBEC}(x)
    \;=\; \sum_{i\in\mathcal{S}}\;\sum_{j\in\mathcal{C}}
      \sqrt{C_i\cdot C_j}\;\cdot\;\max(0,\,-\rho_{ij}).
\end{equation}
CBEC is non-zero only when three conditions hold simultaneously:
a safe class carries elevated~$C_i$, a critical class carries
elevated~$C_j$, and their MC predictions are negatively correlated
($\rho_{ij}<0$), indicating that probability mass flows
systematically from one to the other across stochastic passes.
The geometric mean $\sqrt{C_i C_j}$ requires \emph{both} classes to
be uncertain, and the correlation gate
$\max(0,-\rho_{ij})$ filters coincidental co-elevation.

CBEC is partially robust to degradation of the Taylor approximation
for two reasons.
First, the correlation~$\rho_{ij}$ is computed directly from MC
samples and does not depend on the second-order expansion.
Second, the geometric mean compresses outlier inflation: if
skewness inflates~$C_j$ by a factor of~$\alpha$, the geometric
mean is inflated by only~$\sqrt{\alpha}$.
These properties explain CBEC's empirical robustness under MC
dropout, where $C_k$ estimates for rare critical classes degrade
(Section~\ref{app:mcdropout}).

\subsection{Diagnostic Metrics}
\label{app:diagnostics}

The following quantities assess the reliability and structure of the
$C_k$ estimates.
They are not used as deferral scores.

\paragraph{Skewness diagnostic.}
For each class $k$ with $\sigma_k^2>0$, the ratio
\begin{equation}\label{eq:app-rho}
  \rho_k
    \;=\; \frac{|m_{3,k}|}{3\,\mu_k\cdot\sigma_k^2}
\end{equation}
equals the magnitude of the third-order Taylor correction relative
to the second-order term~$C_k$ (see Section~\ref{sec:skewness}).
When $\rho_k\ll 1$, the cubic remainder is negligible and $C_k$ is
reliable; when $\rho_k>0.3$, the second-order approximation is
degrading and~$C_k$ should be supplemented or replaced by CBEC.
In practice, majority classes (high~$\mu_k$, far from the simplex
boundary) produce approximately symmetric MC distributions and low
$\rho_k$; rare classes ($\mu_k\approx 0$) produce right-skewed
distributions bounded by zero and correspondingly higher~$\rho_k$.

\paragraph{Epistemic profiles.}
For each true class $i$, the normalised conditional expectation
\begin{equation}\label{eq:app-profile}
  \mathrm{Profile}_k(i)
    \;=\; \mathbb{E}\!\Bigl[\frac{C_k}{\textstyle\sum_j C_j}
      \;\Big|\; y=i\Bigr],
    \qquad k=0,\dots,K{-}1,
\end{equation}
reveals the average distribution of epistemic uncertainty across
classes for inputs of true class~$i$.
For example, a profile $(0.14,\,0.27,\,0.55,\,0.05)$ for Grade~0
images indicates that the model's uncertainty about healthy images
concentrates predominantly on Grade~2 (moderate DR), signalling
systematic Grade~0--Grade~2 confusion.

\paragraph{Error epistemic signatures.}
Conditioning on both true and predicted labels yields the
error-type signature
\begin{equation}\label{eq:app-sig}
  \mathrm{Sig}_k(i\!\to\!j)
    \;=\; \mathbb{E}\bigl[C_k
      \;\big|\; y=i,\;\hat{y}=j\bigr],
    \qquad k=0,\dots,K{-}1.
\end{equation}
Different error types produce distinct~$C_k$ fingerprints even when
their scalar MI is similar, enabling targeted diagnostic
interpretation (see Section~\ref{app:error-signatures}).

\paragraph{Epistemic confusion matrix.}
\begin{equation}\label{eq:app-eij}
  E_{ij}
    \;=\; \mathbb{E}\!\bigl[\sqrt{C_i\cdot C_j}
      \;\cdot\;\max(0,-\rho_{ij})\bigr].
\end{equation}
Entry $E_{ij}$ quantifies the average directed epistemic confusion
between classes~$i$ and~$j$, aggregated over the test set.
The matrix is symmetric with zero diagonal (since
$\rho_{kk}=1$ implies $\max(0,-\rho_{kk})=0$).
High values in the $\mathcal{S}\times\mathcal{C}$ block indicate
that the model's epistemic uncertainty concentrates at the
clinically dangerous boundary, validating both the safe--critical
partition and the CBEC aggregation.

\subsection{Evaluation Metrics for Selective Prediction}
\label{app:eval-metrics}

The following metrics evaluate the quality of a deferral policy that,
at coverage level~$c$, retains a fraction~$c$ of test inputs for
autonomous prediction and defers the remainder to human review.
Let $\mathrm{kept}(c)$ denote the set of retained samples.

\paragraph{Critical false-negative rate.}
\begin{equation}\label{eq:app-cfnr}
  \mathrm{Critical\;FNR}(c)
    \;=\;
    \frac{\sum_{n\in\mathrm{kept}(c)}
      \mathbf{1}[y_n\in\mathcal{C}]\cdot
      \mathbf{1}[\hat{y}_n\notin\mathcal{C}]}
    {\sum_{n\in\mathrm{kept}(c)}
      \mathbf{1}[y_n\in\mathcal{C}]}\,.
\end{equation}
This is the primary clinical safety metric: among the critical-class
samples that the system decides autonomously (does not defer), it
measures the fraction misclassified as safe.
Crucially, within-critical errors (e.g.\ Grade~3 predicted as
Grade~2) are \emph{not} counted, since the patient still receives
treatment; only boundary crossings into~$\mathcal{S}$ are penalised.

\paragraph{Critical error rate.}
\begin{equation}\label{eq:app-cerr}
  \mathrm{Critical\;Err}(c)
    \;=\;
    \frac{\sum_{n\in\mathrm{kept}(c)}
      \mathbf{1}[y_n\in\mathcal{C}]\cdot
      \mathbf{1}[\hat{y}_n\neq y_n]}
    {\sum_{n\in\mathrm{kept}(c)}
      \mathbf{1}[y_n\in\mathcal{C}]}\,.
\end{equation}
A broader metric that includes within-critical misclassifications.
We report it alongside Critical~FNR for completeness but consider
Critical~FNR the primary safety metric.

\paragraph{Area under the selective risk curve (AUSC).}
\begin{equation}\label{eq:app-ausc}
  \mathrm{AUSC}
    \;=\; \int_0^1 R(c)\,dc\,,
\end{equation}
where $R(c)$ is the risk (Critical FNR or error rate) at
coverage~$c$~\citep{elyaniw2010}.
AUSC provides a single-number summary that integrates risk over
all possible operating points, so a policy with low AUSC is
consistently safe regardless of the chosen coverage threshold.
We approximate the integral via the trapezoidal rule at
$n=200$ equally spaced coverage levels; 95\% confidence intervals
are obtained from 200 bootstrap resamples of the test set.

Note that AUSC and Critical~FNR at a fixed operating point can
disagree: a policy that dominates at low coverages (50--70\%) but
is inferior at 80\% will produce favourable AUSC yet a suboptimal
fixed-point FNR.
This phenomenon appears in the MC dropout experiment
(Section~\ref{app:mcdropout}).

\paragraph{Accuracy and macro~F1.}
Standard classification metrics computed on the retained set:
\begin{equation}\label{eq:app-acc}
  \mathrm{Accuracy}(c)
    = \frac{\sum_{n\in\mathrm{kept}(c)}
        \mathbf{1}[\hat{y}_n=y_n]}{|\mathrm{kept}(c)|},
  \qquad
  \mathrm{Macro\;F1}(c)
    = \frac{1}{K}\sum_{k=0}^{K-1}F_{1,k}\,.
\end{equation}
In our setting, accuracy is dominated by the majority class
(Grade~0 constitutes approximately 70\% of the test set) and can be
misleading: a policy that defers all minority-class samples achieves
high accuracy but poor safety.
Macro~F1 partially corrects the class-prevalence bias by weighting
classes equally, but remains an aggregate metric that does not
specifically target clinical safety.
We report both alongside the primary Critical~FNR.
\section{Diabetic Retinopathy Experiment: Supplementary Details}
\label{app:dr}

\subsection{Dataset and Preprocessing}\label{app:dataset}

We pool three public DR grading datasets: EyePACS ($35{,}108$ images), APTOS~2019 ($3{,}662$), and Messidor-2 ($1{,}200$), totalling $39{,}970$ fundus images.
The original five-level International Clinical DR Scale is consolidated into four grades: Grade~0 (no~DR, 70.4\%), Grade~1 (mild NPDR, 7.4\%), Grade~2 (moderate NPDR, 16.3\%), and Grade~3 (severe NPDR + PDR, 5.8\%).
We merge the two highest grades because both require urgent referral, and the small severe NPDR subgroup is insufficient for reliable posterior estimation.

Images are resized to $256 \times 256$ pixels.
Data are split patient-stratified into train/val/test = 28{,}018/4{,}004/7{,}948, with a patient-leakage check confirming no patient appears in multiple splits.
The test set preserves the class distribution: Grade~0 = 5{,}584 (70.3\%), Grade~1 = 587 (7.4\%), Grade~2 = 1{,}322 (16.6\%), Grade~3 = 455 (5.7\%).

\subsection{Model Architecture and Training}\label{app:model}

\paragraph{Bayesian EfficientNet-B4.}
We construct a fully Bayesian variant of EfficientNet-B4 (width multiplier $1.4$, depth multiplier $1.8$) using the low-rank variational framework of \citet{toure2026singular}, which parameterises each weight matrix as $W = AB^\top$ with learned posteriors on the factors $A \in \mathbb{R}^{m \times r}$, $B \in \mathbb{R}^{n \times r}$, reducing the parameter count from $O(mn)$ to $O(r(m+n))$.
The architecture comprises a stem convolution ($3 \to 48$ filters), 32~MBConv blocks across 7~stages, a top expansion ($448 \to 1792$), and a classification head (Conv $1792 \to 256$, GAP, Dense $256 \to 128$, Dropout, Dense $128 \to 4$).

Pointwise ($1{\times}1$) and expansion convolutions use low-rank factorisation $W = AB^\top$ with depth-aware rank selection: compression increases from $1.5\times$ at early layers to $25\times$ at deep layers, with rank $r = \lfloor d_{\text{in}} d_{\text{out}} / (\text{compression} \cdot (d_{\text{in}} + d_{\text{out}})) \rfloor$.
Depthwise convolutions and squeeze-excitation $1{\times}1$ convolutions use direct (full-rank) variational parameterisation.
Dense layers in the classification head use the same low-rank factorisation.

All variational layers use the local reparameterisation trick with $W = \mu + \mathrm{softplus}(\rho) \odot \epsilon$, $\epsilon \sim \mathcal{N}(0,I)$.
The prior is a scale-mixture Gaussian $\pi \mathcal{N}(0, \sigma_1^2) + (1-\pi)\mathcal{N}(0, \sigma_2^2)$ with $\pi = 0.5$, $\sigma_1 = 1.0$, $\sigma_2 = e^{-6}$.
Posterior initialisation uses He-scaled uniform means with $\rho$ initialised so that $\mathrm{softplus}(\rho) \approx 0.09 \sqrt{2/d_{\text{in}}}$.
Training uses KL annealing with the scale frozen at zero for the first epochs, then linearly warmed up.
BatchNormalization layers remain deterministic.
\paragraph{Full-rank layers.}\label{par:fullrank}
Depthwise convolutions lack the $m \times n$ matrix structure required for low-rank factorisation: each channel has an independent $k \times k$ kernel ($k \in \{3,5\}$), so there is no shared row--column space on which to impose $W = AB^\top$.
SE layers are already bottlenecked by the squeeze ratio ($r_{\text{SE}} = 0.25$), yielding weight matrices such as $112 \times 28$ that are smaller than the rank our depth-aware rule would select; further factorisation would over-compress the channel-attention mechanism.
Both cases therefore use direct variational parameterisation with element-wise $(\mu, \rho)$ pairs.
\paragraph{MC Dropout baseline.}
We also train a standard (deterministic) EfficientNet-B4 with dropout ($p = 0.3$) before the final dense layer.
At test time, dropout is kept active for $S = 30$ forward passes, producing MC~dropout posterior samples.

\paragraph{Inference.}
For both models, uncertainty is computed from $S = 30$ stochastic forward passes, producing $\vp^{(s)} \in \Delta^3$ for each test image.
All metrics in Appendix~\ref{app:metrics} are computed from these $S$ samples.

\subsection{Baseline Classification Performance}\label{app:classification}

Table~\ref{tab:classification} reports the Bayesian model's per-class precision, recall, and F1 on the test set.
The model achieves overall accuracy $0.794$ and quadratic weighted kappa $0.646$.
Grade~1 (mild NPDR) is effectively undetectable (F1 $= 0.00$), consistent with its visual subtlety and small sample size; this does not affect our analysis since Grade~1 belongs to the safe class partition.

\begin{table}[!ht]
\centering
\caption{Classification report for Bayesian EfficientNet-B4 on DR test set ($n = 7{,}948$).}
\label{tab:classification}
\begin{tabular}{lcccc}
\toprule
Class & Precision & Recall & F1 & Support \\
\midrule
Grade 0 (No DR) & 0.83 & 0.98 & 0.89 & 5{,}584 \\
Grade 1 (Mild)  & 0.00 & 0.00 & 0.00 & 587 \\
Grade 2 (Moderate) & 0.62 & 0.56 & 0.59 & 1{,}322 \\
Grade 3 (Severe+PDR) & 0.75 & 0.28 & 0.41 & 455 \\
\midrule
Macro avg & 0.55 & 0.45 & 0.47 & 7{,}948 \\
Weighted avg & 0.73 & 0.79 & 0.75 & 7{,}948 \\
\bottomrule
\end{tabular}
\end{table}

The confusion matrix reveals the dominant error modes:
Grade~0 absorbs most misclassifications from all other grades (497 from Grade~1, 548 from Grade~2, 107 from Grade~3), reflecting the class imbalance and the model's conservatism toward the majority class.
Within the critical partition, 220~Grade~3 samples are predicted as Grade~2 (severity underestimate) and 39~Grade~2 samples as Grade~3.

\subsection{Skewness Diagnostic Analysis}\label{app:skewness}

The skewness diagnostic $\rho_k$ (Definition~\ref{def:rho}) assesses whether $C_k$ is reliable for each class and sample.
Table~\ref{tab:skewness} reports per-class $\rho_k$ statistics on the full test set.

\begin{table}[!ht]
\centering
\caption{Skewness diagnostic $\rho_k$ by grade (Bayesian model, $n = 7{,}948$).
``Reliable'' = fraction of samples with $\rho_k < 0.3$.}
\label{tab:skewness}
\begin{tabular}{lcccc}
\toprule
Grade & Median $\rho_k$ & Mean $\rho_k$ & $p_{90}$ & Reliable \\
\midrule
0 (No DR) & 0.012 & 0.086 & 0.116 & 94.6\% \\
1 (Mild)  & 0.037 & 0.082 & 0.135 & 97.0\% \\
2 (Moderate) & 0.124 & 0.198 & 0.423 & 81.6\% \\
3 (Severe+PDR) & 0.218 & 0.328 & 0.706 & 63.4\% \\
\bottomrule
\end{tabular}
\end{table}
The threshold should not be interpreted as universal. We therefore also
swept the reliability cutoff over $\{0.1,0.2,0.3,0.5\}$ under a
true-class-conditioned analysis. Tightening the cutoff to $0.1$ lowers the
reliable fraction for critical Grades~2--3 to roughly $68$--$75\%$, while
relaxing it to $0.5$ raises the reliable fraction to roughly $93$--$97\%$.
The qualitative conclusion is unchanged: safe classes remain substantially
more reliable than critical classes, and the critical grades are the most
sensitive to the chosen tolerance.
Two patterns emerge.
First, the safe classes (Grades~0--1) have consistently low $\rho_k$ (median $< 0.04$, $>94\%$ reliable), because their high mean probabilities ($\mu_0 \approx 0.7$) place them far from the simplex boundary where skewness is amplified.
Second, the critical classes show progressively higher skewness: Grade~2 (median $\rho_2 = 0.12$, 82\% reliable) and Grade~3 (median $\rho_3 = 0.22$, 63\% reliable), consistent with their low base rates pushing $\mu_k$ toward the boundary.
Across the full test set, 57\% of samples have \emph{all} four $\rho_k < 0.3$.
See Figure~\ref{fig:skewness-app} for the full distribution.

\begin{figure}[!ht]
  \centering
  \includegraphics[width=0.8\columnwidth]{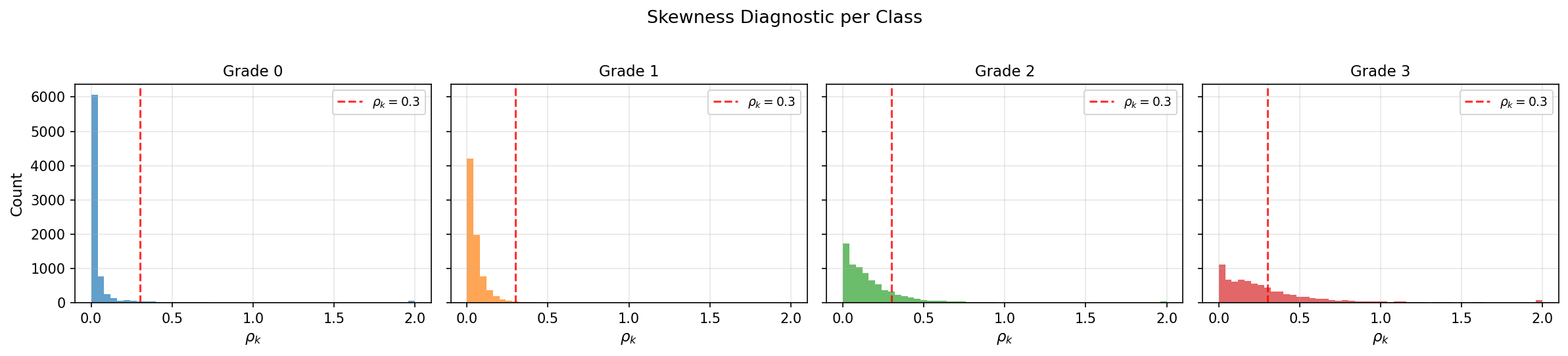}
  \caption{Distribution of skewness diagnostic $\rho_k$ by class.
  Safe classes (Grades~0--1) cluster near zero; critical classes (Grades~2--3) exhibit heavier tails, reflecting boundary suppression effects on rare-class posterior samples.}
  \label{fig:skewness-app}
\end{figure}

These statistics explain the BNN ranking in Table~\ref{tab:selective-bnn}: with the majority of critical-class $C_k$ estimates reliable, $C_{\text{crit\_max}}$ can exploit the per-class signal effectively.
The MC~dropout model (Appendix~\ref{app:mcdropout}) produces coarser posterior samples (fewer effective degrees of freedom), inflating $\rho_k$ and degrading $C_k$ for critical classes, which explains why CBEC, whose correlation gate is Taylor-free and whose restricted domain $\mathcal{S}\times\mathcal{C}$ filters irrelevant co-elevation, outperforms $C_{\text{crit\_max}}$ in that regime.

\subsection{Class-Imbalance Stress Test}
\label{app:imbalance-sweep}

To isolate the effect of class prevalence on the decomposition, we ran a
controlled four-class FashionMNIST subset experiment in which the training
prevalence of one target class was reduced from $100\%$ to $10\%$, while
keeping labels clean, the architecture fixed, and the test set balanced.
As the target class becomes rarer, its mean prediction decreases
($\mu_k: 0.933 \to 0.734$), raw variance increases
($1.42{\times}10^{-3} \to 3.82{\times}10^{-3}$), and the normalised
contribution $C_k$ increases more strongly
($1.47{\times}10^{-3} \to 6.12{\times}10^{-3}$).
Approximation quality deteriorates only gradually: the target-class
$\rho_k$ increases from $0.025$ to $0.129$, the absolute gap
$|\sum_k C_k-\MI|$ increases from $5.4{\times}10^{-4}$ to
$9.8{\times}10^{-4}$, and the reliable fraction remains high, decreasing
from $99.2\%$ to $93.9\%$. Thus class imbalance makes the
boundary-aware normalisation more important, but prevalence alone does not
determine reliability; $\rho_k$ remains the appropriate diagnostic.
\subsection{\texorpdfstring{Theory Validation: MI vs.\ $\sum_k C_k$}{Theory Validation: MI vs. sum Ck}}
\label{app:validation}

Figure~\ref{fig:theory-validation} plots $\sum_k C_k$ against exact MI for all 7{,}948 test samples, complementing the summary statistics reported in Section~\ref{sec:results-dr}.
The slight systematic overestimation visible at high MI values is consistent with the third-order correction in Lemma~\ref{lem:third-order}: right-skewed distributions near the simplex boundary produce $m_{3,k} > 0$, and the omitted correction $-\frac{1}{6}\sum_k m_{3,k}/\mu_k^2$ causes $\sum_k C_k$ to slightly exceed MI.
The right panel confirms this: residuals $(\sum_k C_k - \mathrm{MI})$ are predominantly positive for high-skewness samples (pink/red points), validating that the skewness diagnostic $\rho_k$ correctly identifies inputs where the approximation degrades.

\begin{figure}[!ht]
  \centering
  \includegraphics[width=0.65\columnwidth]{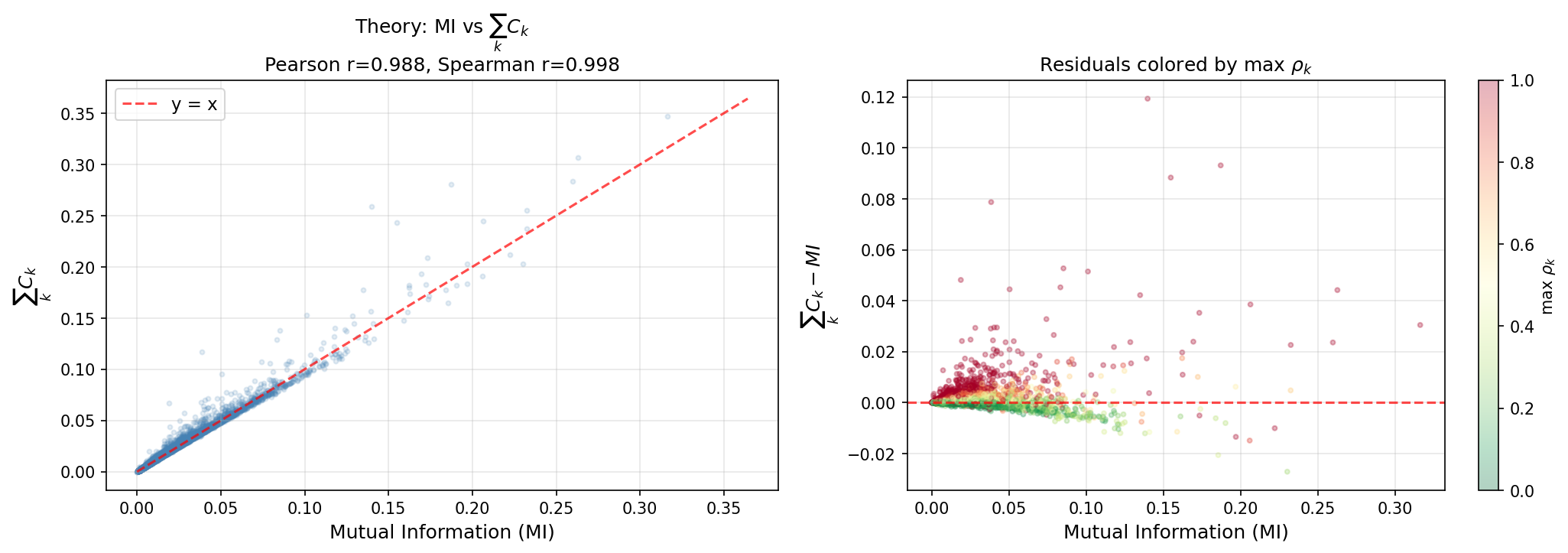}
  \caption{$\sum_k C_k$ vs.\ exact MI for all $7{,}948$ test samples (Pearson $r = 0.988$, Spearman $r = 0.998$).
  \textbf{Left:} Scatter plot; the near-perfect rank correlation confirms that the second-order approximation preserves the ordering of epistemic uncertainty.
  \textbf{Right:} Residuals coloured by maximum per-class $\rho_k$; positive residuals concentrate among high-skewness samples, as predicted by Lemma~\ref{lem:third-order}.}
  \label{fig:theory-validation}
\end{figure}

\subsection{Extended Selective Prediction Analysis}\label{app:extended-selective}

Figure~\ref{fig:ausc-curves} shows the full selective risk curves for all 10 deferral policies across continuous coverage levels.
The left panel (Critical FNR) confirms that $C_{\text{crit\_max}}$ achieves the lowest AUSC (0.285), with $C_{\text{crit\_sum}}$ second (0.329) and CBEC third (0.415).
The gap widens at lower coverages where clinical operating points typically lie, and the unnormalised variance baselines (Sale\_EU\_crit: 0.650, Var\_crit: 0.606) perform substantially worse than even scalar MI (0.436), confirming that boundary suppression severely degrades raw variance for rare critical classes.

The right panel addresses a natural concern: does targeting critical-class uncertainty sacrifice overall accuracy?
The answer is no. $C_{\text{crit\_max}}$ achieves competitive error-rate AUSC (0.143), comparable to entropy (0.127) and MaxProb (0.126), which optimise for overall accuracy rather than safety.
This demonstrates that the $C_k$ decomposition identifies samples where critical-class confusion drives errors, rather than deferring difficult samples indiscriminately.

A further distinction emerges at the clinically relevant 80\% coverage point: $C_{\text{crit\_max}}$ retains 1{,}293 critical samples (27.2\% deferred) while achieving FNR $0.302$, whereas MI retains only 1{,}183 (33.4\% deferred) with FNR $0.339$.
$C_{\text{crit\_max}}$ is simultaneously \emph{less aggressive} at deferring critical cases and \emph{more accurate} on those it keeps, the hallmark of a well-targeted deferral policy.
Entropy, by contrast, defers 42.1\% of critical samples at 80\% coverage because it prioritises total uncertainty rather than critical-class epistemic uncertainty.

\begin{figure}[!ht]
  \centering
  \includegraphics[width=0.75\columnwidth]{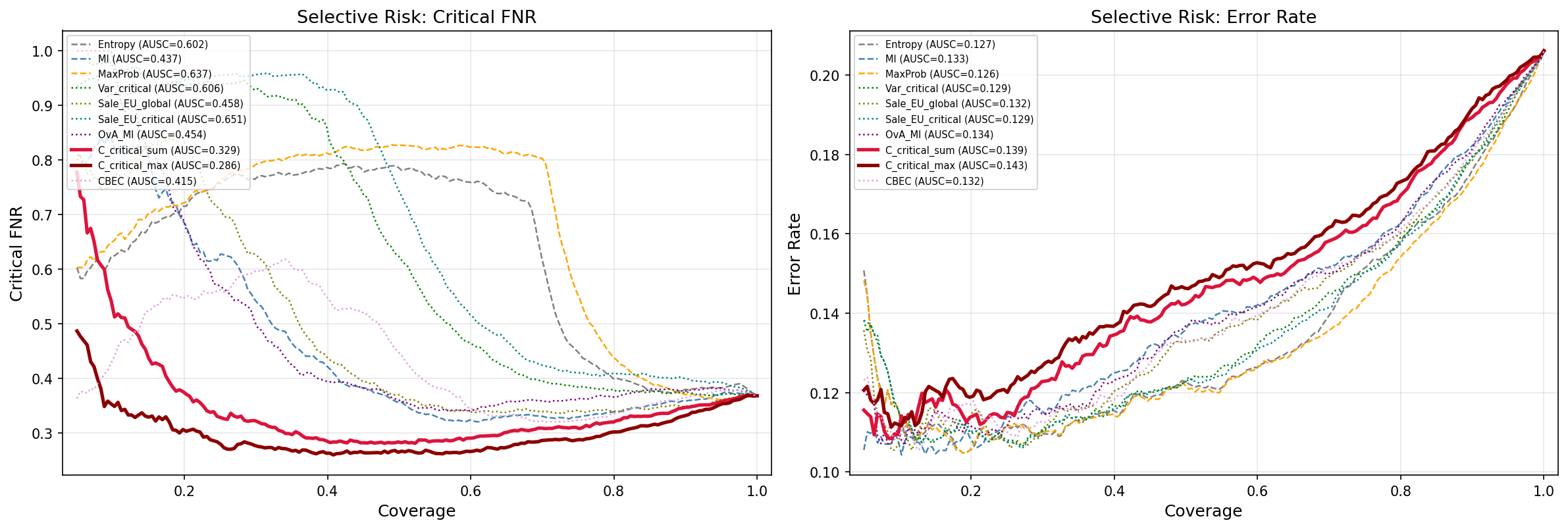}
  \caption{Selective risk curves for all 10 deferral policies.
  \textbf{Left:} Critical FNR across coverage levels; $C_{\text{crit\_max}}$ achieves the lowest AUSC ($0.285$), dominating all baselines across the full coverage range.
  \textbf{Right:} Error rate; $C_{\text{crit\_max}}$ remains competitive (AUSC $= 0.143$), confirming that per-class targeting does not sacrifice overall accuracy.}
  \label{fig:ausc-curves}
\end{figure}

\subsection{Full Bootstrap Results: Bayesian Model}\label{app:bootstrap-bnn}

Table~\ref{tab:full-bnn} extends Table~\ref{tab:selective-bnn} with all evaluated policies and additional metrics at 80\% coverage.
The ranking stability column reports the fraction of 200 bootstrap resamples in which each policy achieves the lowest AUSC\@.%
\footnote{We also evaluated a binary cross-boundary MI variant (MI computed on the binary safe-vs-critical partition), which achieves AUSC $0.284 \pm 0.015$, essentially tied with $C_{\text{crit\_max}}$ (mutual win rate 49.5\%/50.5\%).
It is excluded from the main comparison because it is not part of the $C_k$ framework and exhibits numerical instability ($\texttt{NaN}$ values when $\mu_{\text{safe}} \approx 0$ or $\mu_{\text{safe}} \approx 1$).}

\begin{table}[!ht]
\centering
\caption{Full bootstrap results for Bayesian EfficientNet-B4 ($S{=}30$, 200 resamples).
Policies ranked by AUSC (Critical FNR).}
\label{tab:full-bnn}
\resizebox{0.7\columnwidth}{!}{
\begin{tabular}{llcccc}
\toprule
Family & Policy & AUSC$\downarrow$ & Critical FNR @80\%$\downarrow$ & Acc @80\%$\uparrow$ & Win\% \\
\midrule
Per-class $C_k$ & $C_{\text{crit\_max}}$ & $0.285 \pm 0.016$ & $0.302 \pm 0.013$ & $0.827 \pm 0.005$ & 50.5 \\
Per-class $C_k$ & $C_{\text{crit\_sum}}$ & $0.327 \pm 0.017$ & $0.321 \pm 0.014$ & $0.831 \pm 0.005$ & 0.0 \\
Boundary & CBEC & $0.416 \pm 0.020$ & $0.335 \pm 0.014$ & $0.840 \pm 0.005$ & 0.0 \\
Scalar & MI & $0.436 \pm 0.019$ & $0.339 \pm 0.014$ & $0.838 \pm 0.005$ & 0.0 \\
Per-class MI & OvA\_MI & $0.452 \pm 0.017$ & $0.367 \pm 0.015$ & $0.838 \pm 0.005$ & 0.0 \\
Per-class var. & Sale\_EU\_global & $0.457 \pm 0.018$ & $0.341 \pm 0.015$ & $0.840 \pm 0.005$ & 0.0 \\
Scalar & Entropy & $0.604 \pm 0.022$ & $0.401 \pm 0.016$ & $0.842 \pm 0.005$ & 0.0 \\
Per-class var. & Var\_crit & $0.606 \pm 0.014$ & $0.379 \pm 0.016$ & $0.842 \pm 0.005$ & 0.0 \\
Scalar & MaxProb & $0.639 \pm 0.022$ & $0.439 \pm 0.017$ & $0.846 \pm 0.005$ & 0.0 \\
Per-class var. & Sale\_EU\_crit & $0.650 \pm 0.013$ & $0.409 \pm 0.016$ & $0.843 \pm 0.005$ & 0.0 \\
\bottomrule
\end{tabular}}
\end{table}

\paragraph{Pairwise statistical significance.}
Table~\ref{tab:pairwise-bnn} reports bootstrap pairwise comparisons for the proposed methods against key baselines.
$P(\text{row} < \text{col})$ denotes the fraction of bootstrap resamples in which the row policy achieves lower AUSC than the column policy; values above 0.975 indicate significance at the 5\% level (two-tailed).

\begin{table}[!ht]
\centering
\caption{Bootstrap pairwise comparisons (BNN). Entry = $P(\text{row AUSC} < \text{col AUSC})$.}
\label{tab:pairwise-bnn}
\begin{tabular}{lccc}
\toprule
& MI & Var\_crit & Sale\_EU\_crit \\
\midrule
$C_{\text{crit\_max}}$ & $1.000^{***}$ & $1.000^{***}$ & $1.000^{***}$ \\
$C_{\text{crit\_sum}}$ & $1.000^{***}$ & $1.000^{***}$ & $1.000^{***}$ \\
CBEC & $0.950$  & $1.000^{***}$ & $1.000^{***}$ \\
\bottomrule
\end{tabular}
\end{table}

\subsection{Epistemic Profiles}\label{app:profiles}

Table~\ref{tab:profiles} reports the normalised epistemic mass $\E[C_k / \sum_j C_j \mid y = i]$, showing which classes absorb epistemic uncertainty for each true label.
Figure~\ref{fig:profiles-app} visualises both the raw and normalised profiles as heatmaps.

\begin{figure}[!ht]
  \centering
  \includegraphics[width=\columnwidth]{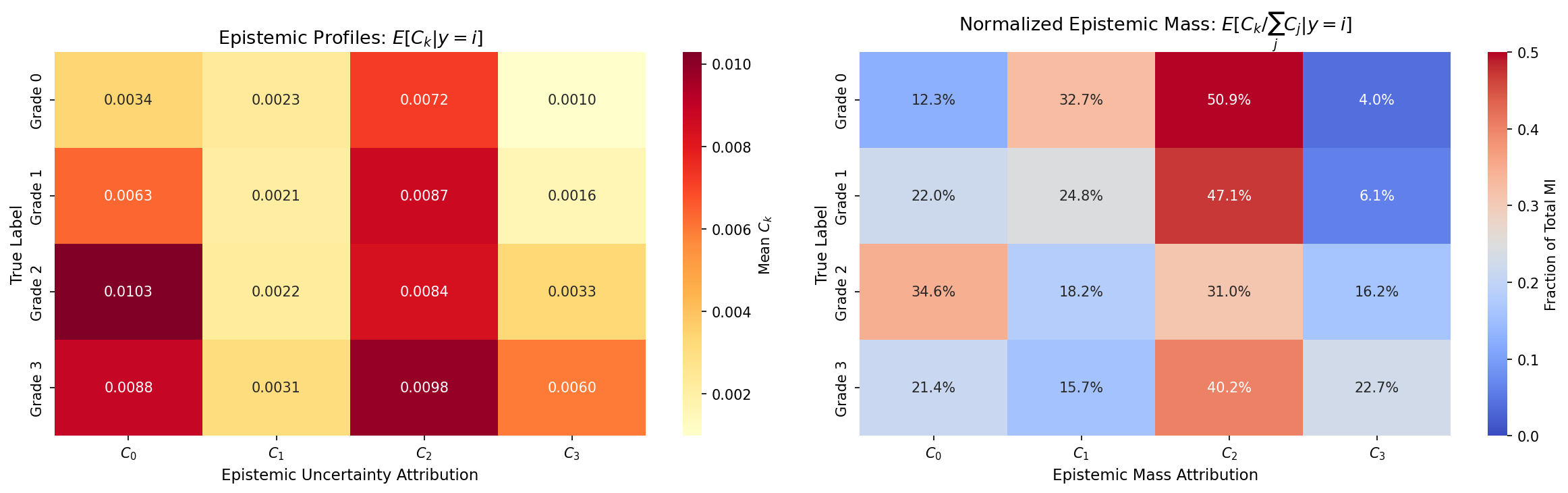}
  \caption{Epistemic profiles.
  \textbf{Left:} Raw $\E[C_k \mid y{=}i]$.
  \textbf{Right:} Normalised $\E[C_k / \sum_j C_j \mid y{=}i]$.
  Grade~2 dominates the epistemic budget across all true classes, identifying moderate DR as the model's primary source of confusion.}
  \label{fig:profiles-app}
\end{figure}

\begin{table}[!ht]
\centering
\caption{Normalised epistemic profiles: $\E[C_k / \sum_j C_j \mid y{=}i]$ (Bayesian model).
Each row sums to 100\%.
Bold = dominant component.}
\label{tab:profiles}
\begin{tabular}{lcccc}
\toprule
True class & $C_0$ & $C_1$ & $C_2$ & $C_3$ \\
\midrule
Grade 0 & 12.3\% & 32.7\% & \textbf{50.9\%} & 4.0\% \\
Grade 1 & 22.0\% & 24.8\% & \textbf{47.1\%} & 6.1\% \\
Grade 2 & \textbf{34.6\%} & 18.2\% & 31.0\% & 16.2\% \\
Grade 3 & 21.4\% & 15.7\% & \textbf{40.2\%} & 22.7\% \\
\bottomrule
\end{tabular}
\end{table}

Grade~2 is the dominant source of epistemic uncertainty across all true classes, consistent with it being the most confusable category (moderate DR is clinically subtle).
For true Grade~0 and Grade~1, over 50\% and 47\% of the epistemic mass concentrates on $C_2$ respectively, indicating the model's primary source of doubt about safe-class images is whether they might be moderate DR.
For true Grade~3, the epistemic mass is split between $C_2$ (40.2\%) and $C_3$ (22.7\%), reflecting the severity continuum.
Grade~2 is also the dominant epistemic class ($\arg\max_k C_k$) for 64.9\% of all test samples, confirming that moderate DR is the model's primary source of confusion at the individual-input level as well.

\subsection{Error Signatures}\label{app:error-signatures}

A distinctive advantage of $\vC(x)$ over scalar MI is that it fingerprints the \emph{structure} of model confusion.
Figure~\ref{fig:error-signatures} shows $\E[C_k \mid y{=}i, \hat{y}{=}j]$ for five clinically important outcomes, and Table~\ref{tab:error-sigs} provides the numerical values.

\begin{figure}[!ht]
  \centering
  \includegraphics[width=\columnwidth]{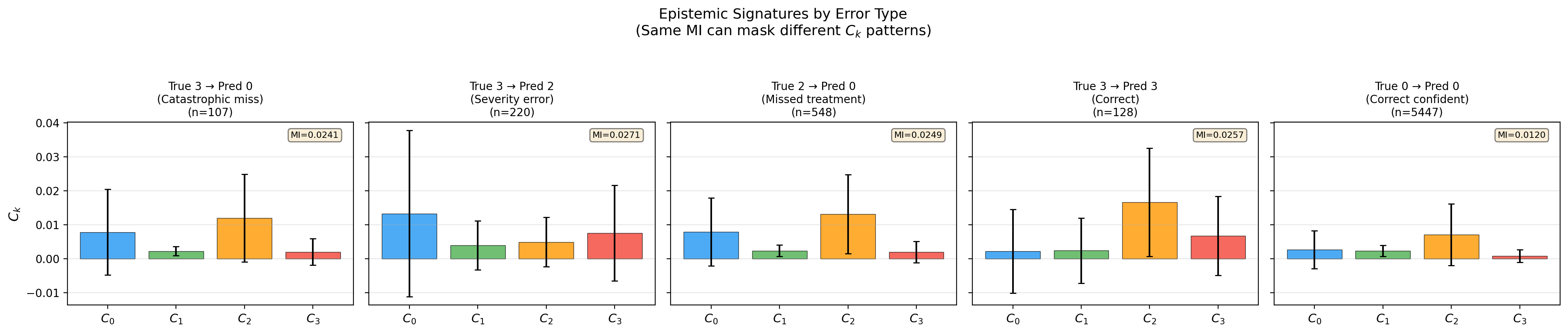}
  \caption{Epistemic signatures of error types.
  Each bar shows $\E[C_k \mid y{=}i, \hat{y}{=}j]$ for class $k \in \{0,1,2,3\}$.
  Catastrophic misses ($3 \to 0$) and severity errors ($3 \to 2$) have similar MI but distinct $C_k$ fingerprints, revealing different confusion pathways invisible to scalar metrics.}
  \label{fig:error-signatures}
\end{figure}

\begin{table}[!ht]
\centering
\caption{Error epistemic signatures: $\E[C_k \mid y{=}i, \hat{y}{=}j]$ for key error types.
Errors with nearly identical scalar MI exhibit different $C_k$ profiles, revealing distinct confusion pathways that suggest different remediation strategies.}
\label{tab:error-sigs}
\resizebox{0.7\columnwidth}{!}{
\begin{tabular}{lccccccc}
\toprule
Error & $n$ & MI & $C_0$ & $C_1$ & $C_2$ & $C_3$ & Dominant \\
\midrule
$3 \to 0$ (catastrophic) & 107 & 0.024 & 0.008 & 0.002 & \textbf{0.012} & 0.002 & $C_2$ \\
$3 \to 2$ (severity) & 220 & 0.027 & \textbf{0.013} & 0.004 & 0.005 & 0.008 & $C_0$ \\
$2 \to 0$ (missed treat.) & 548 & 0.025 & 0.008 & 0.002 & \textbf{0.013} & 0.002 & $C_2$ \\
$3 \to 3$ (correct) & 128 & 0.026 & 0.002 & 0.002 & \textbf{0.017} & 0.007 & $C_2$ \\
$0 \to 0$ (correct conf.) & 5447 & 0.012 & 0.003 & 0.002 & \textbf{0.007} & 0.001 & $C_2$ \\
\bottomrule
\end{tabular}}
\end{table}

The key interpretability result is that errors with similar scalar MI have distinctive $C_k$ fingerprints.
Catastrophic misses (true Grade~3 predicted as Grade~0, $n{=}107$) and severity underestimates (true Grade~3 predicted as Grade~2, $n{=}220$) have nearly identical scalar MI ($0.024$ vs.\ $0.027$~nats) but quite different $C_k$ signatures.
The catastrophic miss concentrates epistemic mass on $C_2$: the model does not jump directly from severe to healthy, but routes through moderate DR, identifying Grade~2 as the ``bottleneck'' of confusion.
The severity underestimate instead elevates $C_0$, indicating the model's primary doubt is whether the image might be healthy.
These qualitatively different failure modes, invisible to any scalar metric, suggest distinct intervention strategies: additional moderate-vs-severe training data for the first, and better healthy-vs-severe discrimination for the second.

\subsection{Epistemic Confusion Matrix}\label{app:confusion}

Figure~\ref{fig:confusion-app} visualises the epistemic confusion matrix $E_{ij} = \E[\sqrt{C_i C_j}\,\max(0,-\rho_{ij})]$, which quantifies pairwise confusion between classes using the same correlation-gated geometric mean that underlies CBEC (Eq.~\ref{eq:cbec}).
Table~\ref{tab:epi-confusion} provides the numerical values.

The cross-boundary block (safe$\times$critical) dominates: its mean value
is $2.7\times$ the within-safe mean and $10.6\times$ the within-critical
mean, confirming that the model's epistemic uncertainty concentrates at
the clinically dangerous boundary.
The strongest pairwise confusion is Grade~0$\leftrightarrow$Grade~2
($5.08 \times 10^{-3}$), consistent with the epistemic profiles (Table~\ref{tab:profiles}) showing
Grade~2 as the dominant source of uncertainty for healthy images.
This structure validates both the safe/critical partition and the CBEC
aggregation, which specifically targets cross-boundary confusion.

\begin{figure}[!ht]
  \centering
  \includegraphics[width=0.65\columnwidth]{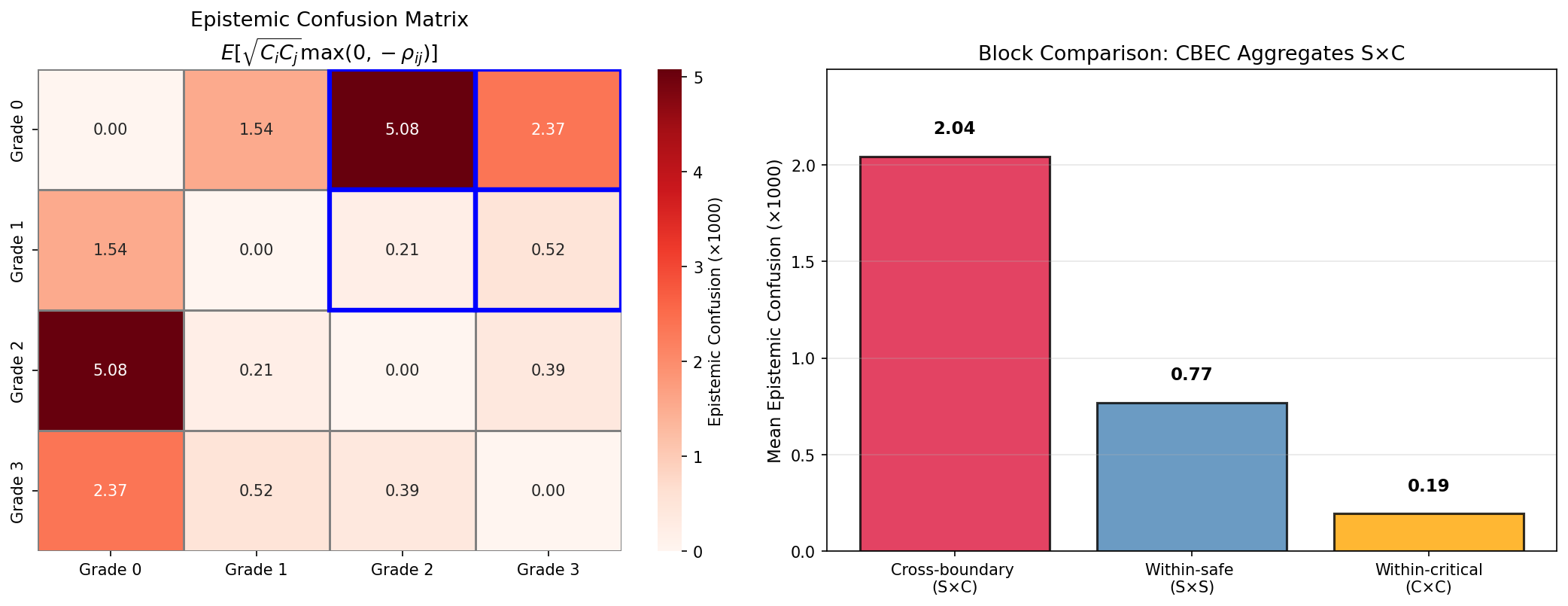}
  \caption{Epistemic confusion matrix $E_{ij} = \E[\sqrt{C_i C_j}\,\max(0,-\rho_{ij})]$.
  The cross-boundary block (safe$\times$critical, blue outline) dominates, with the Grade~0$\leftrightarrow$Grade~2 pair ($5.08 \times 10^{-3}$) exhibiting the strongest confusion.
  Within-critical confusion is an order of magnitude weaker ($10.6\times$ smaller mean), confirming that the model's epistemic uncertainty concentrates at the clinically dangerous safe/critical boundary.}
  \label{fig:confusion-app}
\end{figure}

\begin{table}[!ht]
\centering
\caption{Epistemic confusion matrix ($\times 10^3$, Bayesian model).
Horizontal rule separates the safe (Grades~0--1) and critical (Grades~2--3) partitions; entries in the off-diagonal blocks constitute the cross-boundary confusion aggregated by CBEC\@.}
\label{tab:epi-confusion}
\begin{tabular}{lcccc}
\toprule
& Gr.~0 & Gr.~1 & Gr.~2 & Gr.~3 \\
\midrule
Grade 0 & 0.00 & 1.54 & \textbf{5.08} & 2.37 \\
Grade 1 & 1.54 & 0.00 & 0.21 & 0.52 \\
\midrule
Grade 2 & \textbf{5.08} & 0.21 & 0.00 & 0.39 \\
Grade 3 & 2.37 & 0.52 & 0.39 & 0.00 \\
\bottomrule
\end{tabular}
\end{table}

\subsection{Full Bootstrap Results: MC Dropout}\label{app:mcdropout}

Table~\ref{tab:full-mcdropout} reports the MC~dropout evaluation.
The ranking reversal relative to the BNN (Table~\ref{tab:full-bnn}) is consistent with both the skewness analysis (Appendix~\ref{app:skewness}) and the findings of \citet{smith2018understanding}, who showed that MC~dropout underestimates posterior uncertainty and creates regions of spurious confidence.
The coarser posterior samples inflate $\rho_k$ for critical classes, degrading the reliability of individual $C_k$ values and hence $C_{\text{crit\_max}}$, while leaving CBEC's correlation-gated structure intact, since CBEC computes correlations directly from the MC draws rather than relying on the second-order approximation.

\begin{table}[!ht]
\centering
\caption{Full bootstrap results for MC Dropout ($p{=}0.3$, $S{=}30$, 200 resamples).
CBEC achieves the lowest AUSC, winning 100\% of bootstrap samples, while $C_{\text{crit\_max}}$ drops to fifth place.
The ranking reversal relative to Table~\ref{tab:full-bnn} reflects MC~dropout's inflated skewness for critical classes (Figure~\ref{fig:mcdropout-reliability}).}
\label{tab:full-mcdropout}
\resizebox{0.7\columnwidth}{!}{
\begin{tabular}{llccc}
\toprule
Family & Policy & AUSC$\downarrow$ & Critical FNR @80\%$\downarrow$ & Acc @80\%$\uparrow$ \\
\midrule
Boundary & \textbf{CBEC} & $\mathbf{0.197 \pm 0.012}$ & $0.316 \pm 0.015$ & $0.867 \pm 0.004$ \\
Scalar & Entropy & $0.249 \pm 0.017$ & $\mathbf{0.294 \pm 0.014}$ & $0.871 \pm 0.004$ \\
Scalar & MaxProb & $0.268 \pm 0.017$ & $0.340 \pm 0.016$ & $0.874 \pm 0.004$ \\
Per-class var. & Sale\_EU\_global & $0.364 \pm 0.022$ & $0.410 \pm 0.018$ & $0.870 \pm 0.004$ \\
Per-class $C_k$ & $C_{\text{crit\_max}}$ & $0.419 \pm 0.025$ & $0.390 \pm 0.017$ & $0.854 \pm 0.004$ \\
Per-class var. & Var\_crit & $0.424 \pm 0.026$ & $0.451 \pm 0.018$ & $0.868 \pm 0.004$ \\
Scalar & MI & $0.425 \pm 0.024$ & $0.459 \pm 0.019$ & $0.864 \pm 0.004$ \\
Per-class $C_k$ & $C_{\text{crit\_sum}}$ & $0.436 \pm 0.026$ & $0.404 \pm 0.017$ & $0.860 \pm 0.004$ \\
Per-class var. & Sale\_EU\_crit & $0.463 \pm 0.027$ & $0.527 \pm 0.019$ & $0.866 \pm 0.004$ \\
Per-class MI & OvA\_MI & $0.526 \pm 0.028$ & $0.556 \pm 0.020$ & $0.861 \pm 0.004$ \\
\bottomrule
\end{tabular}}
\end{table}

CBEC reduces AUSC by 53.6\% relative to MI ($0.197$ vs.\ $0.425$) and wins 100\% of bootstrap samples.
Scalar entropy ($0.249$) outperforms all epistemic methods including MI ($0.425$) under MC~dropout; this is because entropy captures total uncertainty (aleatoric + epistemic), and under MC~dropout's coarser posterior, the aleatoric component carries complementary signal that pure epistemic measures miss.
$C_{\text{crit\_max}}$ still outperforms its unnormalised counterpart Sale\_EU\_crit ($0.419$ vs.\ $0.463$), confirming that the $1/\mu_k$ normalisation provides benefit even when the Taylor approximation is degraded.

\begin{figure}[!ht]
  \centering
  \includegraphics[width=\columnwidth]{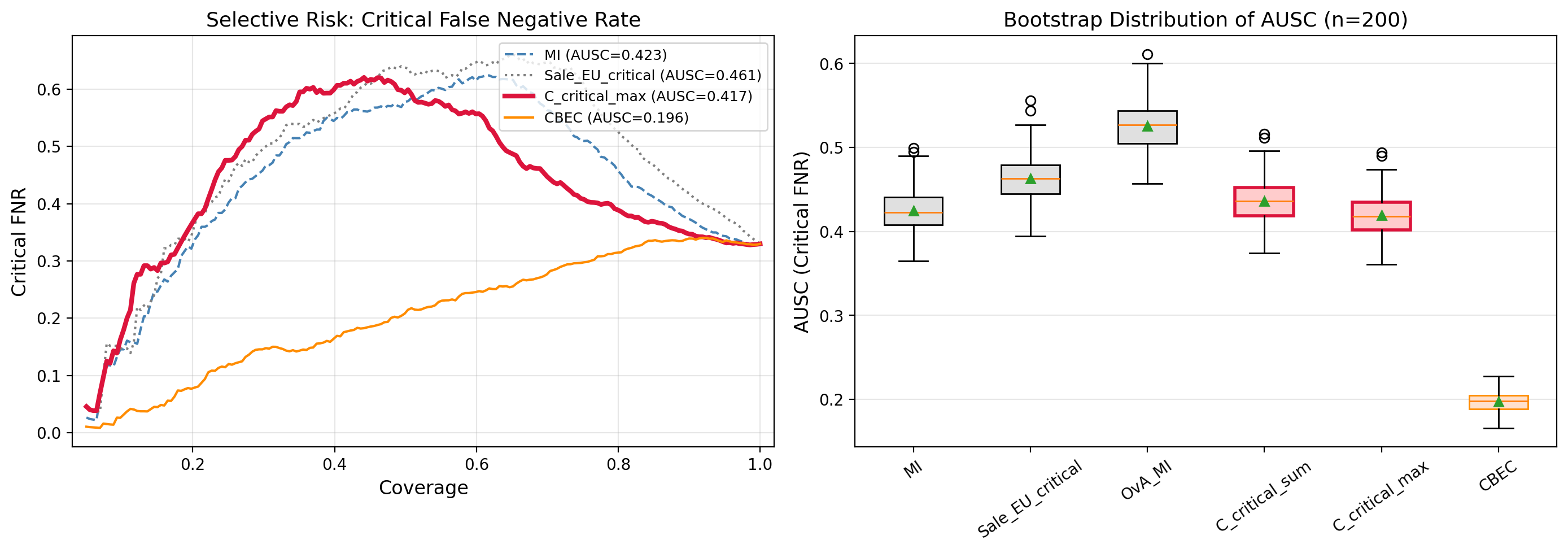}
  \caption{MC~dropout selective prediction on diabetic retinopathy.
  \textbf{Left:} Critical FNR vs.\ coverage curves. CBEC (orange) dominates at all coverage levels; $C_{\text{crit\_max}}$ (red) matches MI at low coverage but diverges at mid-range.
  \textbf{Right:} Bootstrap distribution of AUSC ($n{=}200$). CBEC achieves the lowest AUSC with the tightest spread; $C_{\text{crit\_max}}$ and MI are statistically indistinguishable.}
  \label{fig:mcdropout-selective}
\end{figure}

\paragraph{Why $C_k$ degrades under MC~dropout: a skewness analysis.}
Figure~\ref{fig:mcdropout-reliability} reveals the mechanism behind $C_{\text{crit\_max}}$'s ranking reversal.
Under the low-rank BNN, the fraction of samples with reliable $C_k$ ($\rho_k < 0.3$) remains above $80\%$ for Grades~0--2 and drops to $63\%$ only for Grade~3 (Severe/PDR), the rarest and most difficult class.
Under MC~dropout, the same statistic degrades monotonically with clinical severity: $88\%$ (Grade~0), $72\%$ (Grade~1), $54\%$ (Grade~2), and only $22\%$ (Grade~3).

The right panel quantifies the gap.
Median $\rho_k$ is comparable between the two models at Grade~0 ($0.013$ vs.\ $0.012$), where both posteriors are well-behaved.
The gap widens progressively: at Grade~1, MC~dropout's median $\rho_k$ is $3.5\times$ higher than the BNN's ($0.130$ vs.\ $0.037$); at Grade~2 it is $2.1\times$ higher ($0.264$ vs.\ $0.124$); and at Grade~3 the MC~dropout median exceeds the reliability boundary entirely ($1.125$ vs.\ $0.218$).
For the most safety-critical class, more than three quarters of MC~dropout samples operate in the regime where the second-order Taylor approximation has broken down.
This pattern is consistent with MC~dropout's known tendency to underestimate tail variance~\citep{smith2018understanding}: the resulting predictive distributions are more concentrated than the true posterior, but with heavier skewness, pushing $\rho_k$ above the reliability threshold precisely for the classes that matter most.

\begin{figure}[!ht]
  \centering
  \includegraphics[width=\columnwidth]{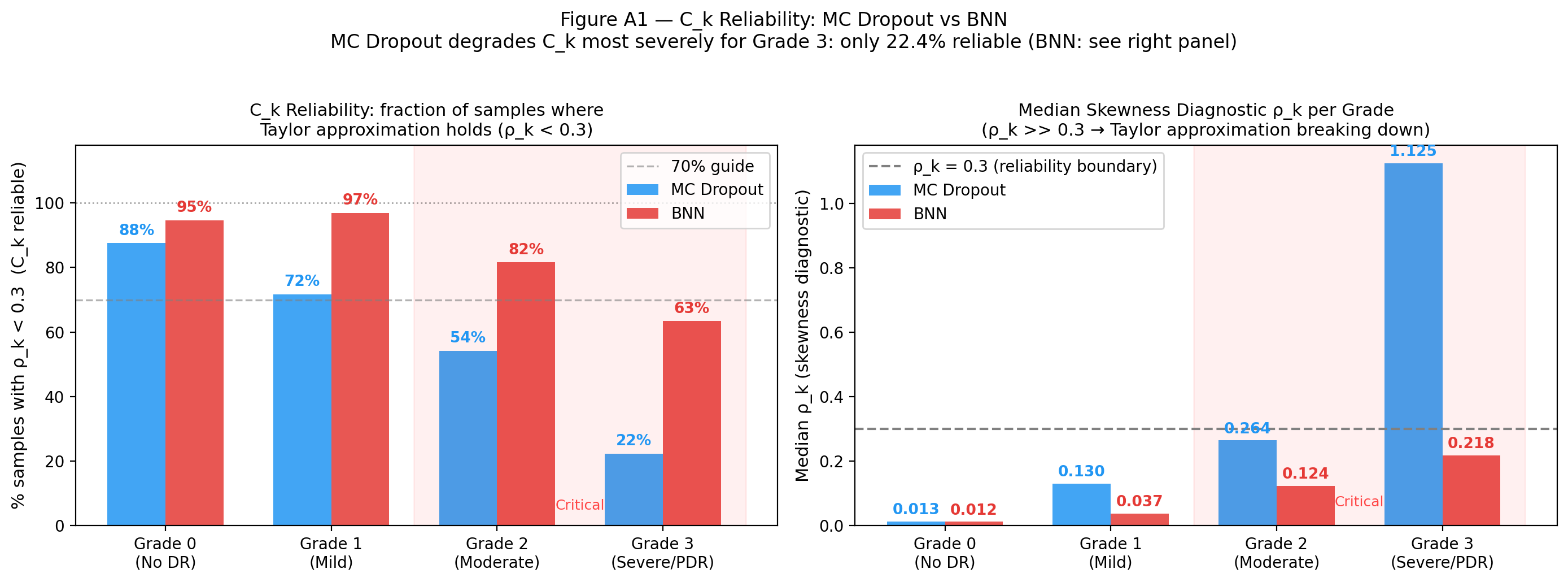}
  \caption{$C_k$ reliability comparison: MC~dropout vs.\ low-rank BNN on diabetic retinopathy.
  \textbf{Left:} Fraction of test samples with $\rho_k < 0.3$ (reliable $C_k$) per DR grade.
  The BNN maintains reliability above $60\%$ for all grades; MC~dropout drops to $22\%$ for Grade~3.
  Dashed line: $70\%$ guide.
  \textbf{Right:} Median skewness diagnostic $\rho_k$ per grade.
  MC~dropout exceeds the $\rho_k = 0.3$ reliability boundary at Grade~3 (median $1.125$), while the BNN remains below it ($0.218$).
  Critical grades (2, 3) are shaded.}
  \label{fig:mcdropout-reliability}
\end{figure}

These results provide direct empirical support for the complementarity claimed in Section~\ref{sec:exp_dr}: $C_{\text{crit\_max}}$ is the preferred deferral metric under well-calibrated posteriors where $\rho_k$ remains low, while CBEC provides a robust fallback when the posterior approximation inflates skewness beyond the Taylor regime.
The skewness diagnostic $\rho_k$ therefore serves as a practical model-selection criterion: if median $\rho_k$ for the critical classes exceeds $0.3$, practitioners should prefer CBEC over $C_{\text{crit\_max}}$.
\subsection{Deep Ensemble Validation: Per-Class Metrics}
\label{app:deep-ensemble}

To verify that the benefit of the $1/\mu_k$ normalisation persists across
inference paradigms, we evaluate a deep ensemble of five members on the
same diabetic retinopathy test set ($7{,}948$ images), using the identical
evaluation pipeline.
Each member shares the Phase~1 EfficientNet-B4 backbone and is fine-tuned
independently with a different random seed; diversity arises from head
initialisation rather than stochastic inference.
We restrict the comparison to MI as the scalar epistemic reference and the
full per-class family: raw per-class variance baselines from
\citet{sale2024labelwise} and our proposed $C_k$ metrics.

\paragraph{Reliability.}
The deep ensemble achieves near-uniform $C_k$ reliability across all DR
grades: $95.7\%$, $91.2\%$, $91.5\%$, and $85.1\%$ of samples satisfy
$\rho_k < 0.3$ for Grades~0--3 respectively.
This places the ensemble firmly in the high-reliability regime, in
contrast to MC~dropout where Grade~3 reliability collapses to $22\%$
(Figure~\ref{fig:mcdropout-reliability}).

\paragraph{Results.}
Table~\ref{tab:deep-ensemble-sale} reports AUSC and critical FNR at
$80\%$ coverage over 200 bootstrap resamples.

\begin{table}[!ht]
\centering
\caption{Deep ensemble: per-class metrics. AUSC and critical FNR at
$80\%$ coverage ($\downarrow$ better); 200 bootstrap resamples.}
\label{tab:deep-ensemble-sale}
\small
\begin{tabular}{lcc}
\toprule
Policy & AUSC$\downarrow$ & Crit.\ FNR @80\%$\downarrow$ \\
\midrule
Var\_crit                            & $0.408 \pm 0.030$ & $0.362 \pm 0.017$ \\
Sale\_EU\_crit~\citep{sale2024labelwise} & $0.447 \pm 0.031$ & $0.414 \pm 0.018$ \\
$C_{\text{crit\_max}}$              & $0.390 \pm 0.029$ & $0.314 \pm 0.016$ \\
$C_{\text{crit\_sum}}$              & $0.406 \pm 0.029$ & $0.333 \pm 0.016$ \\
\textbf{CBEC}                        & $\mathbf{0.223 \pm 0.018}$ & $\mathbf{0.237 \pm 0.013}$ \\
\bottomrule
\end{tabular}
\end{table}

Two findings are consistent across all three inference regimes.
First, $C_{\text{crit\_max}}$ outperforms its unnormalised counterpart
$\text{Sale\_EU\_crit}$ by $12.9\%$ in AUSC ($0.390$ vs.\ $0.447$),
confirming that the $1/\mu_k$ normalisation provides a systematic
advantage over raw per-class variance regardless of the inference method.
Second, CBEC achieves the lowest AUSC and critical FNR, outperforming MI
by $37.0\%$ ($0.223$ vs.\ $0.354$) and both Sale et al.\ variants on
every metric.
Notably, all three proposed $C_k$ metrics outperform Sale\_EU\_crit,
their direct unnormalised counterpart, confirming the boundary-correction
benefit of the $1/\mu_k$ normalisation even in the high-reliability
ensemble regime where the Taylor approximation is near-exact
($\sum_k C_k / \text{MI} = 0.996$ on this test set).
\section{OoD Detection Experiment: Supplementary Details}
\label{app:ood-exp}

\subsection{Datasets and Preprocessing}\label{app:ood-data}

\paragraph{FashionMNIST $\to$ KMNIST.}
FashionMNIST~\citep{xiao2017fashion} comprises 60{,}000 grayscale images ($28 \times 28$) across 10 clothing categories; in code, the 60{,}000 training images are split into 50{,}000 train and 10{,}000 validation, and the standard 10{,}000 FashionMNIST test set is used as the in-distribution test set.
KMNIST~\citep{clanuwat2018deep} provides 10{,}000 test images of cursive Japanese characters in the same format, used entirely as OoD data.
Preprocessing is identical for ID/OoD images: cast to \texttt{float32}, normalized by dividing pixel values by 126, and flattened each $28\times 28$ image to a 784-dimensional vector; no data augmentation is used.

\paragraph{MIMIC-III ICU $\to$ Newborn.}
We extract adult ICU admissions from MIMIC-III~\citep{johnson2016mimic}, yielding 44 clinical features with binary mortality labels, split into train/test = 40{,}406/4{,}490.
The Newborn unit (5{,}357 admissions) serves as OoD data.
The 44-feature pipeline merges (on \texttt{SUBJECT\_ID}, \texttt{HADM\_ID}, \texttt{ICUSTAY\_ID}): (i) static patient features (including age, ICU LOS, pre-ICU hospital time, gender one-hot), (ii) per-stay mean/std aggregates of mapped vital-sign ITEMIDs, and (iii) per-stay mean/std aggregates of mapped lab ITEMIDs restricted to timestamps within ICU stay. Newborn admissions are excluded from ID data and kept as OoD. Adult-ID preprocessing applies an 8$\times$IQR outlier filter, plausibility filters (\texttt{mean\_combined\_bp\_dia} and \texttt{mean\_combined\_bp\_sys} $>10$, \texttt{time\_at\_hosp\_pre\_ic\_admission} $>0$), fills missing std-feature entries with 0, then uses median imputation and MinMax scaling fit on train and applied to test/OoD. In the OoD loader used for experiments, newborn \texttt{age} (and \texttt{weight}, if present) are set to the adult-training mean values before evaluation.

\subsection{Model Architecture and Training}\label{app:ood-model}

Both tasks use fully connected Bayesian neural networks with ReLU activations and low-rank Gaussian posteriors.
Output layers are linear (no activation), trained with logits, and mapped to class probabilities via softmax at inference.
All models share the following configuration:

\paragraph{Shared specification.}
The prior is a scale-mixture Gaussian~\citep{blundell2015weight} with $\pi{=}0.5$, $\sigma_1{=}1.0$, $\sigma_2{=}\exp(-6)$.
Posterior initialisation is adaptive per layer: $\sigma_w^2 = 2/d_{\text{in}}$ for ReLU layers, $2/(d_{\text{in}}{+}d_{\text{out}})$ otherwise; damping is $0.55$ for $r > 5$ (else $0.32$); factor means $A_\mu, B_\mu \sim \mathcal{U}(-a, a)$ with $a = \text{damping}\sqrt{3}(\sigma_w^2/r)^{1/4}$; bias means $b_\mu = 0$; and $A_\rho, B_\rho, b_\rho$ are set to constant $\rho_{\text{init}} = \log(\exp(\sigma_{\text{init}}) - 1)$ with $\sigma_{\text{init}} = 0.09(\sigma_w^2/r)^{1/4}$.
All models are trained with Adam at learning rate $10^{-3}$.

\paragraph{Per-task details.}
Table~\ref{tab:ood-hyperparams} summarises the task-specific hyperparameters.

\begin{table}[!ht]
\centering
\caption{OoD model configurations.}
\label{tab:ood-hyperparams}
\begin{tabular}{lcc}
\toprule
 & FashionMNIST & MIMIC-III \\
\midrule
Hidden layers & $2 \times 1{,}200$ & $2 \times 128$ \\
Rank $r$ & 25 & 15 \\
Output classes & 10 & 2 \\
Batch size & 128 & 64 \\
Epochs & 50 & 256 \\
KL schedule & Linear warmup, 10 epochs & Fixed \\
KL scale & $0 \to 1/N_{\text{train}}$ & $0.5/\lceil N_{\text{train}}/\text{batch}\rceil$ \\
Class weighting & None & $w_0{=}1,\; w_1{=}1/\text{pos\_frac} \approx 11.9$ \\
MC samples $S$ & 50 & 512 \\
\bottomrule
\end{tabular}
\end{table}

Results are reported as mean $\pm$ std over 5 independent training seeds.

\subsection{Extended OoD Detection Results}\label{app:ood-details}

Table~\ref{tab:ood-distributional} reports per-metric distributional statistics for both benchmarks.
On FashionMNIST, MI and $EU_{\text{var}}$ share the same OoD/ID ratio ($5.92\times$) but MI achieves higher AUROC because its absolute values provide better threshold separation.
On MIMIC-III, $EU_{\text{var}}$ has the highest ratio ($1.71\times$) yet the lowest AUROC among epistemic measures ($0.778$), confirming that ratio alone does not determine detection performance when the dynamic range of scores is compressed by boundary suppression.

\begin{table}[!ht]
\centering
\caption{OoD distributional analysis: mean uncertainty and OoD/ID ratio.}
\label{tab:ood-distributional}
\resizebox{0.5\columnwidth}{!}{
\begin{tabular}{llccc}
\toprule
Dataset & Metric & Mean (ID) & Mean (OoD) & Ratio$\uparrow$ \\
\midrule
\multirow{4}{*}{FashionMNIST}
  & Neg.\ MSP & 0.0506 & 0.1043 & 2.07 \\
  & MI & 0.0096 & 0.0569 & 5.92 \\
  & $EU_{\text{var}}$ & 0.0056 & 0.0333 & 5.92 \\
  & $\sum_k C_k$ & 0.0106 & 0.0677 & \textbf{6.43} \\
\midrule
\multirow{4}{*}{MIMIC-III}
  & Neg.\ MSP & 0.2543 & 0.3150 & 1.24 \\
  & MI & 0.0378 & 0.0598 & 1.61 \\
  & $EU_{\text{var}}$ & 0.0289 & 0.0485 & 1.71 \\
  & $\sum_k C_k$ & 0.0377 & 0.0601 & \textbf{1.62} \\
\bottomrule
\end{tabular}}
\end{table}

\paragraph{Per-class decomposition (MIMIC-III).}
The binary setting permits direct inspection of per-class contributions: $\sum_k C_k = C_0 + C_1$.
Table~\ref{tab:mimic-perclass} reveals that distributional shift affects the two classes asymmetrically.
The survival class exhibits a $2.15\times$ increase from ID to OoD ($C_0$: $0.014 \to 0.030$), while the mortality class shows only $1.30\times$ ($C_1$: $0.022 \to 0.029$).
The critical-class-only metric $C_1$ achieves AUROC $0.740 \pm 0.074$, substantially below $\sum_k C_k$ at $0.815 \pm 0.017$: by focusing exclusively on mortality-class uncertainty, it misses the larger OoD signal carried by $C_0$.
This illustrates a task-dependent trade-off: $C_{\text{crit\_max}}$ targets safety-critical selective prediction (Section~\ref{sec:exp_dr}), where critical-class uncertainty directly determines clinical harm; for general OoD detection, all-class aggregation captures the complete distributional shift.

\begin{table}[!ht]
\centering
\caption{Per-class decomposition on MIMIC-III OoD detection.}
\label{tab:mimic-perclass}
\resizebox{0.5\columnwidth}{!}{
\begin{tabular}{lcccc}
\toprule
Metric & Mean (ID) & Mean (OoD) & Ratio$\uparrow$ & AUROC$\uparrow$ \\
\midrule
$C_0$ (survival) & 0.014 & 0.030 & 2.15 & $0.773 \pm 0.005$ \\
$C_1$ (mortality) & 0.022 & 0.029 & 1.30 & $0.740 \pm 0.074$ \\
\midrule
$\sum_k C_k$ & 0.038 & 0.060 & 1.62 & $\mathbf{0.815 \pm 0.017}$ \\
\bottomrule
\end{tabular}}
\end{table}

\subsection{Per-Class Epistemic Decomposition for FashionMNIST}
\label{app:perclass_fmnist}

The main text reports $\sum_k C_k$ as the best OoD metric for
FashionMNIST$\to$KMNIST (AUROC $0.735 \pm 0.009$) but presents only the
scalar aggregate.
Figure~\ref{fig:ck_fmnist} decomposes this into the ten individual
$C_k$ values, revealing which clothing categories drive the OoD signal
when the model encounters KMNIST inputs.

\begin{figure}[htbp]
  \centering
  \includegraphics[width=\linewidth]{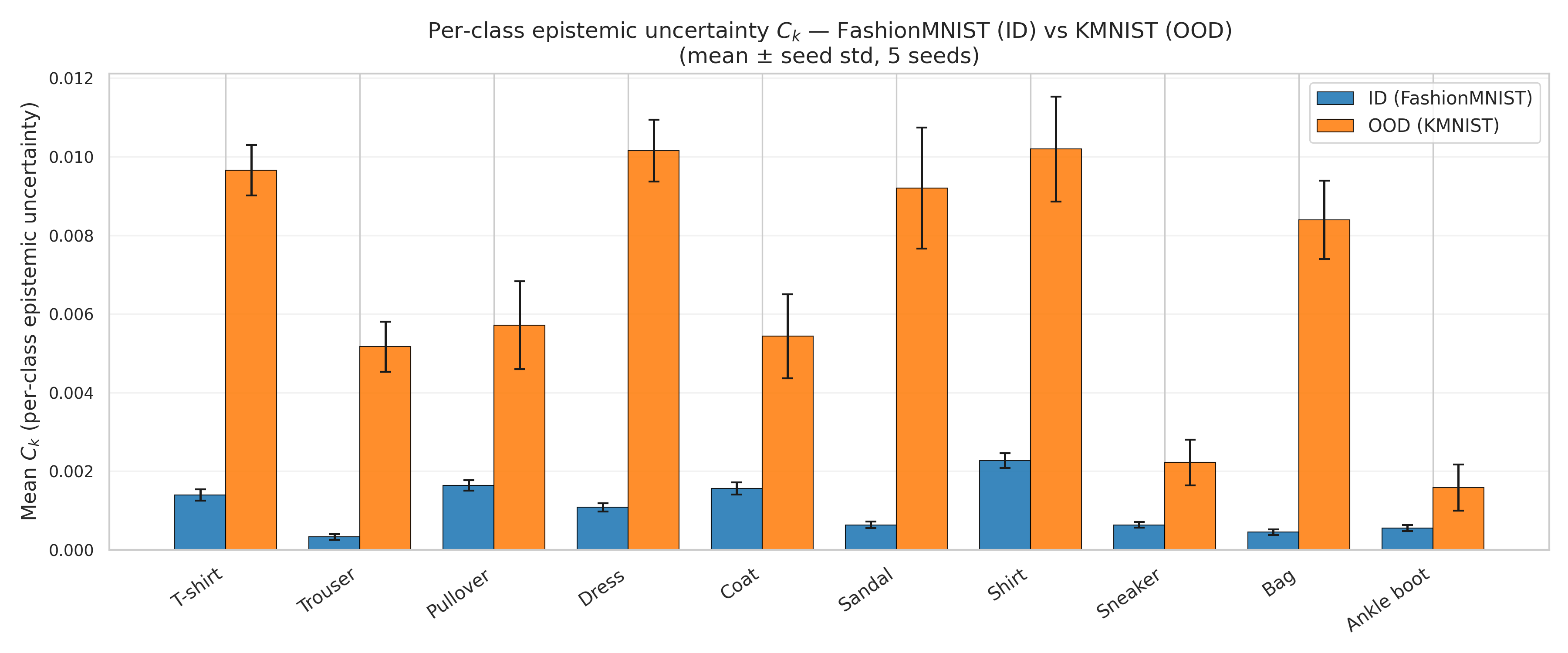}
  \caption{%
    Per-class epistemic uncertainty $C_k$ for each of the ten FashionMNIST
    categories, comparing ID (FashionMNIST test, $N=10{,}000$) vs.\ OoD
    (KMNIST test, $N=10{,}000$) samples.
    Bars: mean over 5 seeds; error bars: seed standard deviation.
     All ten classes show $\bar{C}_k^{\text{OoD}} >
    \bar{C}_k^{\text{ID}}$, though the magnitude varies considerably
    (Bag: $18.7\times$; Ankle boot: $2.8\times$; Table~\ref{tab:ck_fmnist}).
    The OoD signal is \emph{consistent in direction} across all categories
    but not uniform in magnitude, suggesting that KMNIST activates some
    FashionMNIST class representations more than others.
  }
  \label{fig:ck_fmnist}
\end{figure}

Table~\ref{tab:ck_fmnist} reports the numerical values for each class.
\begin{table}[htbp]
  \centering
  \caption{%
    Per-class mean $C_k$ for FashionMNIST (ID) vs.\ KMNIST (OoD),
    averaged over 5 seeds with seed standard deviation.
    Ratio $= \bar{C}_k^{\text{OoD}} / \bar{C}_k^{\text{ID}}$.
  }
  \label{tab:ck_fmnist}
  \small
  \begin{tabular}{lrrrrr}
    \toprule
    Class
      & $\bar{C}_k^{\text{ID}}$ & $\sigma^{\text{ID}}$
      & $\bar{C}_k^{\text{OoD}}$ & $\sigma^{\text{OoD}}$
      & Ratio \\
    \midrule
    T-shirt/top  & 0.00140 & 0.00014 & 0.00966 & 0.00064 &  6.92 \\
    Trouser      & 0.00033 & 0.00007 & 0.00517 & 0.00064 & 15.65 \\
    Pullover     & 0.00164 & 0.00013 & 0.00571 & 0.00112 &  3.49 \\
    Dress        & 0.00108 & 0.00010 & 0.01015 & 0.00079 &  9.37 \\
    Coat         & 0.00156 & 0.00015 & 0.00543 & 0.00106 &  3.48 \\
    Sandal       & 0.00064 & 0.00008 & 0.00920 & 0.00154 & 14.43 \\
    Shirt        & 0.00228 & 0.00019 & 0.01020 & 0.00133 &  4.48 \\
    Sneaker      & 0.00064 & 0.00007 & 0.00222 & 0.00058 &  3.49 \\
    Bag          & 0.00045 & 0.00007 & 0.00840 & 0.00099 & 18.67 \\
    Ankle boot   & 0.00056 & 0.00008 & 0.00158 & 0.00058 &  2.85 \\
    \midrule
    $\sum_k C_k$ (aggregate) & 0.01056 & -  & 0.06773 & - & 6.43 \\
    \bottomrule
  \end{tabular}
\end{table}
We additionally evaluate each individual $C_k(x)$ as a scalar OoD score.
Table~\ref{tab:ck_fmnist_auroc} shows that every class provides non-trivial
ID/OoD separation, with AUROC ranging from $0.537$ for Ankle boot to $0.683$
for Bag. These per-class AUROCs remain below the aggregate $\sum_k C_k$
AUROC of $0.735 \pm 0.009$, confirming that the shift is broad enough that
aggregation is the strongest detector, while the individual components reveal
where the shift is concentrated.

\begin{table}[htbp]
  \centering
  \caption{Per-class OoD detection on FashionMNIST$\to$KMNIST using each
  $C_k(x)$ as an individual OoD score. Ratios and AUROCs are mean $\pm$ std
  over 5 seeds.}
  \label{tab:ck_fmnist_auroc}
  \small
  \begin{tabular}{lcc}
    \toprule
    Class & OoD/ID ratio & AUROC using $C_k$ \\
    \midrule
    T-shirt/top & $6.971 \pm 0.788$ & $0.670 \pm 0.027$ \\
    Trouser     & $16.046 \pm 2.370$ & $0.656 \pm 0.024$ \\
    Pullover    & $3.477 \pm 0.549$ & $0.578 \pm 0.020$ \\
    Dress       & $9.428 \pm 1.005$ & $0.662 \pm 0.020$ \\
    Coat        & $3.481 \pm 0.602$ & $0.577 \pm 0.026$ \\
    Sandal      & $14.729 \pm 3.755$ & $0.650 \pm 0.019$ \\
    Shirt       & $4.467 \pm 0.337$ & $0.615 \pm 0.025$ \\
    Sneaker     & $3.494 \pm 0.906$ & $0.549 \pm 0.054$ \\
    Bag         & $18.854 \pm 2.173$ & $0.683 \pm 0.030$ \\
    Ankle boot  & $2.921 \pm 1.439$ & $0.537 \pm 0.050$ \\
    \bottomrule
  \end{tabular}
\end{table}
For reference, the aggregate $\sum_k C_k$ achieves an OoD/ID mean-score
ratio of $6.43\times$ (vs.\ $5.92\times$ for MI and $5.92\times$ for
$EU_{\text{var}}$), confirming that the $\mu_k$-normalisation in $C_k$
provides a consistent boost over the raw variance sum.

\subsection{\texorpdfstring{Score Distribution Analysis and the $EU_{\text{var}}$ Paradox}{Score Distribution Analysis and the EUvar Paradox}}
\label{app:distributions}

On MIMIC-III, $EU_{\text{var}}$ achieves the \emph{highest} mean OoD/ID
ratio ($1.705\times$) yet the \emph{lowest} AUROC ($0.778$) among the
three primary Bayesian metrics.
Figure~\ref{figures:distributions_mimic} makes this paradox immediately visual.

\begin{figure}[htbp]
  \centering
  \includegraphics[width=\linewidth]{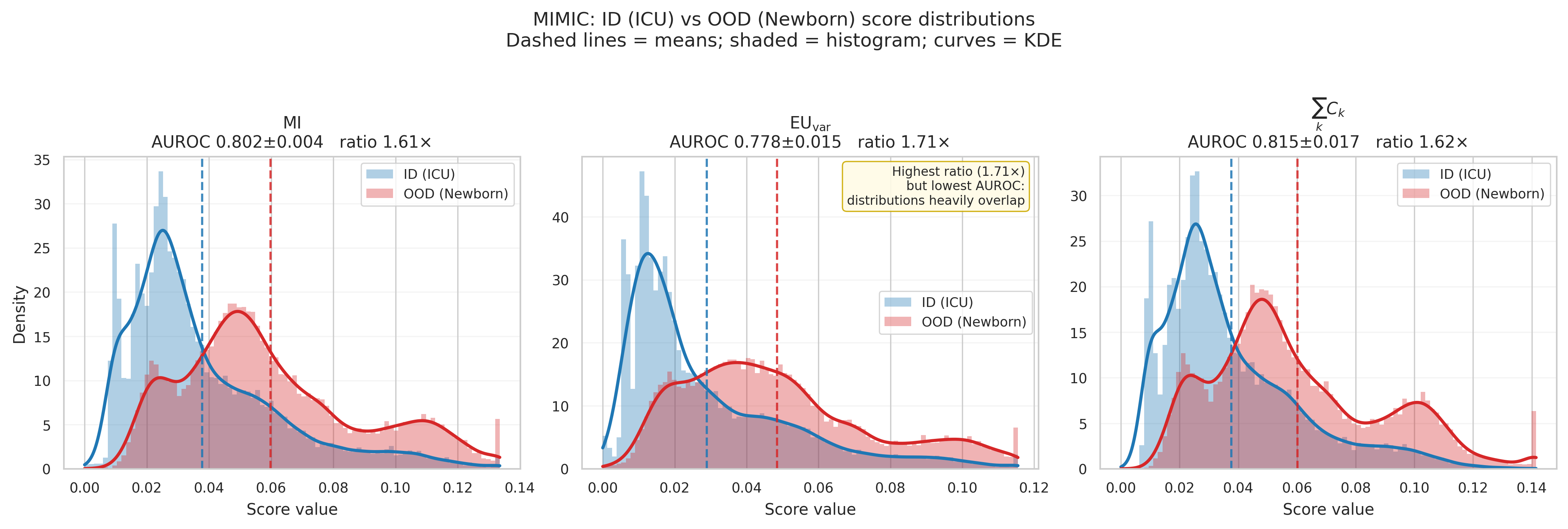}
  \caption{%
    MIMIC ICU (blue) vs.\ Newborn (red) score distributions for the three
    primary Bayesian metrics, pooled across 5 seeds
    (${\approx}22{,}000$ ID samples, ${\approx}27{,}000$ OoD samples).
    Shaded areas: histograms (90 bins, density-normalised).
    Solid curves: kernel density estimates (Scott's rule).
    Dashed verticals: distribution means.
    \textbf{Centre panel:} despite achieving the largest mean separation
    ($1.705\times$), both ID and OoD $EU_{\text{var}}$ distributions
    concentrate near zero with a shared heavy right tail.
    \textbf{Right panel:} $\sum_k C_k$ achieves better tail separation,
    yielding the highest AUROC ($0.815$).
  }
  \label{figures:distributions_mimic}
\end{figure}

\paragraph{Why a high mean ratio does not imply better separability.}
AUROC measures $P(s_{\mathrm{OoD}} > s_{\mathrm{ID}})$ for a randomly
chosen pair drawn independently from the OoD and ID score distributions.
This probability is governed by the \emph{full distributional overlap}, not
merely the ratio of means.
Under a Gaussian approximation of the log-scores, the relationship is
\begin{equation}
  \mathrm{AUROC}
  \;\approx\;
  \Phi\!\left(
    \frac{\mu_{\mathrm{OoD}} - \mu_{\mathrm{ID}}}
         {\sqrt{\sigma^2_{\mathrm{OoD}} + \sigma^2_{\mathrm{ID}}}}
  \right),
  \label{eq:auroc_gaussian_approx}
\end{equation}
where $\Phi$ is the standard-normal CDF.
A high OoD/ID mean ratio inflates the numerator, but if the within-group
standard deviations are also large, the denominator grows proportionally
and the AUROC gain is suppressed, analogous to Cohen's~$d$.

\paragraph{Why EU$_{\mathrm{var}}$ suffers from dynamic range compression.}
EU$_{\mathrm{var}} = \sum_k \mathrm{Var}[p_k]$ is an un-normalised
variance sum.
For a sigmoid binary classifier, predictive variance is
\emph{heteroscedastic}: samples near the decision boundary have high
$\mathrm{Var}[p_k]$ irrespective of ID/OoD status, simultaneously inflating
the right tail of \emph{both} distributions.
Although the OoD mean exceeds the ID mean by a factor of $1.71\times$,
the bulk of both distributions concentrates near zero with only a thin
right tail separating them.
The resulting large within-group spread keeps the denominator
of~\eqref{eq:auroc_gaussian_approx} large, yielding the lowest AUROC
($0.778$) among the three Bayesian metrics.

\paragraph{Why $\sum_k C_k$ achieves better separation with a lower ratio.}
$C_k = \mathrm{Var}[p_k]\,/\,(2\mu_k)$ normalises each variance component
by the corresponding mean probability.
This acts as a variance-stabilising transform: large absolute variances
that arise when $\mu_k$ is also large are down-weighted, while contributions
from classes with small $\mu_k$ are amplified.
The net effect narrows the within-group spread for both ID and OoD
populations, so a smaller mean ratio ($1.62\times$ vs.\ $1.71\times$)
still translates to cleaner tail separation and a higher AUROC ($0.815$).

\paragraph{Why the paradox is absent on FashionMNIST.}
On FashionMNIST$\to$KMNIST, $EU_{\text{var}}$ and MI attain nearly
identical mean ratios ($5.923\times$ vs.\ $5.921\times$), and their
AUROC gap is only $0.014$.
The distinction arises because KMNIST induces a strong, directionally
consistent uncertainty lift across all ten softmax outputs, preventing
the range compression that
afflicts $EU_{\text{var}}$ on MIMIC's binary sigmoid output.

\subsection{Empirical Verification of the Skewness Diagnostic}
\label{app:skewness_ood}

Lemma~\ref{lem:third-order} establishes that the expected entropy satisfies
\begin{equation}
  \mathbb{E}\bigl[\mathcal{H}(\mathbf{p})\bigr]
  \;\approx\;
  \mathcal{H}(\boldsymbol{\mu})
  - \frac{1}{2}\sum_k \frac{\mathrm{Var}[p_k]}{\mu_k}
  + \frac{1}{6}\sum_k \frac{m_{3,k}}{\mu_k^2},
  \label{eq:lem3}
\end{equation}
where $m_{3,k} = \frac{1}{S}\sum_{s=1}^S(p_k^{(s)} - \mu_k)^3$ is
the empirical third central moment.
The reliability of the second-order approximation (i.e.\ $\sum_k C_k$)
is governed by the skewness diagnostic
\begin{equation}
  \rho_k(x)
  = \frac{|m_{3,k}|}{3\,\mu_k \cdot \mathrm{Var}[p_k]}
  = \frac{\bigl|\text{3rd-order correction}\bigr|}
         {\bigl|\text{2nd-order correction}\bigr|},
  \label{eq:rho_app}
\end{equation}
so $\rho_k < 1$ signals a reliable quadratic regime and
$\rho_k \gg 1$ signals regime breakdown.
Note that $3\mu_k \cdot \mathrm{Var}[p_k] = 6\mu_k^2 C_k$, so
$\rho_k = |m_{3,k}| / (6\mu_k^2 C_k)$: the $1/\mu_k^2$ factor amplifies
$\rho_k$ for classes with small predicted probability, which is precisely
the singularity avoided by using $C_k$ rather than $C_k^{(3)}$.

The main text claims OoD inputs have higher $\rho_k$ than ID inputs.
Figures~\ref{fig:rho_mimic} and~\ref{fig:rho_fmnist} provide the first
direct empirical confirmation, with one instructive exception.

\paragraph{MIMIC-III ($K=2$).}

\begin{figure}[htbp]
  \centering
  \includegraphics[width=0.88\linewidth]{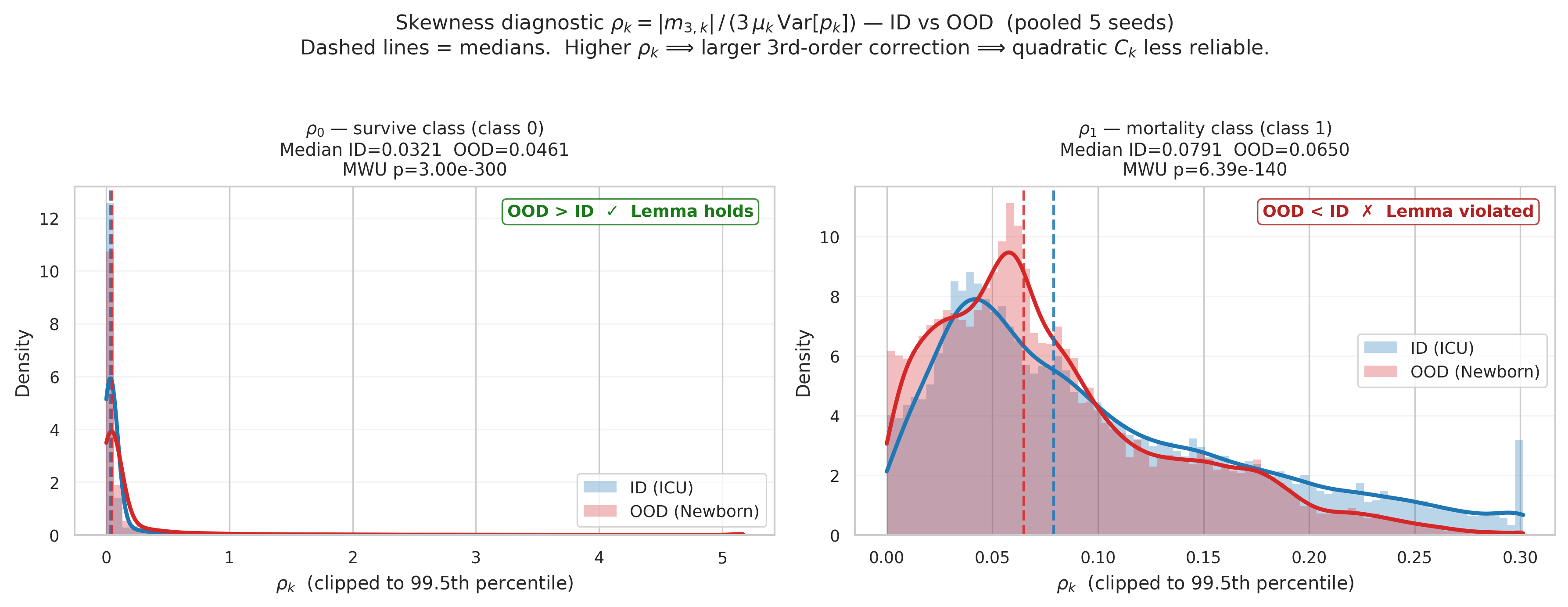}
  \caption{%
    Skewness diagnostic $\rho_k$ for ID (ICU, blue) and
    OoD (Newborn, red) on MIMIC, pooled across 5 seeds.
    Shaded areas: density-normalised histograms clipped to the
    99.5th percentile; solid curves: kernel density estimates;
    dashed verticals: medians.
    \textbf{Left ($\rho_0$, survive):} Both distributions concentrate
    well below $0.3$, confirming that $C_0$ is reliable for both
    populations; the median shift ($0.032 \to 0.046$) is statistically
    significant but practically small.
    \textbf{Right ($\rho_1$, mortality):} The OoD distribution shifts
    \emph{left} (median $0.065 < 0.079$), reversing the pattern
    predicted by Lemma~\ref{lem:third-order}; near-zero $\mu_1$ for
    newborns collapses $\Var[p_1]$ faster than $1/\mu_1$ compensates,
    concentrating the third-order signal in class~0.}
  
  \label{fig:rho_mimic}
\end{figure}

\begin{table}[htbp]
  \centering
  \caption{%
    MIMIC: median and mean of $\rho_k$ for ID (ICU) and OoD (Newborn)
    populations, pooled across 5 seeds.
    Mann--Whitney $U$ test, two-sided.
  }
  \label{tab:rho_mimic}
  \small
  \begin{tabular}{lcccccc}
    \toprule
    & \multicolumn{2}{c}{Median}
    & \multicolumn{2}{c}{Mean}
    & \multirow{2}{*}{MWU $p$}
    & \multirow{2}{*}{Direction} \\
    \cmidrule(lr){2-3}\cmidrule(lr){4-5}
    {Class} & {ID} & {OoD} & {ID} & {OoD} & & \\
    \midrule
    $\rho_0$ (survive)   & 0.032 & 0.046 & 0.142 & 0.222 & $< 10^{-300}$ & OoD $>$ ID\;$\checkmark$ \\
    $\rho_1$ (mortality) & 0.079 & 0.065 & 0.099 & 0.080 & $< 10^{-130}$ & OoD $<$ ID\;$\times$    \\
    \bottomrule
  \end{tabular}
\end{table}

\paragraph{The $\rho_1$ reversal on MIMIC.}
The mortality class ($k{=}1$) shows the \emph{opposite} direction to the
Lemma~\ref{lem:third-order} prediction: OoD newborns have \emph{lower}
$\rho_1$ than ICU patients (median $0.065 < 0.079$, $p < 10^{-130}$).
The mechanism is coherent: the model learns to predict near-zero mortality
probability for newborns, a structurally healthy population.
Consequently, all $S = 512$ MC samples pile up near $p_1 \approx 0$,
producing a distribution over $p_1$ that is \emph{tight} (small
$\mathrm{Var}[p_1]$) and slightly right-skewed, so the third central
moment $|m_{3,1}|$ collapses faster than the $1/\mu_1$ amplification in
the denominator of $\rho_1$ can compensate.
The OoD detection signal for MIMIC lives entirely in class~0 (survive):
$\rho_0$ shifts right by $43\%$ on the median.

Crucially, this does not contradict higher epistemic uncertainty for OoD
inputs.
Because $C_k = \mathrm{Var}[p_k]/(2\mu_k)$, a smaller $\mathrm{Var}[p_1]$
and a simultaneously smaller $\mu_1$ can still yield a larger $C_1$:
for OoD newborns, $\mu_1$ collapses faster than $\mathrm{Var}[p_1]$, so
$C_1^{\text{OoD}} > C_1^{\text{ID}}$ (Table~\ref{tab:mimic-perclass})
even while $\mathrm{Var}[p_1]^{\text{OoD}} < \mathrm{Var}[p_1]^{\text{ID}}$.
The $\mu_k$ normalisation in $C_k$ is precisely what decouples epistemic
uncertainty from MC spread, and it is the MC spread, not $C_k$, that
governs $\rho_k$.

The reversal is not specific to MIMIC or to minority classes.
The correct sufficient condition is: \emph{the MC draws are more
consistent about $p_k$ for OoD samples than for ID samples}, i.e.\
$\mathrm{Var}[p_k]^{\text{OoD}} < \mathrm{Var}[p_k]^{\text{ID}}$.
The $\rho_1$ reversal does not invalidate Lemma~\ref{lem:third-order},
which makes no per-class monotonicity claim, but it highlights that
binary-output models can concentrate their third-order regime into a single
class depending on the OoD population's characteristics.

\paragraph{FashionMNIST ($K=10$).}

\begin{figure}[htbp]
  \centering
  \includegraphics[width=\linewidth]{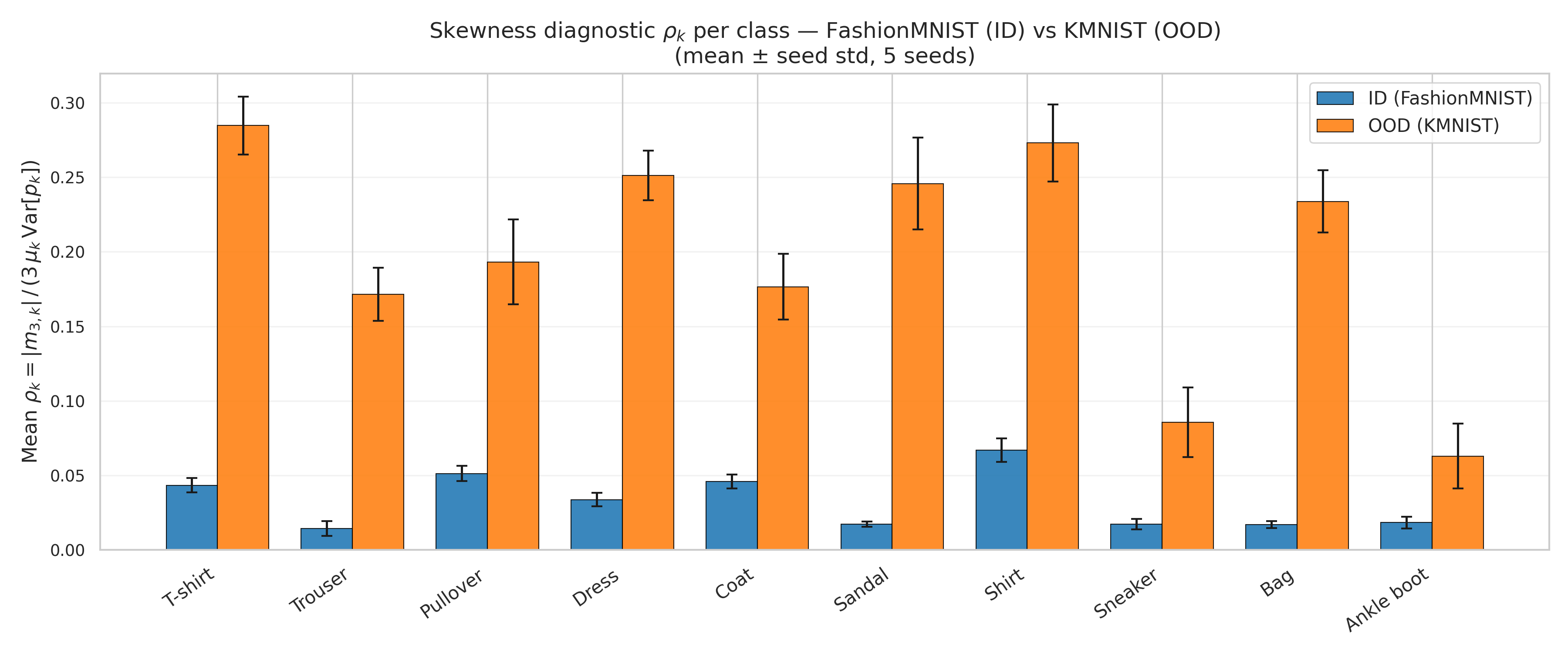}
  \caption{%
    Per-class skewness diagnostic $\rho_k$ for FashionMNIST (ID, blue)
    vs.\ KMNIST (OoD, orange), averaged over 5 seeds.
    Error bars: seed standard deviation.
All ten categories show $\rho_k^{\text{OoD}} > \rho_k^{\text{ID}}$,
    confirming that KMNIST inputs produce systematically more asymmetric
    MC posteriors across every class, with OoD/ID ratios ranging from
    $3.4\times$ (Ankle boot) to $14.2\times$ (Sandal).
    The directional pattern is consistent with the $C_k$ profile
    (Figure~\ref{fig:ck_fmnist}), though the per-class rankings differ,
    as $\rho_k$ depends on third-moment structure rather than variance alone.
  }
  \label{fig:rho_fmnist}
\end{figure}

\begin{table}[htbp]
  \centering
  \caption{%
    FashionMNIST: per-class mean $\rho_k$ for ID and OoD populations,
    averaged over 5 seeds, with Mann--Whitney $U$ $p$-values computed on
    pooled samples ($5 \times 10{,}000$ per split per class).
  }
  \label{tab:rho_fmnist}
  \small
  \begin{tabular}{lrrrrcr}
    \toprule
    Class
      & $\bar\rho_k^{\text{ID}}$ & $\sigma^{\text{ID}}$
      & $\bar\rho_k^{\text{OoD}}$ & $\sigma^{\text{OoD}}$
      & MWU $p$ & OoD/ID \\
    \midrule
    T-shirt/top  & 0.04333 & 0.00476 & 0.28470 & 0.01950 & $\approx 0$ &  6.57 \\
    Trouser      & 0.01426 & 0.00494 & 0.17139 & 0.01775 & $\approx 0$ & 12.02 \\
    Pullover     & 0.05121 & 0.00506 & 0.19319 & 0.02847 & $8 \times 10^{-260}$ &  3.77 \\
    Dress        & 0.03375 & 0.00451 & 0.25126 & 0.01659 & $\approx 0$ &  7.45 \\
    Coat         & 0.04594 & 0.00462 & 0.17652 & 0.02213 & $7 \times 10^{-264}$ &  3.84 \\
    Sandal       & 0.01726 & 0.00182 & 0.24578 & 0.03093 & $\approx 0$ & 14.24 \\
    Shirt        & 0.06693 & 0.00777 & 0.27297 & 0.02592 & $\approx 0$ &  4.08 \\
    Sneaker      & 0.01728 & 0.00344 & 0.08564 & 0.02328 & $3 \times 10^{-56}$  &  4.96 \\
    Bag          & 0.01688 & 0.00236 & 0.23378 & 0.02085 & $\approx 0$ & 13.85 \\
    Ankle boot   & 0.01832 & 0.00397 & 0.06293 & 0.02174 & $5 \times 10^{-13}$  &  3.43 \\
    \bottomrule
  \end{tabular}
\end{table}
\paragraph{Interpretation.}
These figures close the loop between Lemma~\ref{lem:third-order} and
the empirical OoD results.
OoD inputs are not merely \emph{more uncertain} (higher $C_k$) but more
\emph{asymmetrically uncertain} (higher $\rho_k$), at least for classes
where the OoD population has meaningful predicted probability mass.
When $\rho_k \ll 1$ (ID regime), the second-order approximation is
trustworthy and $\sum_k C_k \approx \text{MI}$, so the two metrics converge.
When $\rho_k$ is large (OoD regime), the cubic correction
in~\eqref{eq:lem3} becomes significant; $\sum_k C_k$ captures this
additional signal through its $1/\mu_k$ weighting, while MI averages
over the sampling distribution without per-class amplification.
On FashionMNIST, where KMNIST induces consistently higher uncertainty
across all ten logit outputs, the reversal is absent: every class shows
$\rho_k^{\text{OoD}} > \rho_k^{\text{ID}}$ (Figure~\ref{fig:rho_fmnist}).

\subsection{Deep Ensemble OoD Detection: MIMIC-III}
\label{app:ood-ensemble}

We repeat the MIMIC-III ICU$\to$Newborn OoD detection experiment using
a deep ensemble of five members (one deterministic forward pass per
member) in place of the low-rank Bayesian model.
The ensemble uses the same 44-feature pipeline and identical evaluation
code; uncertainty is derived from disagreement across members rather
than MC sampling.
As the ensemble consists of a single trained instance, no seed standard
deviation is reported.

\begin{table}[!ht]
\centering
\caption{OoD detection on MIMIC-III ICU$\to$Newborn: deep ensemble.
AUROC and OoD/ID mean-score ratio. Best in \textbf{bold}.}
\label{tab:ood-ensemble}
\small
\begin{tabular}{lcc}
\toprule
Method & AUROC$\uparrow$ & Ratio$\uparrow$ \\
\midrule
Neg.\ MSP        & $0.701$ & $1.68$ \\
MI               & $0.752$ & $2.32$ \\
$EU_{\text{var}}$& $0.750$ & $\mathbf{2.72}$ \\
\midrule
$\sum_k C_k$     & $\mathbf{0.753}$ & $2.31$ \\
\bottomrule
\end{tabular}
\end{table}

The ranking $\sum_k C_k > \text{MI} > EU_{\text{var}}$ mirrors the
low-rank Bayesian result (Table~\ref{tab:ood-main}).
Critically, $EU_{\text{var}}$ again achieves the highest OoD/ID ratio
($2.72\times$) yet the lowest AUROC among epistemic measures: without
mean-normalisation, boundary suppression compresses the dynamic range
of scores, so a numerically large mean shift is swamped by
within-group spread, failing to produce separable distributions.
$\sum_k C_k$ corrects this via the $1/\mu_k$ weighting, yielding the
highest AUROC across both inference regimes.
\section{Epistemic Sensitivity to Data Quality: Supplementary Details}
\label{app:disentangle}

\subsection{Model Architecture and Training}\label{app:disentangle-model}

Both end-to-end models use low-rank Gaussian posteriors with the shared specification described in Appendix~\ref{app:ood-model}.

\paragraph{Fashion-MNIST.}
Fully connected Bayesian network, two hidden layers of $400$ units, rank~$15$, ReLU activations.
Trained for $50$ epochs with Adam (lr $= 10^{-3}$), batch size $128$, $S{=}50$ MC~samples at inference.

\paragraph{CIFAR-10 (from scratch).}
CNN with Conv~$32{\to}64{\to}128$, each followed by ReLU and $2{\times}2$ max-pooling, global average pooling, and a low-rank Bayesian head (rank~$32$).
Trained for $100$ epochs with Adam (lr $= 10^{-3}$), batch size $128$, $S{=}50$ MC~samples at inference.

\paragraph{Label noise injection.}
For each noise rate $\alpha \in \{0.1, 0.2, 0.3, 0.4, 0.5\}$, a fraction $\alpha$ of training labels is replaced uniformly at random across all classes.
The test set remains clean throughout.
Models are retrained from scratch at each noise level; results are mean over $5$ seeds.

\subsection{Transfer Learning Experiments}\label{app:disentangle-transfer}

To isolate the effect of end-to-end Bayesian training, we repeat the disentanglement protocol with a frozen ImageNet-pretrained ResNet-50 backbone, replacing only the classifier head with one of three Bayesian variants:
\begin{itemize}
    \item \textbf{Low-rank variational} (rank $= 32$, constrained posterior),
    \item \textbf{Full-rank variational} (unconstrained Bayesian weight distributions),
    \item \textbf{MC~Dropout} ($p{=}0.3$, deterministic weights with dropout-based uncertainty).
\end{itemize}

This design enables a controlled comparison: CIFAR-10 from scratch vs.\ CIFAR-10 transfer uses the \emph{same dataset}, \emph{same $K$}, and \emph{same Bayesian model type} (low-rank), isolating the training regime as the causal variable.

\paragraph{Why a relative ratio.}
The naive absolute ratio $R_{\mathrm{abs}}(\alpha) = (\bar{U}_e(\alpha) - \bar{U}_e(0))/(\bar{U}_a(\alpha) - \bar{U}_a(0))$ is confounded by the scale of epistemic uncertainty.
Low-rank models produce baseline MI an order of magnitude below their full-rank counterparts ($0.007$ vs.\ $0.039$ on CIFAR-10; Table~\ref{tab:baseline_comparison}), so the absolute numerator $\bar{U}_e(\alpha) - \bar{U}_e(0)$ is mechanically small regardless of whether the \emph{proportional} increase is large or small.
A low-rank model and a full-rank model with identical proportional epistemic leakage would report very different $R_{\mathrm{abs}}$ values, making cross-model comparison misleading.
The relative formulation $R_{\mathrm{rel}}$ (Equation~\ref{eq:disentangle-ratio}) normalises each delta by its own baseline, yielding a scale-invariant elasticity: for every $1\%$ relative increase in aleatoric uncertainty, how many percent does epistemic uncertainty increase?
This makes $R_{\mathrm{rel}}$ directly comparable across models with different posterior families, rank constraints, and training regimes.

\paragraph{Baseline inflation.}
Table~\ref{tab:baseline_comparison} reports baseline epistemic uncertainty ($\alpha{=}0$) across all configurations.
End-to-end models show negligible inflation of $\sum_k \Ck$ over MI ($1.0$--$1.02\times$), while transfer-learning models exhibit progressively higher inflation: low-rank ($1.08$--$1.17\times$), full-rank ($1.35$--$1.59\times$), MC~dropout ($1.62$--$1.89\times$).
Training regime matters more than dataset: CIFAR-10 low-rank with transfer learning ($1.08\times$) shows more inflation than CIFAR-10 low-rank from scratch ($1.02\times$) on the \emph{same data}.

\begin{figure}[!ht]
  \centering
  \includegraphics[width=0.48\columnwidth]{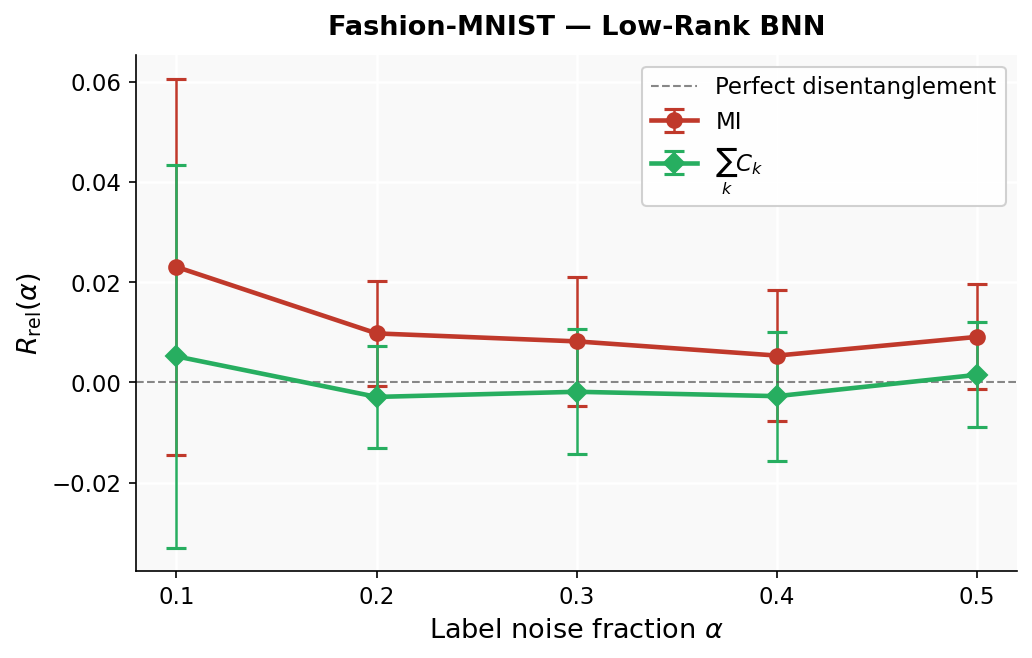}%
  \hfill
  \includegraphics[width=0.48\columnwidth]{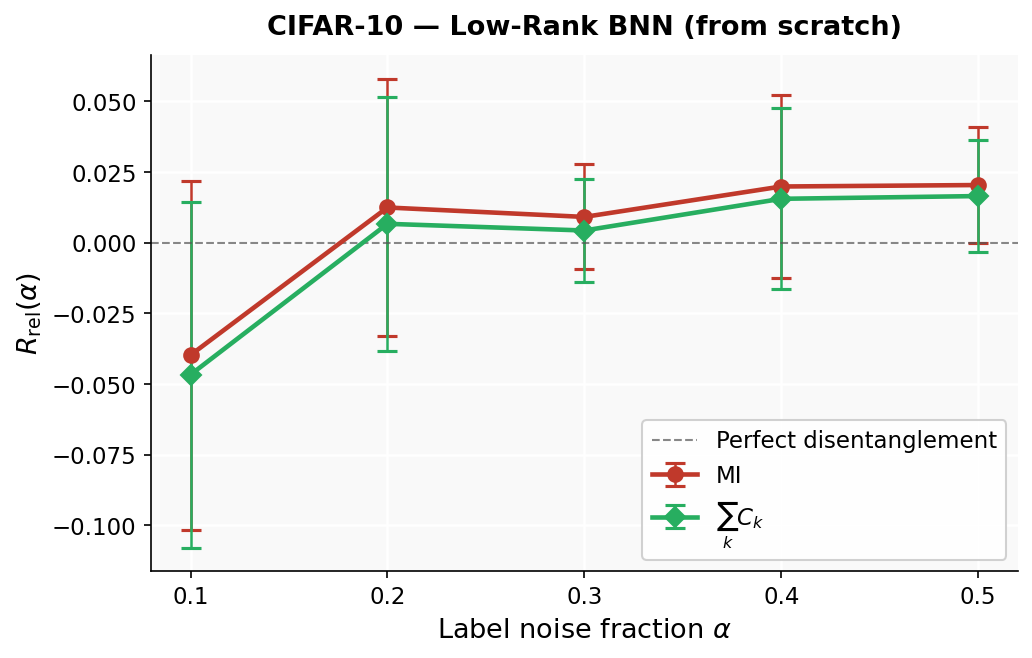}
  \caption{Disentanglement ratios $|R_{\mathrm{rel}}(\alpha)|$ for end-to-end low-rank models. $\sum_k \Ck$ (green) is closer to zero than MI (red) at every noise level.
  \textbf{Left:} Fashion-MNIST. \textbf{Right:} CIFAR-10.}
  \label{fig:disentangle-ratios-e2e}
\end{figure}

\begin{figure}[!ht]
  \centering
  \includegraphics[width=\columnwidth]{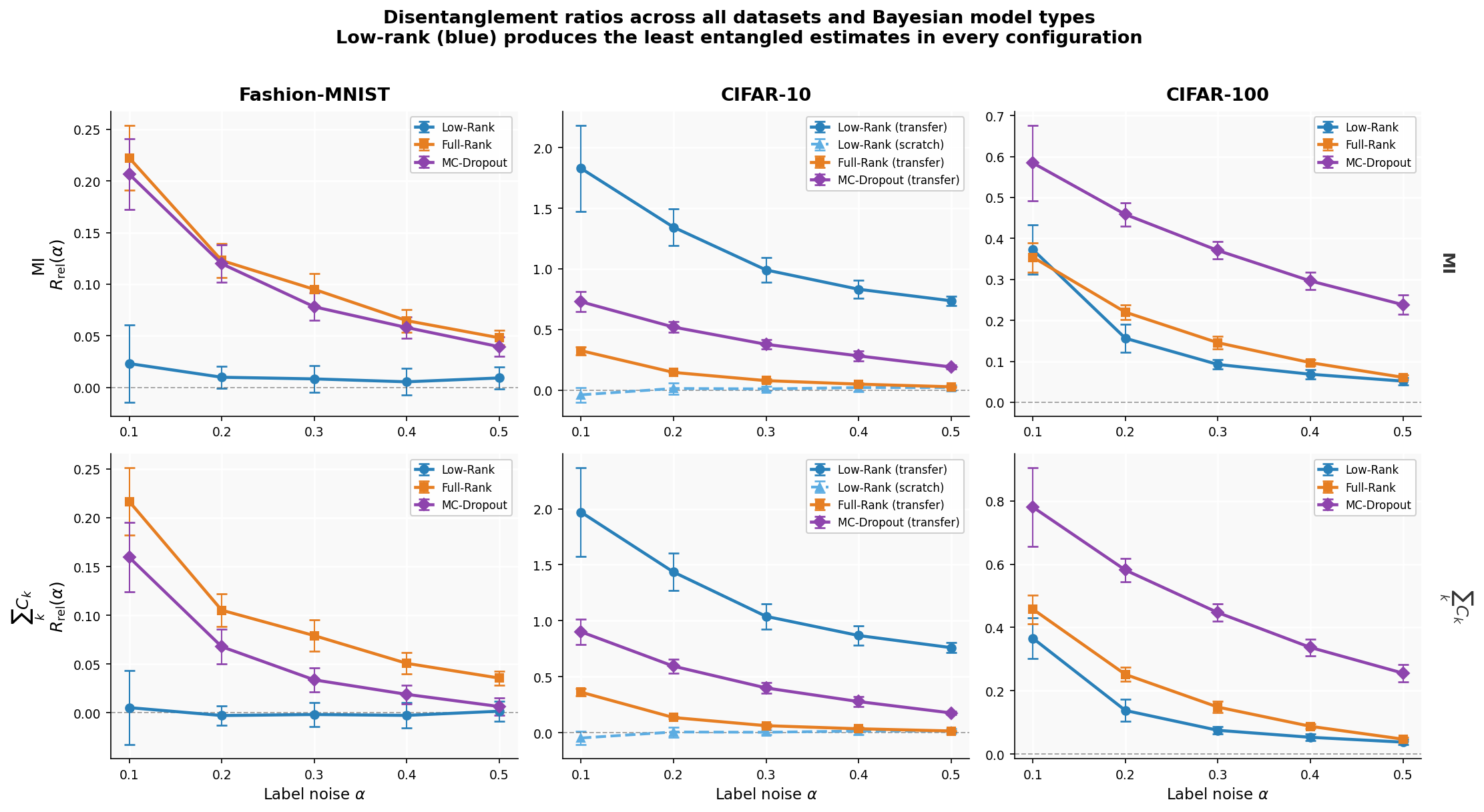}
  \caption{Relative disentanglement ratios $|R_{\mathrm{rel}}(\alpha)|$ across all datasets and Bayesian model types. Top row: MI. Bottom row: $\sum_k \Ck$. Low-rank (blue) produces the least entangled estimates in every configuration.}
  \label{fig:cross-dataset-ratios}
\end{figure}
\begin{table}[!ht]
\centering
\caption{Baseline epistemic uncertainty ($\alpha{=}0$, no label noise). Inflation $= \sum_k \Ck \,/\, \text{MI}$.}
\label{tab:baseline_comparison}
\begin{tabular}{llccccc}
\toprule
Dataset & Model & Training & $K$ & MI & $\sum_k \Ck$ & Inflation \\
\midrule
Fashion-MNIST & Low-rank Bayes & From scratch & 10 & 0.008 & 0.008 & $1.00\times$ \\
CIFAR-10 & Low-rank Bayes & From scratch & 10 & 0.0067 & 0.0069 & $1.02\times$ \\
\midrule
CIFAR-10 & Low-rank Bayes & Transfer & 10 & 0.0045 & 0.0048 & $1.08\times$ \\
CIFAR-10 & Full-rank Bayes & Transfer & 10 & 0.039 & 0.053 & $1.35\times$ \\
CIFAR-10 & MC~Dropout & Transfer & 10 & 0.039 & 0.063 & $1.62\times$ \\
\midrule
CIFAR-100 & Low-rank Bayes & Transfer & 100 & 0.042 & 0.049 & $1.17\times$ \\
CIFAR-100 & Full-rank Bayes & Transfer & 100 & 0.190 & 0.302 & $1.59\times$ \\
CIFAR-100 & MC~Dropout & Transfer & 100 & 0.238 & 0.450 & $1.89\times$ \\
\bottomrule
\end{tabular}
\end{table}

\paragraph{Disentanglement ratios: CIFAR-10 transfer.}
Table~\ref{tab:disentanglement_cifar10} shows that under transfer learning MI outperforms $\sum_k \Ck$ for the low-rank model at all noise levels ($|R_{\mathrm{rel}}| = 1.83$ vs.\ $1.97$ at $\alpha{=}0.1$), consistent with the high inflation of this configuration.
For full-rank models, $\sum_k \Ck$ overtakes MI at $\alpha \geq 0.2$.
MC~dropout shows the worst overall ratios ($0.18$--$0.90$), with $\sum_k \Ck$ recovering at high noise ($\alpha \geq 0.4$).
The model ordering for disentanglement quality is consistent: low-rank best under end-to-end training but worst under transfer, full-rank intermediate, MC~dropout worst overall.

\begin{table}[!ht]
\centering
\caption{$|R_{\mathrm{rel}}(\alpha)|$ for CIFAR-10 ($K{=}10$, transfer learning). Values $< 0.3$ in \textbf{bold}. $\dagger$: $\sum_k \Ck$ outperforms MI.}
\label{tab:disentanglement_cifar10}
\small
\begin{tabular}{lcccccc}
\toprule
& \multicolumn{2}{c}{Low-rank} & \multicolumn{2}{c}{Full-rank} & \multicolumn{2}{c}{MC~Dropout} \\
\cmidrule(lr){2-3} \cmidrule(lr){4-5} \cmidrule(lr){6-7}
$\alpha$ & MI & $\sum_k \Ck$ & MI & $\sum_k \Ck$ & MI & $\sum_k \Ck$ \\
\midrule
0.1 & $1.830$ & $1.970$ & $\mathbf{0.325}$ & $0.362$ & $0.728$ & $0.901$ \\
0.2 & $1.344$ & $1.435$ & $\mathbf{0.145}$ & $\mathbf{0.136}^{\dagger}$ & $0.519$ & $0.594$ \\
0.3 & $0.991$ & $1.039$ & $\mathbf{0.077}$ & $\mathbf{0.063}^{\dagger}$ & $0.377$ & $0.400$ \\
0.4 & $0.832$ & $0.868$ & $\mathbf{0.048}$ & $\mathbf{0.035}^{\dagger}$ & $\mathbf{0.282}$ & $\mathbf{0.278}^{\dagger}$ \\
0.5 & $0.737$ & $0.760$ & $\mathbf{0.026}$ & $\mathbf{0.016}^{\dagger}$ & $\mathbf{0.190}$ & $\mathbf{0.177}^{\dagger}$ \\
\bottomrule
\end{tabular}
\end{table}

\paragraph{Disentanglement ratios: CIFAR-100 transfer.}
Table~\ref{tab:disentanglement_cifar100} extends the analysis to $K{=}100$.
Low-rank achieves strong disentanglement ($|R_{\mathrm{rel}}| < 0.37$) for both metrics at all noise levels, with $\sum_k \Ck$ winning at every $\alpha$.
Full-rank shows MI winning at low noise but $\sum_k \Ck$ recovering at $\alpha \geq 0.4$.
MC~dropout is the worst configuration: MI wins at all noise levels, with ratios exceeding $0.58$.

\begin{table}[!ht]
\centering
\caption{$|R_{\mathrm{rel}}(\alpha)|$ for CIFAR-100 ($K{=}100$, transfer learning). Values $< 0.3$ in \textbf{bold}. $\dagger$: $\sum_k \Ck$ outperforms MI.}
\label{tab:disentanglement_cifar100}
\small
\begin{tabular}{lcccccc}
\toprule
& \multicolumn{2}{c}{Low-rank} & \multicolumn{2}{c}{Full-rank} & \multicolumn{2}{c}{MC~Dropout} \\
\cmidrule(lr){2-3} \cmidrule(lr){4-5} \cmidrule(lr){6-7}
$\alpha$ & MI & $\sum_k \Ck$ & MI & $\sum_k \Ck$ & MI & $\sum_k \Ck$ \\
\midrule
0.1 & $0.373$ & $\mathbf{0.366}^{\dagger}$ & $0.354$ & $0.458$ & $0.584$ & $0.780$ \\
0.2 & $\mathbf{0.156}$ & $\mathbf{0.138}^{\dagger}$ & $\mathbf{0.220}$ & $\mathbf{0.252}$ & $0.459$ & $0.581$ \\
0.3 & $\mathbf{0.092}$ & $\mathbf{0.075}^{\dagger}$ & $\mathbf{0.145}$ & $\mathbf{0.149}$ & $0.371$ & $0.447$ \\
0.4 & $\mathbf{0.069}$ & $\mathbf{0.053}^{\dagger}$ & $\mathbf{0.097}$ & $\mathbf{0.088}^{\dagger}$ & $\mathbf{0.296}$ & $0.337$ \\
0.5 & $\mathbf{0.052}$ & $\mathbf{0.038}^{\dagger}$ & $\mathbf{0.061}$ & $\mathbf{0.047}^{\dagger}$ & $\mathbf{0.238}$ & $\mathbf{0.256}$ \\
\bottomrule
\end{tabular}
\end{table}
\begin{figure}[!ht]
  \centering
  \includegraphics[width=0.37\columnwidth]{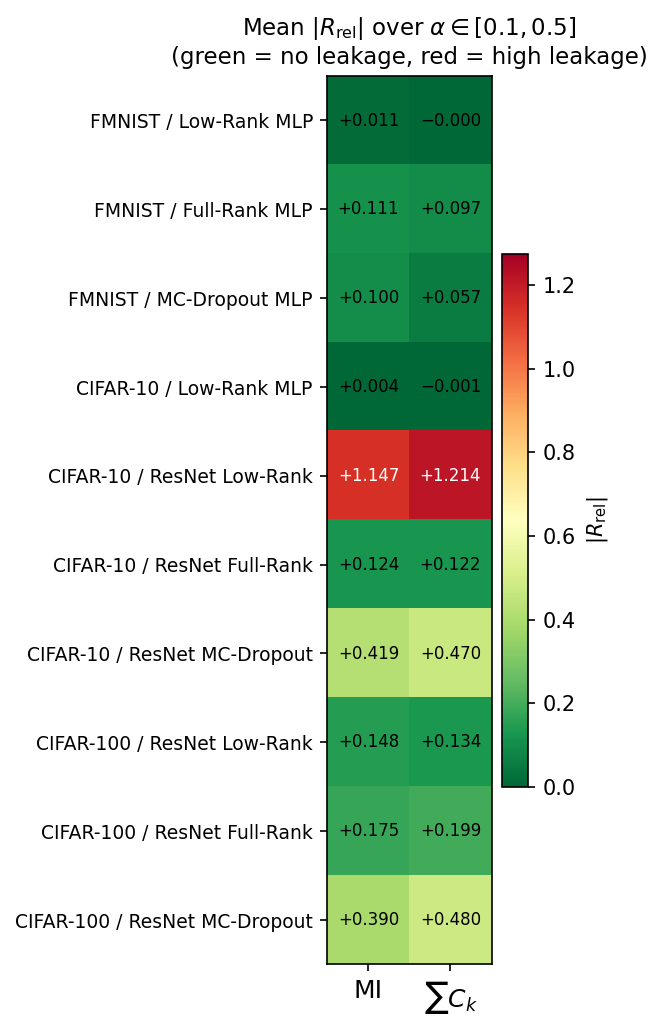}
  \caption{Mean $|R_{\mathrm{rel}}|$ averaged over $\alpha \in \{0.1,\dots,0.5\}$ for all ten model/dataset combinations.
  Green indicates near-zero leakage; red indicates high entanglement.
  End-to-end low-rank models (top rows) achieve the strongest disentanglement for both metrics.}
  \label{fig:disentangle-heatmap}
\end{figure}
\paragraph{Key patterns across transfer-learning experiments.}
The controlled comparison on CIFAR-10 is decisive: using the \emph{same dataset} ($K{=}10$) and \emph{same model type} (low-rank), training from scratch yields $|R_{\mathrm{rel}}| < 0.05$ while transfer learning yields $0.74$--$1.97$.
This directly isolates the training regime as a causal factor.
Within transfer learning, higher baseline inflation ($\sum_k \Ck\,/\,\text{MI}$) correlates with worse $\sum_k \Ck$ performance relative to MI, a pattern that holds for both CIFAR-10 and CIFAR-100.
The implication is that when the entire model participates in the Bayesian posterior, the resulting predictive distributions have a variance structure that is well captured by the $\Ck$ normalisation; when only a classifier head is Bayesian, this structure degrades.

\paragraph{Effect of rank constraint on uncertainty scale.}
The low-rank variational framework of \citet{toure2026singular}
parameterizes each weight matrix as $W = AB^\top$ with independent
mean-field Gaussian posteriors $q(A)q(B)$ on the factors
$A \in \mathbb{R}^{m \times r}$, $B \in \mathbb{R}^{n \times r}$.
The induced posterior on $W$ is \emph{singular} with respect to
Lebesgue measure, concentrating entirely on the manifold
$\mathcal{R}_r$ of rank-$r$ matrices: every sampled weight matrix is
constrained to have rank at most $r$, though the specific column and
row spaces vary across posterior samples.
Under the same transfer-learning protocol on CIFAR-10, the low-rank
model ($r{=}32$) produces baseline MI an order of magnitude below its
full-rank counterpart ($0.0045$ vs.\ $0.039$), a gap that persists
across all noise levels (e.g., at $\alpha{=}0.5$: $0.047$ vs.\
$0.072$).
The from-scratch low-rank model exhibits a similar absolute scale
($\text{MI} = 0.0067$), confirming that the compression is driven
primarily by the rank constraint rather than the training regime.
This compression has two distinct sources.
First, the rank constraint itself restricts the number of
independent directions along which weight matrices can vary; the
resulting predictive distributions inherit this reduced
dimensionality, yielding smaller $\Var[p_k]$ and thus smaller MI.
Second, while the factored posterior introduces structured
correlations between weight entries ($\Covmat(W_{ij}, W_{i'j'}) \neq 0$
whenever the entries share latent factors), it remains mean-field
\emph{within} each factor, which is known to underestimate marginal
variances relative to the true
posterior~\citep{blei2017variational,Turner_Sahani_2011}, further
concentrating the predictive distribution.
Crucially, this compression affects the \emph{absolute scale} of
uncertainty but not the \emph{relative ranking} of inputs or the
\emph{within-model metric comparison}: since $\sum_k \Ck$ and MI
share the same compressed scale by construction
(Equation~\ref{eq:additive}), their relative disentanglement ratios
$R_{\mathrm{rel}}(\alpha)$ remain directly comparable within a given model, and the
finding that $\sum_k \Ck$ leaks less than MI holds independently of
the overall scale.
However, uncertainty thresholds calibrated on one posterior family
cannot be transferred to another without recalibration, and the small
absolute values produced by low-rank models should not be interpreted
as evidence of low true epistemic uncertainty but rather as a
consequence of both the manifold constraint on the posterior support
and the mean-field under-dispersion within the learned factors.

\subsection{Sensitivity to Class Cardinality}\label{app:disentangle-scaling}

The $1/\mu_k$ normalisation that corrects boundary suppression also introduces a scaling effect with the number of classes $K$.
Under the simplex constraint $\sum_k \bar{p}_k = 1$, as $K$ grows the average probability per class decreases.
For a model assigning probability $\alpha$ to the predicted class and distributing the remainder uniformly:
\begin{equation}
  \bar{p}_k \;=\; \frac{1-\alpha}{K-1} \quad \text{for } k \neq \text{pred}.
\end{equation}
Assuming approximately uniform variance $\Var(p_k) \approx \sigma^2$ across classes:
\begin{equation}
  \sum_k \Ck \;\approx\; \frac{\sigma^2}{2\alpha} \;+\; \frac{(K-1)^2\,\sigma^2}{2(1-\alpha)} \;\sim\; \mathcal{O}(K^2),
\end{equation}
whereas MI, measuring entropy differences bounded by $\log K$, does not accumulate linearly over classes.

Empirically, this manifests as higher baseline inflation on CIFAR-100 compared to CIFAR-10 for every model type (Table~\ref{tab:baseline_comparison}): low-rank $1.17\times$ vs.\ $1.08\times$; full-rank $1.59\times$ vs.\ $1.35\times$; MC~dropout $1.89\times$ vs.\ $1.62\times$.
The $K$-dependence interacts with model type: low-rank models remain resilient at $K{=}100$ ($|R_{\mathrm{rel}}| < 0.37$), while MC~dropout fails substantially ($|R_{\mathrm{rel}}|$ up to $0.78$ at $\alpha{=}0.1$).

\paragraph{Mitigation strategies.}
For high-cardinality settings ($K \gtrsim 50$), two modifications can reduce inflation without sacrificing the per-class interpretability of $\Ck$:
(i)~\emph{truncated summation}, restricting $\sum_k \Ck$ to the top-$k$ most probable classes, which eliminates the accumulation of amplified terms from negligible-probability classes;
(ii)~\emph{probability-weighted aggregation}, weighting each $\Ck$ by $\mu_k$ to down-weight low-probability contributions.
We recommend reporting both $\sum_k \Ck$ and MI in high-$K$ settings: the per-class decomposition retains interpretive value even when the aggregate $\sum_k \Ck$ underperforms MI as a scalar summary.

\end{document}